\documentclass{article}
\usepackage{graphicx}
\usepackage{multirow}
\usepackage{tabularx}
\usepackage{bbding}
\usepackage{algorithm}
\usepackage{algorithmic}
\usepackage{amsmath}
\usepackage{amsthm} 
\usepackage{amssymb}
\usepackage{makecell}
\usepackage{colortbl}
\usepackage{subfig}
\usepackage{subcaption}
\usepackage{fancybox}
\usepackage{enumitem}
\definecolor{DownGreen}{RGB}{34, 139, 34}
\definecolor{UpOrange}{RGB}{210, 105, 30}
\definecolor{NeutralGray}{RGB}{120, 120, 120}
\newcommand{\dec}[1]{\textcolor{DownGreen}{\scriptsize\,\raisebox{0.2ex}{$\downarrow$}\,#1}}
\newcommand{\inc}[1]{\textcolor{UpOrange}{\scriptsize\,\raisebox{0.2ex}{$\uparrow$}\,#1}}

\definecolor{rowblue}{RGB}{235, 242, 250}

\newcommand{\oursrow}{\rowcolor{rowblue}}

\newcommand{\method}{\textbf{MuonEq}}
\newtheorem{lemma}{Lemma}[section]
\newtheorem{assumption}{Assumption}[section]
\newtheorem{theorem}{Theorem}[section]

\newtheorem{proposition}[theorem]{Proposition}
\newtheorem{corollary}[theorem]{Corollary}

\newtheorem{remark}{Remark}[section]

\usepackage[preprint]{neurips_2026}

\usepackage[utf8]{inputenc} % allow utf-8 input
\usepackage[T1]{fontenc}    % use 8-bit T1 fonts
\usepackage{hyperref}       % hyperlinks
\usepackage{url}            % simple URL typesetting
\usepackage{booktabs}       % professional-quality tables
\usepackage{amsfonts}       % blackboard math symbols
\usepackage{nicefrac}       % compact symbols for 1/2, etc.
\usepackage{microtype}      % microtypography
\usepackage{xcolor}         % colors

\definecolor{customgray}{gray}{0.9}
\title{MuonEq: Balancing Before Orthogonalization with Lightweight Equilibration}

\author{Da Chang$^{137}$, Qiankun Shi$^{34}$, Lvgang Zhang$^{35}$, Yu Li$^{6}$, Ruijie Zhang$^7$ \\
\bfseries Yao Lu$^3$, Yongxiang Liu$^3$\thanks{Corresponding authors: liuyx@pcl.ac.cn and yuanganzhao@foxmail.com}, Ganzhao Yuan$^{2}$\footnotemark[1] \\[3mm]
\centerline{\normalsize $^1$Shenzhen Institute of Advanced Technology, Chinese Academy of Sciences}\\
\centerline{\normalsize $^2$Shenzhen University of Advanced Technology $^3$Pengcheng Laboratory $^4$Sun Yat-sen University}\\
\centerline{\normalsize $^5$Southern University of Science and Technology $^6$George Washington University}\\
\centerline{\normalsize $^7$University of Chinese Academy of Sciences}
}
\begin{document}

\maketitle

\begin{abstract}
Orthogonalized-update optimizers such as Muon improve training of matrix-valued parameters, but existing extensions typically either rescale updates after orthogonalization or use heavier whitening-based preconditioners before it. We introduce {\method}, a lightweight family of pre-orthogonalization equilibration schemes for Muon with three forms: two-sided row/column normalization (RC), row normalization (R), and column normalization (C). By rebalancing the momentum matrix before finite-step Newton--Schulz orthogonalization, {\method} improves the geometry seen by orthogonalization. We show that finite-step orthogonalization is governed by the input spectrum, especially stable rank and condition number, and that row/column normalization acts as a zeroth-order surrogate for whitening. For hidden matrix weights, R is the default variant. 
Theoretically, {\method} (R) retains the standard $\widetilde{\mathcal O}(T^{-1/4})$ Muon-type nonconvex stationarity guarantee with decoupled weight decay and a horizon-free diminishing learning-rate schedule, and extends it to finite-step NS5 up to an explicit inexactness constant.
In LLaMA2 pretraining on C4, {\method} (R) consistently outperforms Muon on 130M, 350M, and 1B models, with faster convergence and lower validation perplexity.
The code is available at the \href{https://github.com/MaeChd/muon-eq}{MuonEq codebase}.
\end{abstract}
%%%%%%%%%%%%%%%%%%%%%%%%%%%%%%%%

\section{Introduction}
\label{sec:intro}
As the computational and data costs of large language model pretraining grow, optimization increasingly determines sample efficiency, throughput, and scalability. Beyond coordinate-wise adaptive methods such as AdamW~\citep{ilya2019adamw}, Adan~\citep{xie2024adan}, and MARS~\citep{yuan2024mars}, recent work has shifted toward a geometric view of matrix-valued updates. Muon~\citep{jordan6muon}, for instance, uses a Newton--Schulz approximation to the polar factor of the momentum matrix to produce spectrally aligned updates~\citep{bernstein2024old,chen2025muon}. It has shown strong empirical scalability in LLM pretraining, often improving the trade-off among data efficiency, compute, and wall-clock time at fixed target loss~\citep{liu2025muon,shah2025practical,deepseekai2026deepseekv4}. Recent theory further characterizes Muon as spectral-norm steepest descent and establishes nonconvex convergence guarantees for both exact and approximate orthogonalization~\citep{chang2025convergence,Sato2025ConvergenceBA,Li2025ANO,shen2025convergence,sfyraki2025lions,kim2026convergence,shulgin2025beyond}. 
% Our analysis also complements this literature by giving a guarantee closer to the optimizer used in practice.

Follow-up work around Muon mainly proceeds along two directions. \textit{(i) Post-orthogonalization rescaling} augments an already constructed orthogonal update with additional magnitude correction. Muon+~\citep{zhang2026muon+}, NorMuon~\citep{li2025normuon}, and AdaMuon~\citep{si2025adamuon} respectively add row/column normalization, neuron-wise normalization, or element-wise variance adaptation after orthogonalization. These variants show that orthogonal updates remain compatible with later adaptive scaling, but they do not change the matrix that finite-step Newton--Schulz actually sees. \textit{(ii) Pre-orthogonalization geometric correction} instead modifies the coordinate system before spectral optimization. SOAP~\citep{vyas2024soap} connects Shampoo~\citep{gupta2018shampoo} with Adam/Adafactor~\citep{Shazeer2018AdafactorAL,kingma2014adam} through a preconditioned basis, while FISMO~\citep{Xu2026FISMOFM} and Mousse~\citep{Zhang2026MousseRT} further perform spectral optimization in whitened coordinates. 
These methods motivate studying geometry correction before orthogonalization as a promising lever, but existing approaches typically incur substantially higher state and computational overhead.
% These methods suggest that improving geometry before orthogonalization is the more direct lever, but doing so usually requires maintaining left and right preconditioners together with extra matrix operations such as eigendecomposition, inverse power iteration, and whitening/dewhitening maps.

These works motivate a more precise question:
\vspace{-5pt}
\begin{center}
\shadowbox{\begin{minipage}[t]{0.95\columnwidth}%
\it Is there a lightweight design between post-orthogonalization rescaling and full whitening-based preconditioning that directly improves the input geometry for orthogonalization?
\end{minipage}}
\end{center}
\vspace{-8pt}

Our starting point is that, in Muon, the relevant object is not only the exact polar factor but also the finite-step Newton--Schulz approximation used in practice. Section~\ref{sec:ana_1} shows that both the convergence speed of Newton--Schulz iterations and the gap between finite-step iterations and exact orthogonalization are governed by input spectral geometry, especially stable rank and condition number. This points to a simple design principle: improve the geometry \emph{before} orthogonalization, but do so with the smallest possible intervention. A natural formalization is the optimal diagonal preconditioning problem:
\[
\begin{aligned}
\min_{\mathbf D_1,\mathbf D_2\succ 0}\kappa\bigl(\mathbf D_1^{1/2}\mathbf A\mathbf D_2^{-1/2}\bigr).
\end{aligned}
\]
Qu et al.~\citep{qu2025optimal} derive an optimal diagonal preconditioner, but applying it online typically requires an auxiliary iterative solve. Classical row/column equilibration is less exact yet much lighter-weight~\citep{ruiz2001scaling}; moreover, recent operator-geometry results identify row and column normalization as steepest-descent directions under the $p\to\infty$ and $1\to q$ geometries, respectively~\citep{Xu2026moga}. This makes diagonal equilibration a natural lightweight way to improve the matrix presented to Muon-style orthogonalization.

Motivated by this observation, we propose {\method}, a lightweight family of pre-orthogonalization equilibration schemes for Muon, instantiated as two-sided row/column normalization (RC), row normalization (R), and column normalization (C). 
All variants rescale the momentum matrix before Newton--Schulz, using row/column squared-norm reductions from the current momentum.
Section~\ref{sec:analysis} provides the theory: Theorem~\ref{th_ns} quantifies the spectral sensitivity of finite-step orthogonalization; Proposition~\ref{th_zero_rc} shows that row/column normalization is a zeroth-order whitening surrogate; and Proposition~\ref{th_conv_rc}, Theorem~\ref{th_conv_r}, and Corollary~\ref{cor:conv_r_ns} formalize the RC/R trade-off. 
% In particular, the default R variant admits the standard $\widetilde{\mathcal O}(T^{-1/4})$ nonconvex guarantee under a closer-to-practice setting with decoupled weight decay, horizon-free schedules, and practical NS5.
% In particular, the default R variant admits the standard $\widetilde{\mathcal O}(T^{-1/4})$ nonconvex stationarity guarantee while explicitly accounting for decoupled weight decay, a learning-rate schedule independent of the training horizon, and finite-step NS5.
In particular, the default R variant satisfies the standard $\widetilde{\mathcal O}(T^{-1/4})$ nonconvex stationarity guarantee while explicitly accounting for decoupled weight decay, horizon-free learning-rate schedules, and finite-step NS5.
For the hidden Transformer weight matrices targeted by {\method}, recent architecture-aware analyses~\citep{Xu2026moga} identify row-sided geometry as the natural one-sided choice, while column-sided geometry is more aligned with embeddings; hence we use R by default and keep RC/C for analysis and ablation. Compared with Muon+~\citep{zhang2026muon+}, which rescales updates after orthogonalization, {\method} changes the input to orthogonalization. 
Unlike whitening-based methods such as FISMO and Mousse~\citep{Xu2026FISMOFM,Zhang2026MousseRT}, {\method} avoids Kronecker factors, matrix-valued second-order statistics, eigendecomposition, and momentum transport; its diagonal rescaling is computed on the fly from row/column squared norms, with no persistent optimizer state.

Our contributions are threefold. \textit{(i)} First, we propose {\method}, a lightweight family of pre-orthogonalization equilibration schemes for Muon with three forms: RC, R, and C; these variants rebalance the current momentum matrix before Newton--Schulz using row/column squared norms computed on the fly, with R as the default for hidden matrix weights. \textit{(ii)} Second, we analyze how pre-orthogonalization geometry affects finite-step Newton--Schulz: Theorem~\ref{th_ns} links its behavior to stable rank and condition number, Proposition~\ref{th_zero_rc} identifies row/column normalization as a zeroth-order whitening surrogate, and Proposition~\ref{th_conv_rc} together with Theorem~\ref{th_conv_r} and Corollary~\ref{cor:conv_r_ns} formalize the trade-off between stronger two-sided spectral correction and cleaner one-sided geometry. In particular, {\method} (R) satisfies the standard $\widetilde{\mathcal O}(T^{-1/4})$ nonconvex stationarity guarantee while accounting for decoupled weight decay, horizon-free learning-rate schedules, and finite-step NS5. \textit{(iii)} Third, empirically, the default R variant consistently outperforms Muon across model sizes and token budgets, while ablations over RC, R, and C show that R is the default form for the hidden-weight setting targeted by {\method}.

\section{Method}
\label{sec:method}

\textbf{Notation and setup.} Scalars are denoted by non-bold letters, vectors by bold lowercase letters, and matrices by bold uppercase letters. On $\mathbb R^d$, we use the Euclidean inner product $\langle \mathbf x,\mathbf y\rangle_2:=\mathbf x^\top\mathbf y$ and norm $\|\mathbf x\|_2$. For matrices, $\langle \mathbf A,\mathbf B\rangle_F:=\operatorname{tr}(\mathbf A^\top\mathbf B)$ and $\|\mathbf A\|_F$ denote the Frobenius inner product and norm, and $\|\mathbf A\|_*$ denotes the nuclear norm. We write $[m]:=\{1,2,\ldots,m\}$ and let $\mathbb N$ be the set of non-negative integers. The model parameter is a matrix $\mathbf X\in\mathbb R^{m\times n}$, and, unless otherwise stated, we assume $m\ge n$. For any matrix $\mathbf A\in\mathbb R^{m\times n}$ with rank $r\le n$, we write its compact singular value decomposition as $\mathbf A=\mathbf U\Sigma\mathbf V^\top$, where $\Sigma=\operatorname{diag}(\sigma_1,\dots,\sigma_r)$ and $\sigma_1\ge\cdots\ge\sigma_r>0$, 
% and define the exact-polar decomposition $\operatorname{Orth}(\mathbf A):=\mathbf U\mathbf V^\top$. 
and define the polar factor, which we also call the orthogonalized form, as $\mathrm{Orth}(\mathbf{A}):=\mathbf{U}\mathbf{V}^\top$. Here, orthogonalization and polar step mean applying or approximating this map, whereas whitening denotes full Gram inverse-square-root preconditioning. The row and column rescalings introduced below are diagonal equilibration operations.
In finite-data form, the training objective can be written as $f(\mathbf X):=\frac{1}{N}\sum_{i\in[N]} f_i(\mathbf X)$, where $f_i(\mathbf X)$ is the loss associated with the $i$-th sample $\mathbf z_i$. Equivalently, we use the stochastic formulation
\begin{equation}
\min_{\mathbf{X}\in\mathbb{R}^{m\times n}} f(\mathbf{X}),
\qquad
f(\mathbf{X})=\mathbb{E}_{\xi\sim\mathcal D}[f(\mathbf{X};\xi)],
\end{equation}
where $f$ may be nonconvex, $\xi$ is independent of $\mathbf X$, and $\mathbb E_\xi[\cdot]$ denotes expectation with respect to $\xi$. For any matrix $\mathbf A$ with singular values $\{\sigma_i(\mathbf A)\}_{i=1}^r$, we further use
$\operatorname{sr}(\mathbf A):=\frac{\|\mathbf A\|_F^2}{\|\mathbf A\|_2^2},
\kappa(\mathbf A):=\frac{\sigma_1(\mathbf A)}{\sigma_r(\mathbf A)},\kappa_i(\mathbf{A}):=\frac{\sigma_1(\mathbf{A})}{\sigma_i(\mathbf{A})},
p_i:=\frac{\sigma_i(\mathbf A)^2}{\|\mathbf A\|_F^2},
H(p):=-\sum_{i=1}^r p_i\log p_i.$

\textbf{Muon as finite-step orthogonalized momentum.} For matrix parameters, Muon constructs the update by orthogonalizing the momentum matrix. Let $\mathbf M_t$ denote the momentum at iteration $t$. The abstract Muon update is
\begin{equation}
\mathbf O_t = \operatorname{Orth}(\widetilde{\mathbf M}_t),
\qquad
\mathbf X_{t+1} = \mathbf X_t - a \eta_t \mathbf O_t,
\end{equation}
where $\widetilde{\mathbf M}_t$ is the standard momentum or its Nesterov variant~\citep{jordan6muon,liu2025muon,chang2025convergence}; following LLM Muon practice~\citep{liu2025muon}, we set \(a=0.2\sqrt{\max\{m,n\}}\). In practice, $\operatorname{Orth}(\widetilde{\mathbf M}_t)$ is implemented by a fixed number of Newton--Schulz iterations rather than exact polar decomposition~\citep{jordan6muon}. 

% The design question is therefore not only what direction is desirable in the exact limit, but also what matrix should be fed into finite-step Newton--Schulz.

\textbf{Pre-orthogonalization equilibration family.}
{\method} inserts diagonal preconditioning before orthogonalization:
\begin{equation}
\hat{\mathbf M}_t = \textbf{DiagPre}(\widetilde{\mathbf M}_t,s,\varepsilon),
\qquad
\mathbf O_t = \operatorname{Orth}(\hat{\mathbf M}_t),
\qquad
\mathbf X_{t+1} = \mathbf X_t - a \eta_t \mathbf O_t,
\end{equation}
where the mode $s\in\{\mathrm{RC},\mathrm{R},\mathrm{C}\}$ selects one of three equilibration forms.
Let
$$
\mathbf D_{r,t} = \operatorname{diag}\bigl(\operatorname{rowsum}(\widetilde{\mathbf M}_t\odot \widetilde{\mathbf M}_t)+\varepsilon\bigr),\qquad
\mathbf D_{c,t} = \operatorname{diag}\bigl(\operatorname{colsum}(\widetilde{\mathbf M}_t\odot \widetilde{\mathbf M}_t)+\varepsilon\bigr).
$$
We consider the following three maps:
\begin{equation}
\label{eq:diagpre_family}
\textbf{DiagPre}(\widetilde{\mathbf M}_t,s,\varepsilon)=
\begin{cases}
\mathbf D_{r,t}^{-1/2}\widetilde{\mathbf M}_t\mathbf D_{c,t}^{-1/2}, & s=\mathrm{RC},\\
\mathbf D_{r,t}^{-1/2}\widetilde{\mathbf M}_t, & s=\mathrm{R},\\
\widetilde{\mathbf M}_t\mathbf D_{c,t}^{-1/2}, & s=\mathrm{C}.
\end{cases}
\end{equation}
% Unlike methods that rely on matrix-valued or adaptive preconditioners~\citep{gupta2018shampoo,Shazeer2018AdafactorAL,kingma2014adam,Xu2026FISMOFM}, {\method} changes only the matrix fed to Newton--Schulz. Relative to Muon~\citep{jordan6muon}, the orthogonalization routine itself is unchanged; relative to Muon+, NorMuon, and AdaMuon~\citep{zhang2026muon+,li2025normuon,si2025adamuon}, {\method} rebalances the input rather than rescales the output.
Unlike methods that rely on matrix-valued or adaptive preconditioners~\citep{gupta2018shampoo,Shazeer2018AdafactorAL,kingma2014adam,Xu2026FISMOFM}, {\method} changes only the transient input fed to Newton--Schulz. Relative to Muon~\citep{jordan6muon}, the orthogonalization routine and persistent optimizer state are unchanged; relative to Muon+, NorMuon, and AdaMuon~\citep{zhang2026muon+,li2025normuon,si2025adamuon}, {\method} rebalances the input rather than rescales the output.

\textbf{Why the default variant is R.}
This choice is block-dependent rather than a universal preference for row normalization. The trade-off is guided by
{\small
\begin{equation}
\label{eq:error_decomp}
\begin{aligned}
\left\|\mathrm{NS}_K(S(\mathbf M))-\mathrm{Orth}(\mathbf M)\right\|_F
\leq
\underbrace{\left\|\mathrm{NS}_K(S(\mathbf M))-\mathrm{Orth}(S(\mathbf M))\right\|_F}_{\text{finite-step approximation error}}
+\underbrace{\left\|\mathrm{Orth}(S(\mathbf M))-\mathrm{Orth}(\mathbf M)\right\|_F}_{\text{preconditioning bias}},
\end{aligned}
\end{equation}
}

where $S$ is one of RC, R, or C. The first term is the finite-step Newton--Schulz error and depends on the spectrum of $S(\mathbf M)$; Theorem~\ref{th_ns} therefore explains why the more aggressive two-sided RC map often gives the largest spectral correction. The second term is the bias induced by preprocessing. Proposition~\ref{th_zero_rc} shows that row/column normalization is a zeroth-order surrogate for whitening, removing marginal scale mismatch before orthogonalization. Among the one-sided variants, recent operator-geometry analyses identify row-sided geometry as the natural choice for the hidden Transformer weight matrices targeted by {\method}, whereas column-sided geometry is more closely associated with embedding matrices~\citep{Xu2026moga}. This makes R the most appropriate default in our setting. Consistent with this picture, Proposition~\ref{th_conv_rc} characterizes the pure RC regime as a stronger two-sided spectral corrector, whereas Theorem~\ref{th_conv_r} gives the cleanest nonconvex guarantee for the row-normalized variant. 

\begin{table}[t]
\vspace{-10pt}
\centering
\setlength{\tabcolsep}{0pt}
\begin{tabular}{@{}p{0.5\textwidth}@{\hspace{0.02\textwidth}}p{0.46\textwidth}@{}}
\begin{minipage}[t]{0.5\textwidth}
\begin{algorithm}[H]
\caption{{\method} (mode $s\in\{\mathrm{RC},\mathrm{R},\mathrm{C}\}$; default $s=\mathrm{R}$)}
\label{alg:muoneq}
\begin{algorithmic}[1]
\STATE \textbf{Input:} Initial parameters $\mathbf X_1\in\mathbb R^{m\times n}$, learning rates $\eta_t>0$, momentum coefficients $\{\beta_t\}_{t\ge1}\subset[0,1)$, weight decay $\lambda_t\ge 0$, mode $s\in\{\mathrm{RC},\mathrm{R},\mathrm{C}\}$, stability term $\varepsilon\ge0$, Nesterov flag, scaling constant $a=0.2\sqrt{\max(m,n)}$.
\STATE Initialize $\mathbf M_0=\mathbf 0$
\FOR{$t=1$ \textbf{to} $T$}
    \STATE $\mathbf G_t = \nabla f(\mathbf X_t;\xi_t)$
    \STATE $\mathbf M_t = \beta_t \mathbf M_{t-1} + (1-\beta_t)\mathbf G_t$
    \IF{Nesterov}
        \STATE $\widetilde{\mathbf M}_t = \beta_{t+1} \mathbf M_t + (1-\beta_{t+1})\mathbf G_t$
    \ELSE
        \STATE $\widetilde{\mathbf M}_t = \mathbf M_t$
    \ENDIF
    \STATE $\hat{\mathbf M}_t = \textbf{DiagPre}(\widetilde{\mathbf M}_t,s,\varepsilon)$
    \STATE $\mathbf O_t = \mathrm{NS5}(\hat{\mathbf M}_t)$
    \STATE $\mathbf X_{t+1} = (1-\lambda_t\eta_t)\mathbf X_t - a\eta_t\mathbf O_t$
\ENDFOR
\end{algorithmic}
\end{algorithm}
\end{minipage}
&
\begin{minipage}[t]{0.46\textwidth}
\begin{algorithm}[H]
\caption{\textbf{DiagPre}}
\label{alg:diagpre}
\begin{algorithmic}[1]
\STATE \textbf{Input:} Momentum $\widetilde{\mathbf M}_t$, mode $s\in\{\mathrm{RC},\mathrm{R},\mathrm{C}\}$, stability term $\varepsilon\ge0$
\IF{$s\in\{\mathrm{RC},\mathrm{R}\}$}
    \STATE $\mathbf D_{r,t} = \operatorname{diag}(\operatorname{rowsum}(\widetilde{\mathbf M}_t \odot \widetilde{\mathbf M}_t) + \varepsilon)$
\ENDIF
\IF{$s\in\{\mathrm{RC},\mathrm{C}\}$}
    \STATE $\mathbf D_{c,t} = \operatorname{diag}(\operatorname{colsum}(\widetilde{\mathbf M}_t \odot \widetilde{\mathbf M}_t) + \varepsilon)$
\ENDIF
\STATE $\hat{\mathbf M}_t=\left\{\begin{array}{ll}
    \mathbf D_{r,t}^{-1/2}\widetilde{\mathbf M}_t\mathbf D_{c,t}^{-1/2}, & \hbox{$s=\mathrm{RC}$;} \\
    \mathbf D_{r,t}^{-1/2}\widetilde{\mathbf M}_t , & \hbox{$s=\mathrm{R}$;}\\
    \widetilde{\mathbf M}_t\mathbf D_{c,t}^{-1/2}, &\hbox{$s=\mathrm{C}$.}
  \end{array}
  \right.$
\RETURN $\hat{\mathbf M}_t$
\end{algorithmic}
\end{algorithm}
\vspace{-10pt}
% \begin{remark}
% Unless otherwise stated, we use the row-normalized variant $s=\mathrm{R}$. The RC and C modes are kept as lightweight forms for analysis and ablation. 
% When $\beta_t\equiv\beta$, this recovers the standard Nesterov form.
% \end{remark}
\begin{remark}
Unless otherwise stated, we use \(s=\mathrm{R}\); \(\mathrm{RC}\) and
\(\mathrm{C}\) are retained for analysis and ablation. 
% For varying \(\beta_t\),
% the Nesterov line follows the EMA-buffer convention, with effective coefficient
% \(\beta_{t+1}(1-\beta_t)/(1-\beta_{t+1})\); for \(\beta_t\equiv\beta\), this
% recovers the standard form.
For varying \(\beta_t\),
the Nesterov line follows the EMA-buffer convention; for
\(\beta_t\equiv\beta\), it reduces to the Nesterov form~\cite{yuan2024mars,chang2025convergence}. 
\end{remark}
\end{minipage}
\end{tabular}
\label{tab:algorithms}
\vspace{-8pt}
\end{table}
\section{Analysis}
\label{sec:analysis}

The analysis follows the logic of {\method}: Section~\ref{sec:ana_1} establishes the spectrum sensitivity of finite-step Newton--Schulz; Section~\ref{sec:ana_2} interprets row/column normalization as a zeroth-order surrogate for whitening; and Section~\ref{sec:ana_3} motivates the R rather than RC, with row normalization admitting a clean stochastic nonconvex guarantee. 
Together, these results motivate the {\method} family and clarify the trade-off between stronger spectral correction and preprocessing bias. Within the hidden-weight setting studied here, they support using R as our default working variant.
% Together, Theorem~\ref{th_ns}, Proposition~\ref{th_zero_rc}, Proposition~\ref{th_conv_rc}, and Theorem~\ref{th_conv_r} justify the {\method} family and the default choice of R.

\subsection{Why spectral geometry matters for finite-step Newton--Schulz orthogonalization}
\label{sec:ana_1}

For clarity, the theorem below is stated for the wide case $p\le q$; the tall case follows by transposition.

\begin{theorem}
\label{th_ns}
Let $\mathbf{G}\in\mathbb R^{p\times q}$ have rank $r\ge1$ and compact SVD
$\mathbf{G}=\mathbf{U}\Sigma\mathbf{V}^\top$, where
$\Sigma=\operatorname{diag}(\sigma_1,\ldots,\sigma_r)$ with
$\sigma_1\ge\cdots\ge\sigma_r>0$ and $p\le q$.
Consider the Newton--Schulz iteration
\begin{equation*}
\scalebox{0.9}{$\displaystyle
\mathbf{X}_0=\frac{\mathbf{G}}{\|\mathbf{G}\|_F},\quad
\mathbf{X}_{k+1}=\left(
a\mathbf I_p
+b\mathbf{X}_k\mathbf{X}_k^\top
+c(\mathbf{X}_k\mathbf{X}_k^\top)^2
\right)\mathbf{X}_k,
$}
\end{equation*}
where $\phi(s):=as+bs^3+cs^5$ and $q(t):=a+bt+ct^2$ satisfy $a>1, 0<q(t)\le a,t\in[0,1].$
Such coefficient choices include those used in
\citep{jordan6muon,kim2026convergence}.
Then, for each $k\ge0$,
\[
\mathbf{X}_k=\mathbf{U}
\operatorname{diag}\bigl(s_1^{(k)},\ldots,s_r^{(k)}\bigr)
\mathbf{V}^\top,
\qquad
s_i^{(0)}=\frac{\sigma_i}{\|\mathbf{G}\|_F},
\qquad
s_i^{(k+1)}=\phi\bigl(s_i^{(k)}\bigr).
\]
Moreover,
\begin{equation*}
\scalebox{0.8}{$\displaystyle
\frac{1}{\sqrt r}
\bigl\|
\mathbf{X}_k-\operatorname{Orth}(\mathbf{G})
\bigr\|_F
\ge
\frac{1}{\sqrt r}
\left(
\sum_{i=1}^r
\left(
1-\frac{a^k}{\kappa_i\sqrt{\operatorname{sr}(\mathbf{G})}}
\right)_+^2
\right)^{1/2}.$
}
\end{equation*}
Equivalently, the $i$-th singular direction leaves the linear regime at \scalebox{0.85}{$\tau_i =\log_a\left(
\kappa_i\sqrt{\operatorname{sr}(\mathbf{G})}
\right)$}. In particular, \scalebox{0.85}{$\tau_1=\log_a\sqrt{\operatorname{sr}(\mathbf{G})},\tau_r=\log_a\left(
\kappa(\mathbf{G})\sqrt{\operatorname{sr}(\mathbf{G})}
\right),\tau_r-\tau_1=\log_a\kappa(\mathbf{G})$}.
See Appendix~\ref{proof:th_ns} for details.
\end{theorem}

\begin{remark}
\label{remark:cond}
Theorem~\ref{th_ns} shows that the stable rank controls the onset of the
linear to nonlinear transition, while the condition number controls its width
across singular directions.
\end{remark}

\begin{figure}[!htbp]
\centering

\subfloat[1024$\times$1024]{%
\begin{minipage}[b]{0.24\textwidth}
\centering
\includegraphics[width=\linewidth]{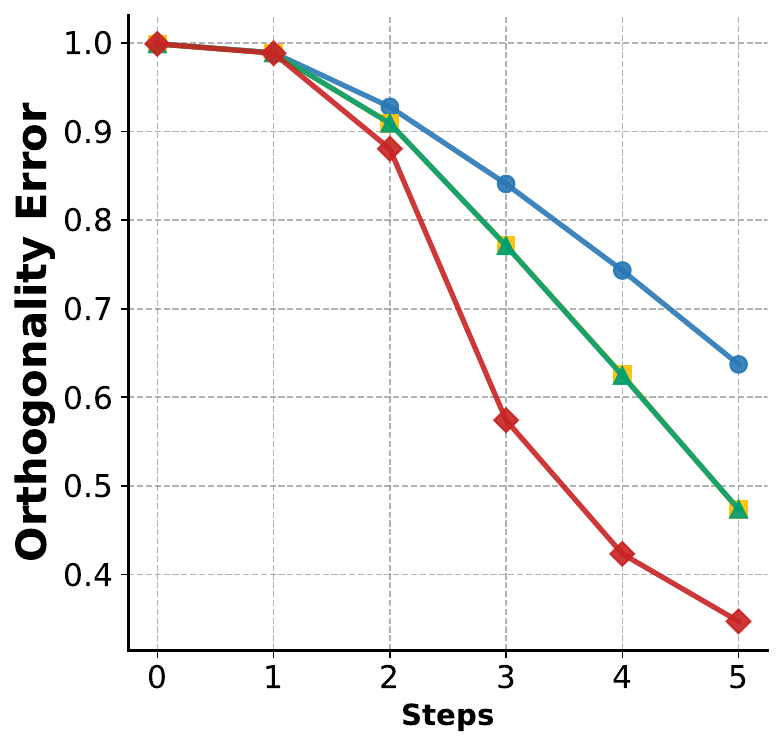}\par\vspace{-3pt}
\includegraphics[width=\linewidth]{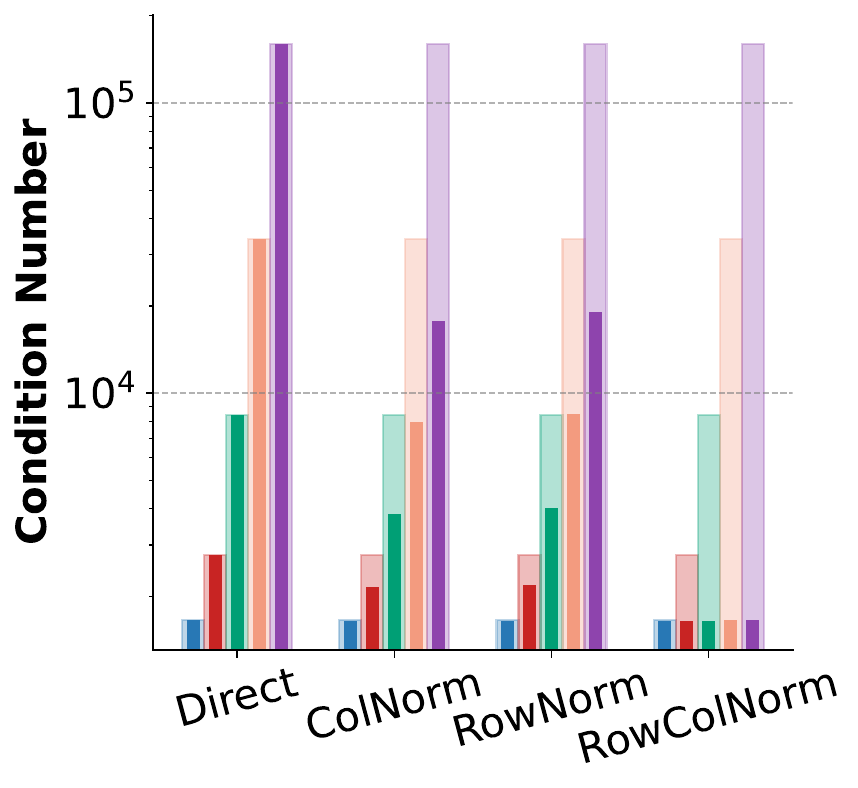}
\end{minipage}
\label{Fig:FG1}%
}
\subfloat[1024$\times$4096]{%
\begin{minipage}[b]{0.23\textwidth}
\centering
\includegraphics[width=\linewidth]{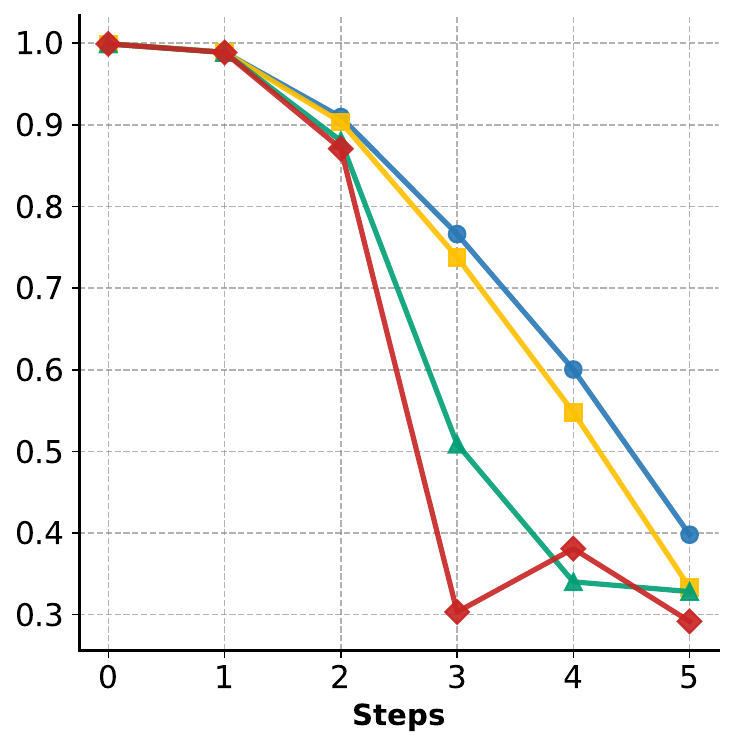}\par\vspace{-3pt}
\includegraphics[width=\linewidth]{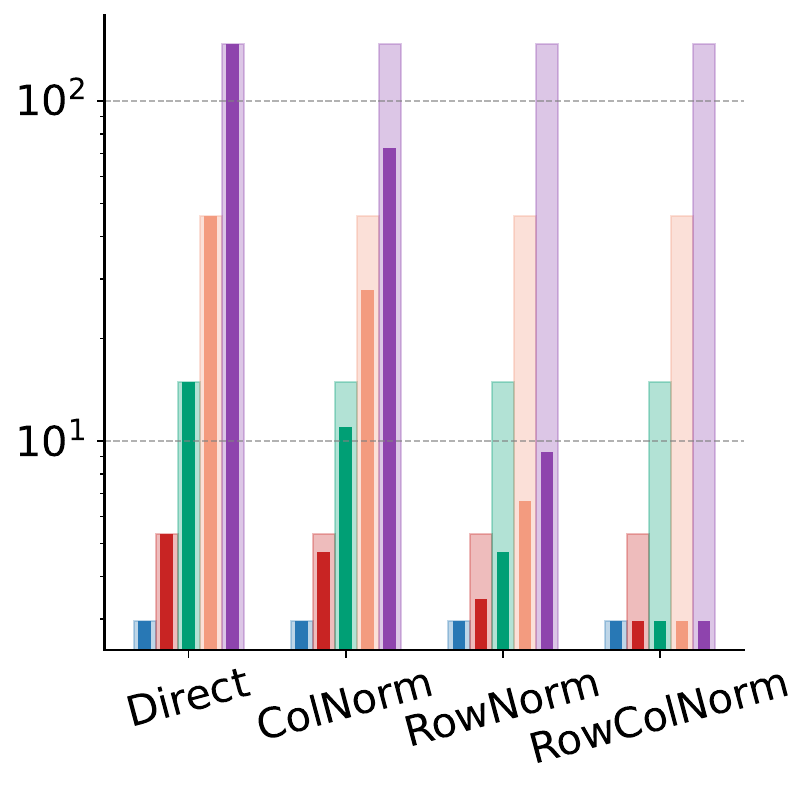}
\end{minipage}
\label{Fig:FG2}%
}
\subfloat[4096$\times$1024]{%
\begin{minipage}[b]{0.23\textwidth}
\centering
\includegraphics[width=\linewidth]{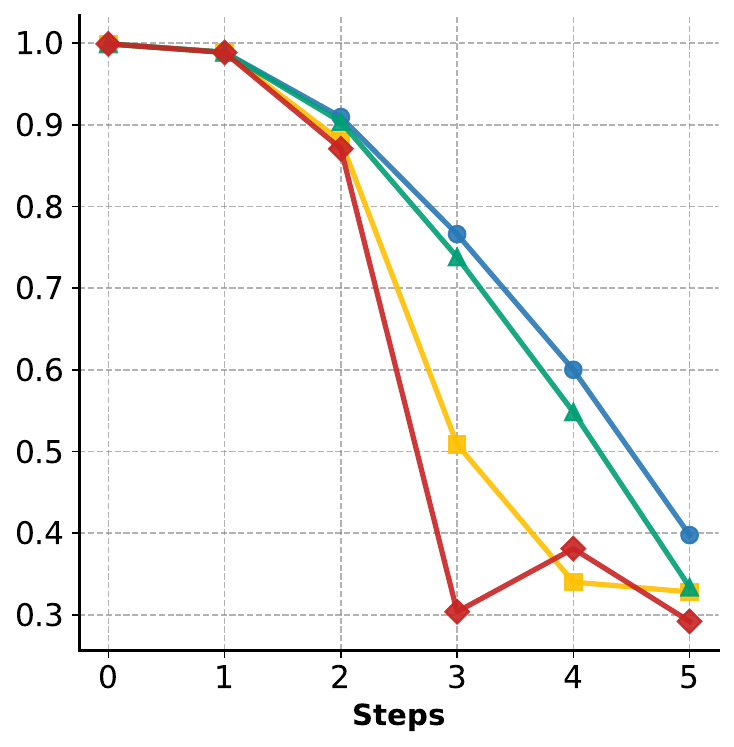}\par\vspace{-3pt}
\includegraphics[width=\linewidth]{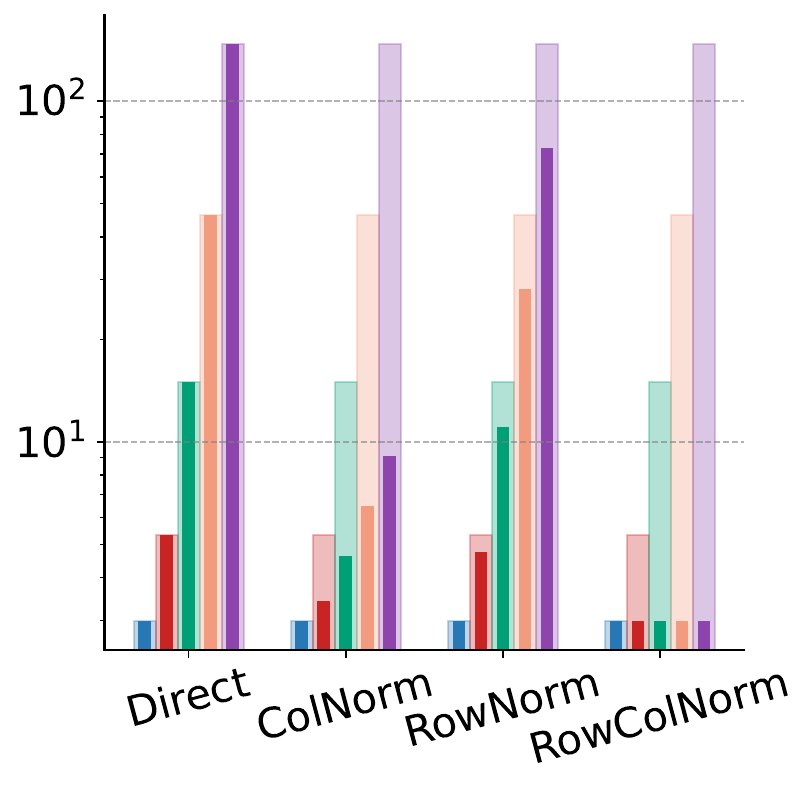}
\end{minipage}
\label{Fig:FG3}%
}
\subfloat[4096$\times$4096]{%
\begin{minipage}[b]{0.23\textwidth}
\centering
\includegraphics[width=\linewidth]{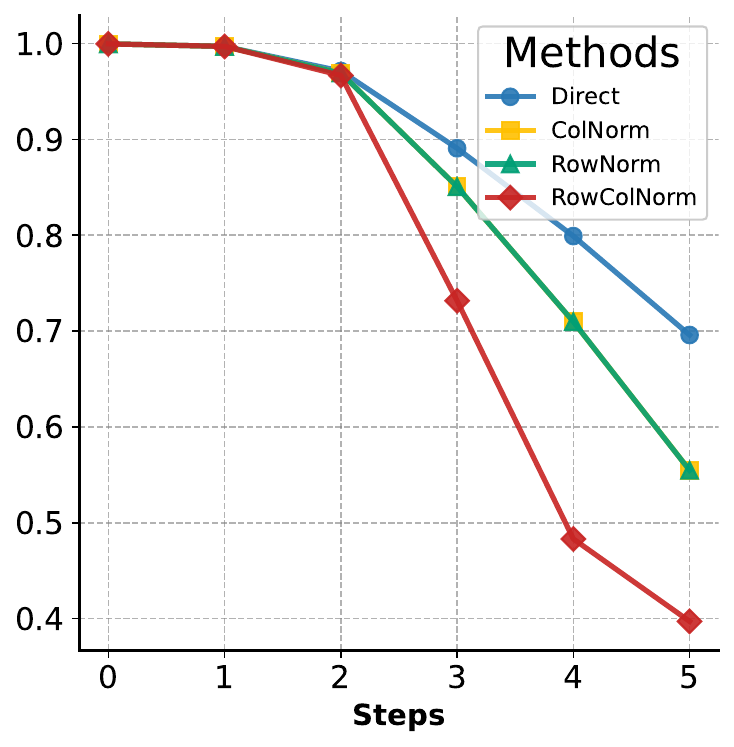}\par\vspace{-3pt}
\includegraphics[width=\linewidth]{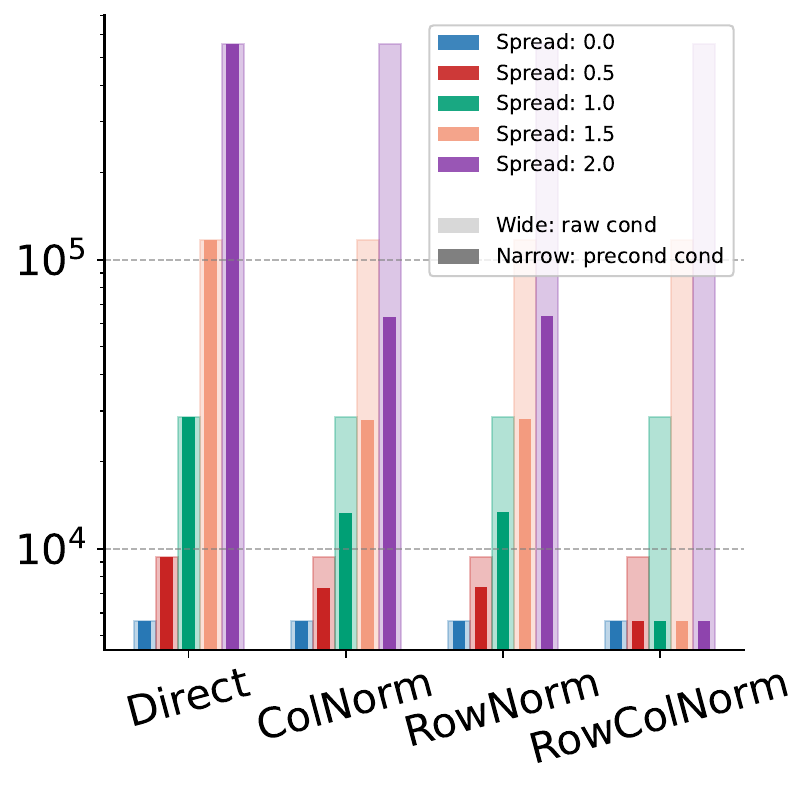}
\end{minipage}
\label{Fig:FG4}%
}
\caption{Random Gaussian matrices with controlled shapes and spectral spreads. Top: finite-step relative Frobenius error to the exact polar factor across Newton--Schulz steps. Bottom: raw and post-normalization condition numbers. Two-sided row/column normalization yields the smallest error and the most consistent spectral compression.}
\label{fig1}
\end{figure}

% Theorem~\ref{th_ns} shows that broad spectra slow finite-step orthogonalization by delaying contraction along small singular directions. 
To isolate this effect, Figure~\ref{fig1} uses random Gaussian matrices with controlled shapes and spectral spreads. The top row shows the finite-step error in Theorem~\ref{th_ns}, and the bottom row shows the corresponding condition numbers before and after normalization.

To verify the same mechanism on real training trajectories, Figure~\ref{fig3} uses pre-NS momentum matrices from a 2.6B-token Muon run on LLaMA2-130M~\cite{Touvron2023Llama2O} trained on C4~\citep{raffel2020exploring}. The top row shows the module-wise median of the error in Theorem~\ref{th_ns}, and the bottom row shows the corresponding mean. Appendix~\ref{appendix:more_res} provides a more detailed module-wise view of the same phenomenon: Figure~\ref{fig:spectral_metrics} reports singular-value entropy and stable-rank diagnostics for the pre-NS momentum matrices, showing that equilibration reshapes the spectrum seen by finite-step Newton--Schulz.

\begin{figure}[!htbp]
\centering
\subfloat[Step 1\%]{%
\begin{minipage}[b]{0.23\textwidth}
\centering
\includegraphics[width=\linewidth]{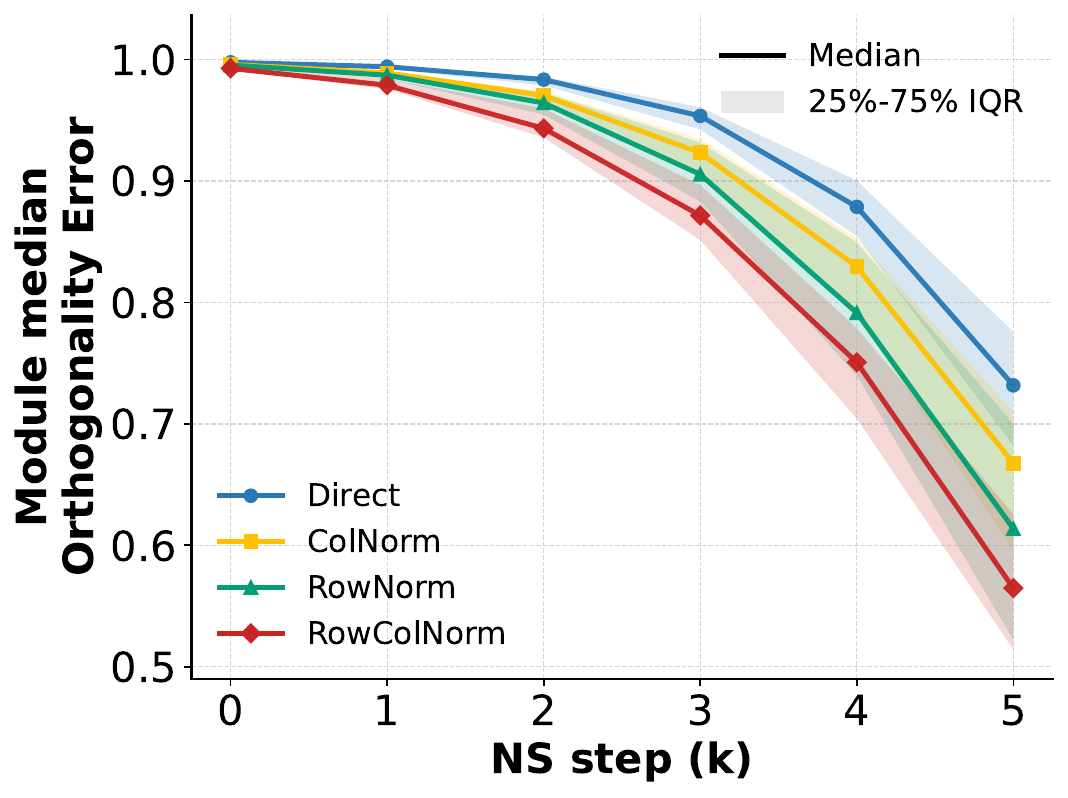}\par\vspace{-3pt}
\includegraphics[width=\linewidth]{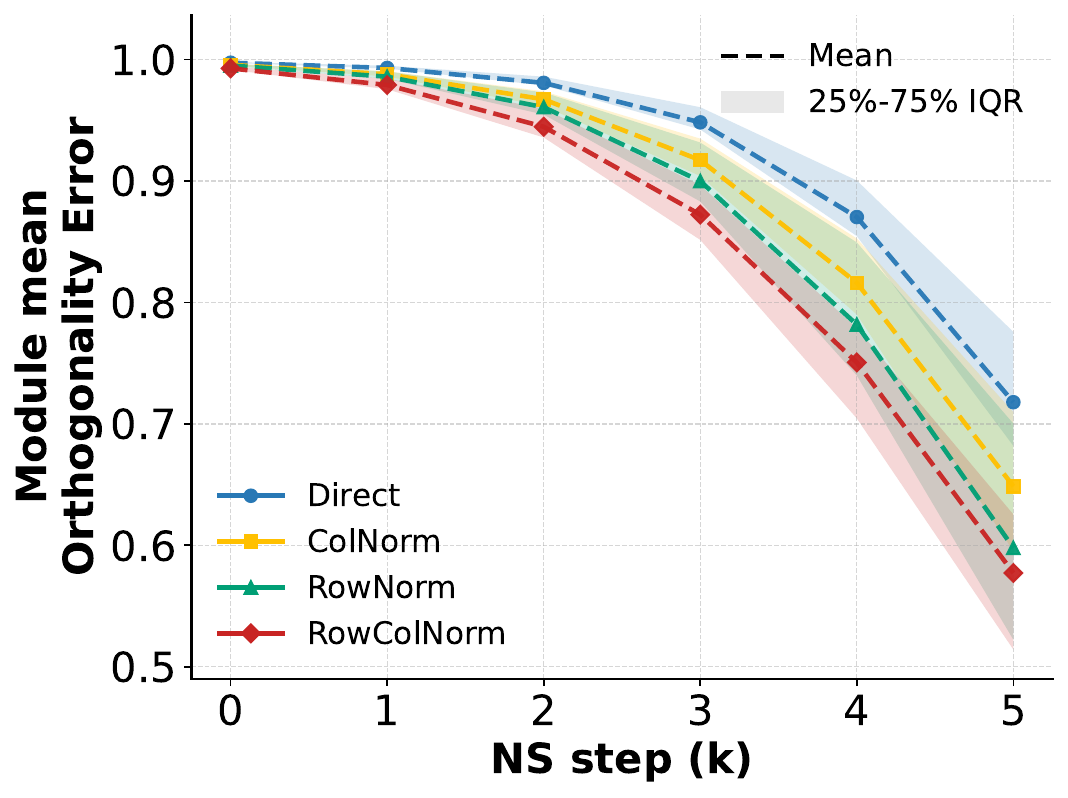}
\end{minipage}
\label{Fig:FG5}%
}
\subfloat[Step 10\%]{%
\begin{minipage}[b]{0.23\textwidth}
\centering
\includegraphics[width=\linewidth]{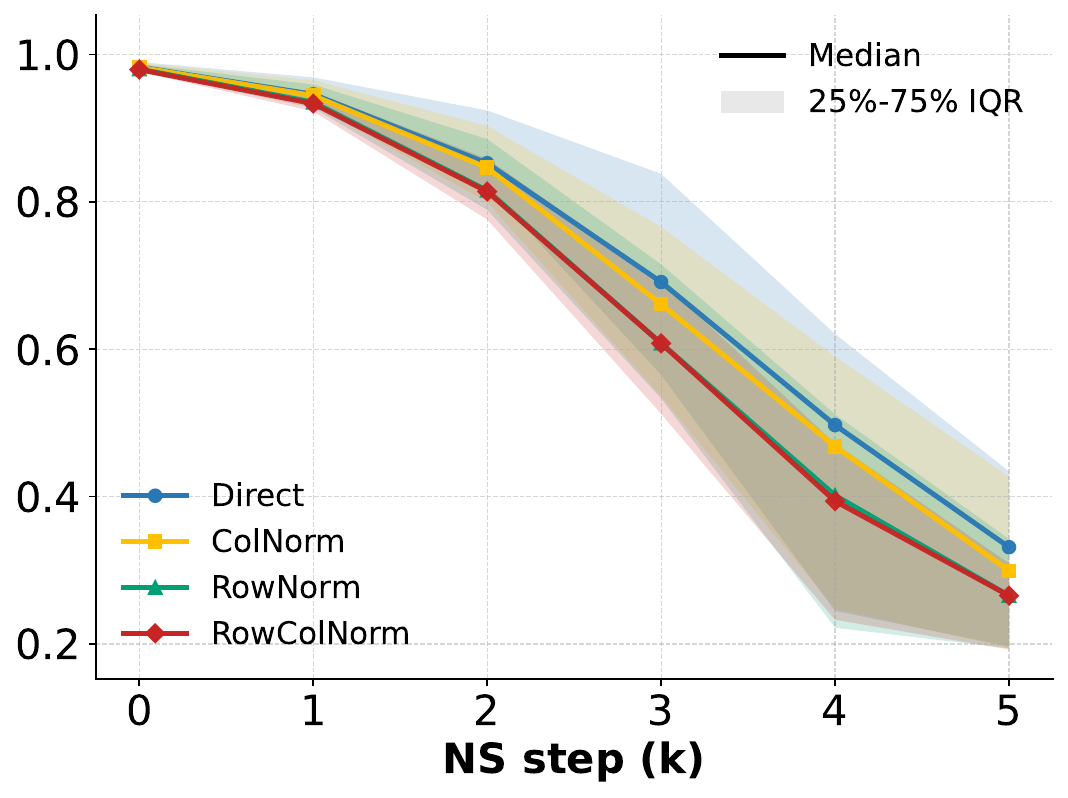}\par\vspace{-3pt}
\includegraphics[width=\linewidth]{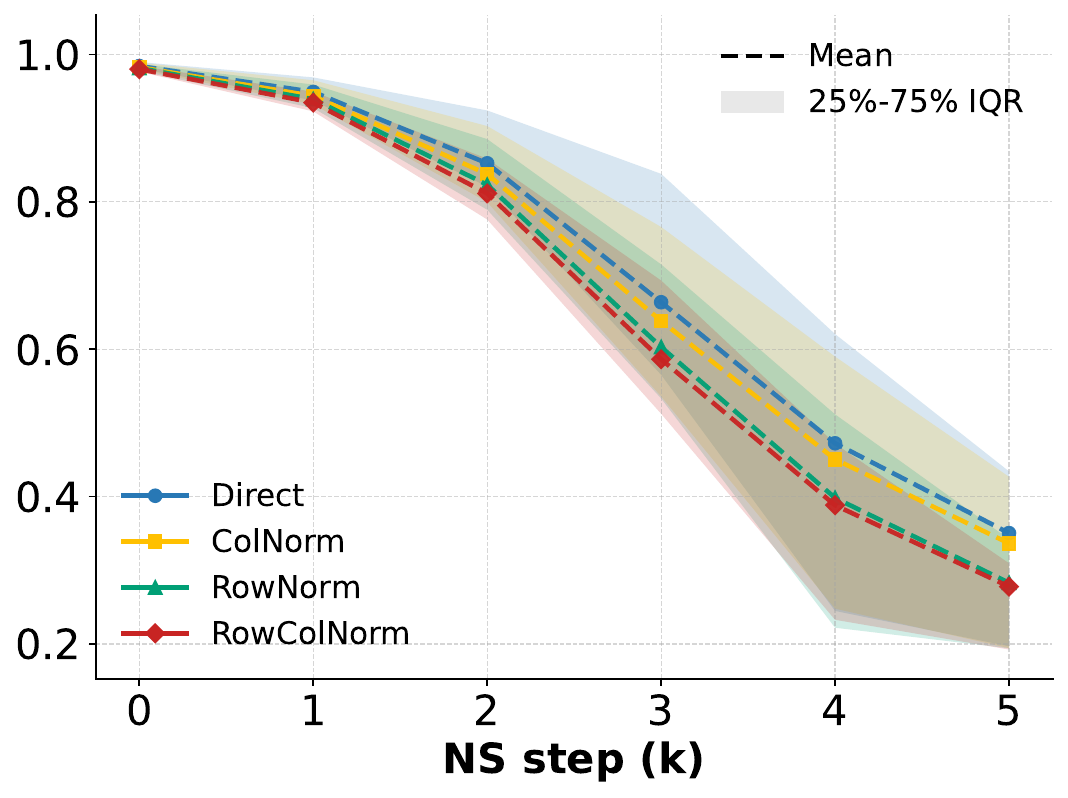}
\end{minipage}
\label{Fig:FG6}%
}
\subfloat[Step 50\%]{%
\begin{minipage}[b]{0.23\textwidth}
\centering
\includegraphics[width=\linewidth]{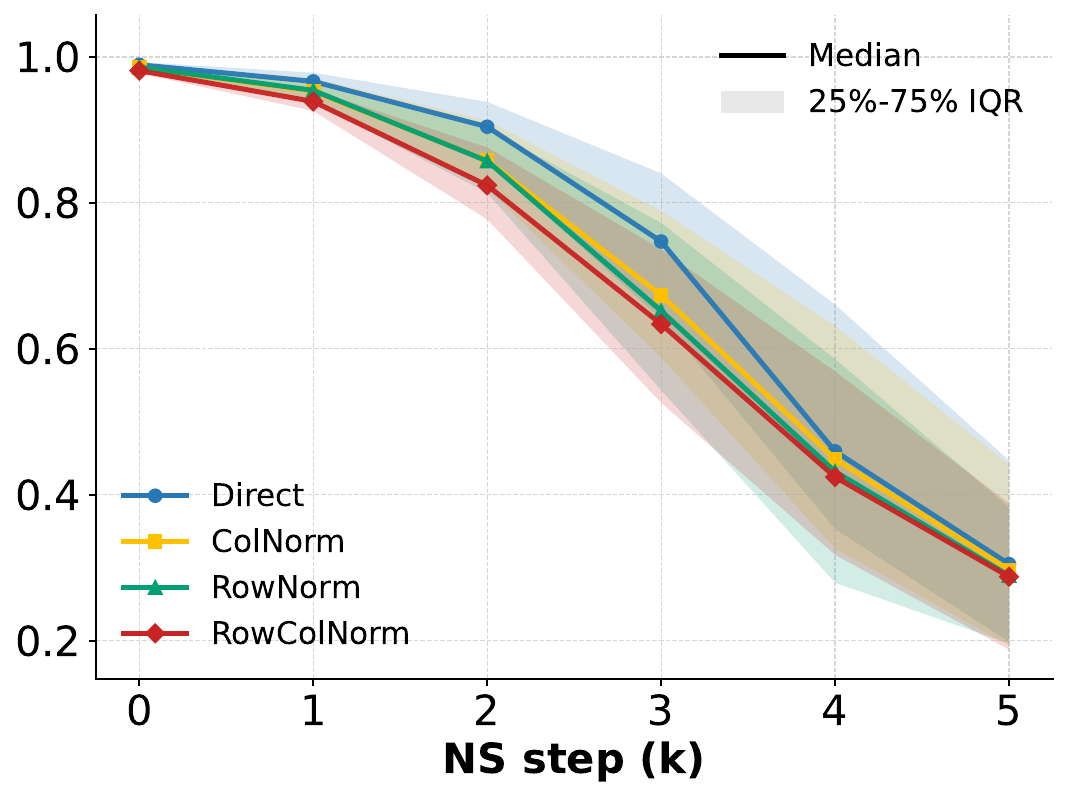}\par\vspace{-3pt}
\includegraphics[width=\linewidth]{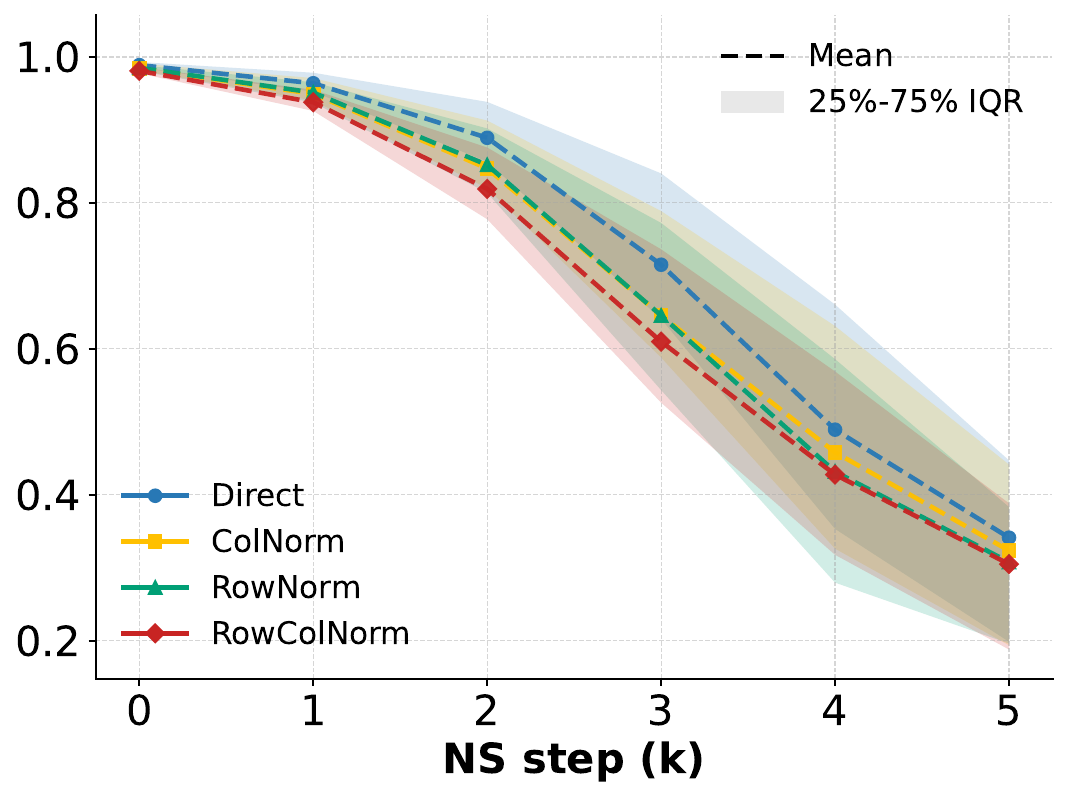}
\end{minipage}
\label{Fig:FG7}%
}
\subfloat[Step 100\%]{%
\begin{minipage}[b]{0.23\textwidth}
\centering
\includegraphics[width=\linewidth]{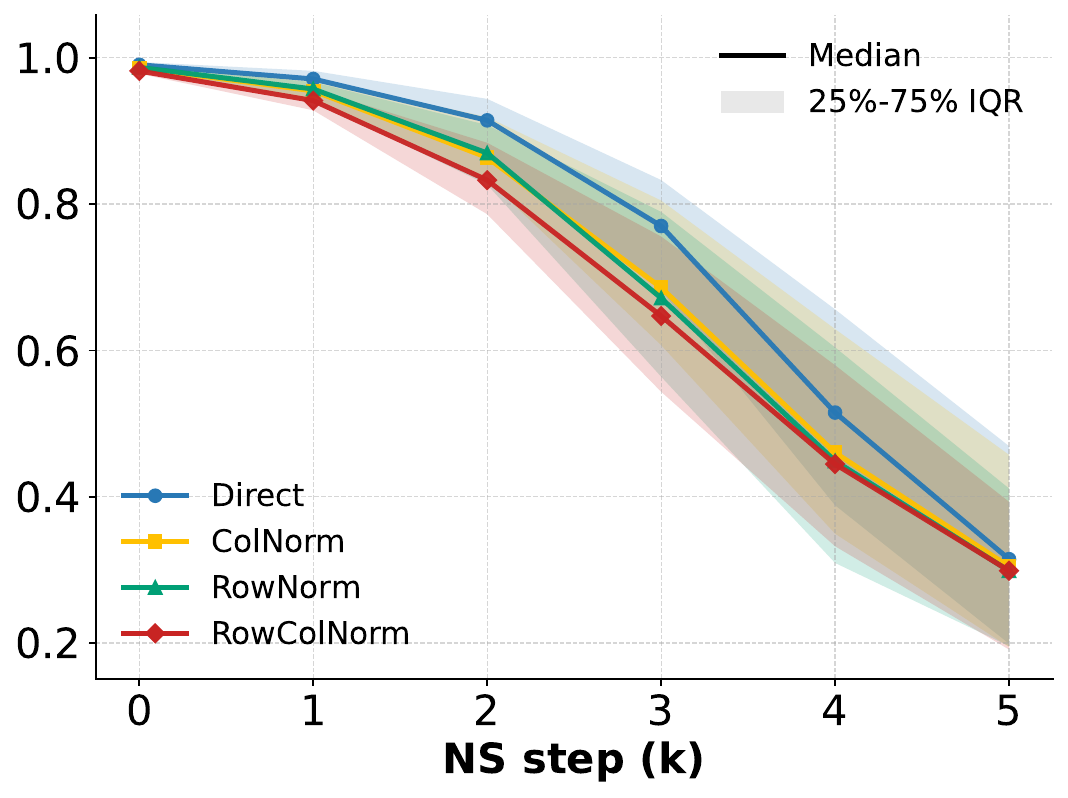}\par\vspace{-3pt}
\includegraphics[width=\linewidth]{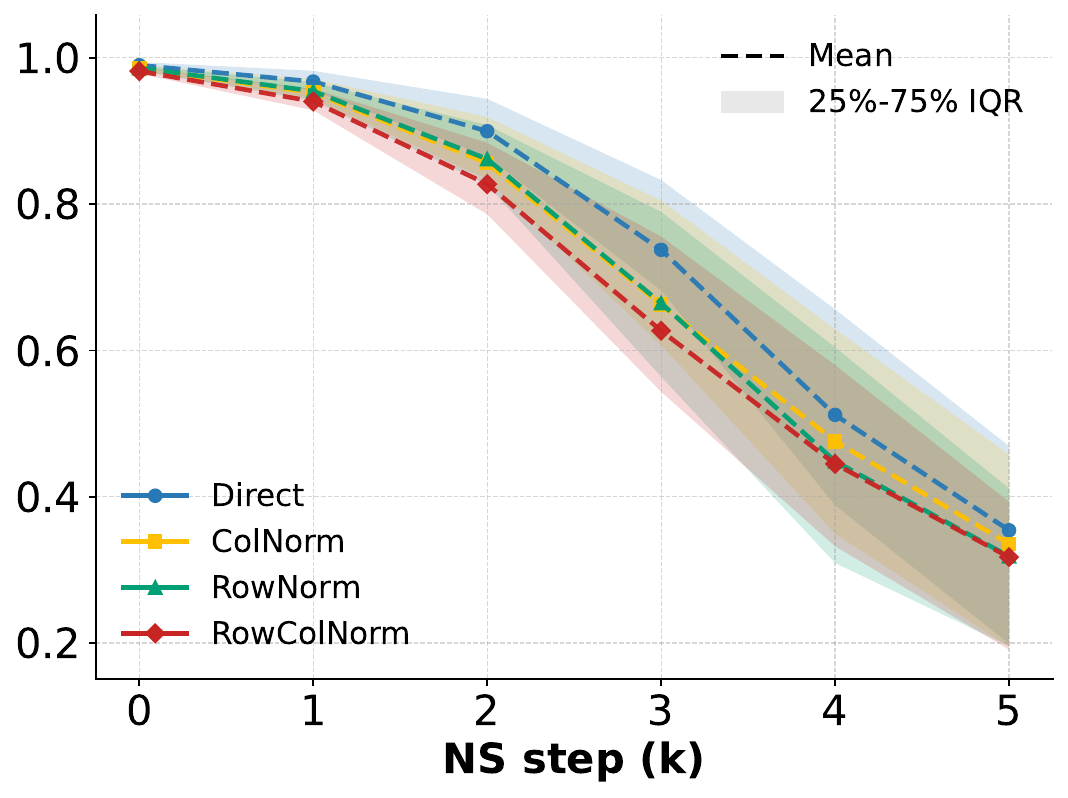}
\end{minipage}
\label{Fig:FG8}%
}
\caption{Finite-step orthogonalization error across Newton--Schulz steps at 1\%, 10\%, 50\%, and 100\% of training. In each column, the top panel shows the module-wise median and the bottom panel shows the module-wise mean; shaded bands denote the 25\%--75\% range. Two-sided row/column normalization decays fastest.}
\label{fig3}
\end{figure}
\vspace{-5pt}
\subsection{Row/column normalization as a zeroth-order surrogate for whitening}
\label{sec:ana_2}

The next result shows that normalization is not merely heuristic rescaling. It is the leading-order term of exact whitening/orthogonalization when the normalized Gram matrix is close to identity. Consequently, normalization removes marginal scale mismatch first and leaves only the correlation correction to the polar step.

\begin{proposition}
\label{th_zero_rc}
The following first-order expansions hold.

\textit{(i) Column/right whitening.} 

Let $\mathbf{M}\in\mathbb R^{p\times q}$ have full column rank, and set
$\mathbf{D}_c:=\operatorname{diag}(\|\mathbf{M}_{:,1}\|_2,\ldots,\|\mathbf{M}_{:,q}\|_2)\succ 0$,
$\mathbf N_c:=\mathbf{M}\mathbf{D}_c^{-1}$, and
$\mathbf C_c:=\mathbf N_c^\top\mathbf N_c-\mathbf I_q$.
If $\|\mathbf{D}_c\mathbf C_c\mathbf{D}_c\|_2$ is sufficiently small, then
$\operatorname{Orth}_r(\mathbf{M})
:=\mathbf{M}(\mathbf{M}^\top\mathbf{M})^{-1/2}
=\mathbf N_c-\mathbf N_c\mathbf L_c\mathbf{D}_c^{-1}
+\mathcal O\bigl(\|\mathbf{D}_c\mathbf C_c\mathbf{D}_c\|_2^2\bigr),$ where $\mathbf L_c$ is the unique solution of $\mathbf{D}_c\mathbf L_c+\mathbf L_c\mathbf{D}_c=\mathbf{D}_c\mathbf C_c\mathbf{D}_c.$
In particular,
\[
\begin{aligned}
\operatorname{Orth}_r(\mathbf{M})=\mathbf N_c+\mathcal O\bigl(\|\mathbf C_c\|_2\bigr).
\end{aligned}
\]
\textit{(ii) Row/left whitening.} 

Let $\mathbf{M}\in\mathbb R^{p\times q}$ have full row rank, and set
$\mathbf{D}_r:=\operatorname{diag}(\|\mathbf{M}_{1,:}\|_2,\ldots,\|\mathbf{M}_{p,:}\|_2)\succ 0$,
$\mathbf N_r:=\mathbf{D}_r^{-1}\mathbf{M}$, and
$\mathbf C_r:=\mathbf N_r\mathbf N_r^\top-\mathbf I_p$.
If $\|\mathbf{D}_r\mathbf C_r\mathbf{D}_r\|_2$ is sufficiently small, then
$\operatorname{Orth}_\ell(\mathbf{M})
:=(\mathbf{M}\mathbf{M}^\top)^{-1/2}\mathbf{M}
=\mathbf N_r-\mathbf{D}_r^{-1}\mathbf L_r\mathbf N_r
+\mathcal O\bigl(\|\mathbf{D}_r\mathbf C_r\mathbf{D}_r\|_2^2\bigr),$ where $\mathbf L_r$ is the unique solution of $\mathbf{D}_r\mathbf L_r+\mathbf L_r\mathbf{D}_r=\mathbf{D}_r\mathbf C_r\mathbf{D}_r.$
In particular,
\[
\begin{aligned}
\operatorname{Orth}_\ell(\mathbf{M})=\mathbf N_r+\mathcal O\bigl(\|\mathbf C_r\|_2\bigr).
\end{aligned}
\]
See Appendix~\ref{proof:th_zero_rc} for details.
\end{proposition}

\begin{corollary}
\label{cor:row_norm_main}
Under the assumptions of the row/left part of Proposition~\ref{th_zero_rc}, if $\|\mathbf C_r\|_2$ is sufficiently small, then
\[
\begin{aligned}
\operatorname{Orth}_\ell(\mathbf N_r)
=
(\mathbf I_p+\mathbf C_r)^{-1/2}\mathbf N_r
=
\mathbf N_r-\frac12\mathbf C_r\mathbf N_r+\mathcal O\bigl(\|\mathbf C_r\|_2^2\bigr),
\end{aligned}
\]
and hence
\[
\begin{aligned}
\|\operatorname{Orth}_\ell(\mathbf N_r)-\mathbf N_r\|_2
\le
\frac12\|\mathbf N_r\|_2\|\mathbf C_r\|_2
+\mathcal O\bigl(\|\mathbf C_r\|_2^2\bigr).
\end{aligned}
\]
In particular, since $\operatorname{diag}(\mathbf N_r\mathbf N_r^\top)=\mathbf I_p$, row normalization removes the marginal scale mismatch, so the leading whitening correction acts only on the residual off-diagonal Gram error.
\end{corollary}

\begin{remark}
\label{remark:gap}
By contrast, the corresponding bound for $\operatorname{Orth}_\ell(\mathbf{M})-\mathbf{M}$ contains the additional zeroth-order term
\[
\begin{aligned}
\|\mathbf{D}_r-\mathbf I_p\|_2\|\mathbf N_r\|_2.
\end{aligned}
\]
Thus row normalization separates marginal rescaling from the genuine whitening step. The column-normalized/right analogue and the corresponding two-sided normalization statements are completely symmetric; see Appendix~\ref{appendix:normalization}.
\end{remark}

Together, Proposition~\ref{th_zero_rc}, Corollary~\ref{cor:row_norm_main}, and Remark~\ref{remark:gap} explain the division of labor in {\method}: normalization removes marginal-scale mismatch at zeroth order, and the subsequent Newton--Schulz orthogonalization handles the remaining correlation structure. This is exactly why \textbf{DiagPre} Eq. \eqref{eq:diagpre_family} is a meaningful lightweight surrogate for a much more expensive whitening step.

\subsection{Convergence analysis of the RC and R variants}
\label{sec:ana_3}
To analyze the {\method} family, we focus on the two variants that are most informative theoretically: RC and the default row-normalized R. RC gives stronger two-sided spectral correction but leads to a more involved alignment term, while R has a cleaner one-sided geometry. Accordingly, Proposition~\ref{th_conv_rc} studies RC. For R, Theorem~\ref{th_conv_r} gives the main guarantee for the exact-polar update, and Corollary~\ref{cor:conv_r_ns} extends it to the finite-step NS5 implementation. Compared with recent Muon analyses~\cite{chang2025convergence,Sato2025ConvergenceBA,shen2025convergence,kim2026convergence,shulgin2025beyond,nagashima2026improved}, our contribution is not to improve the asymptotic rate, but to establish a guarantee closer to the optimizer setting used in practice: the same $\widetilde{\mathcal{O}}(T^{-1/4})$ stationarity rate is proved while explicitly covering decoupled weight decay, a learning-rate schedule independent of the training horizon, and finite-step NS5 through the inexactness constant $\varepsilon_{\mathrm{ns}}$.
% with the mean spectral-norm discrepancy between NS5 and the exact polar factor validated in Figure~\ref{fig:ns5_spec_norm}.
% When $\varepsilon=0$ in the R mode, we write 

\begin{assumption}
\label{ass:1}
The objective is bounded below: there exists $f^{\star}>-\infty$ such that $f(\mathbf{X})\ge f^{\star}$ for all $\mathbf{X}\in\mathbb R^{m\times n}$.
\end{assumption}

\begin{assumption}
\label{ass:2}
The function $f$ is $L$-smooth, that is, $\|\nabla f(\mathbf{Y})-\nabla f(\mathbf{X})\|_F\le L\|\mathbf{Y}-\mathbf{X}\|_F$
for all $\mathbf{X},\mathbf{Y}\in\mathbb R^{m\times n}$.
\end{assumption}

\begin{assumption}
\label{ass:3}
The stochastic gradient is unbiased with bounded variance. Specifically, there exists $\sigma>0$ such that, for all $\mathbf{X}\in\mathbb R^{m\times n}$, $\mathbb E[\nabla f(\mathbf{X};\xi)]=\nabla f(\mathbf{X}), \mathbb E\|\nabla f(\mathbf{X};\xi)-\nabla f(\mathbf{X})\|_F^2\le \sigma^2.$
\end{assumption}

These assumptions are standard in analyses of adaptive and Muon-type methods~\citep{xie2024adan,chen2025muon,DBLP:conf/nips/LiRJ23,DBLP:conf/nips/WangFZ0023,chang2025mgup,li2026on,huang2025limuon}.

\begin{proposition}
\label{th_conv_rc}
Let $\eta_t=t^{-3/4}$, $\beta_t=1-t^{-1/2}$, $h_t:=1-\beta_t=t^{-1/2}$, and $a:=0.2\sqrt{\max\{m,n\}}$. 
Assume Assumptions~\ref{ass:1}, \ref{ass:2}, and \ref{ass:3}. Moreover, assume there exists $G_\infty>0$ such that $\|\mathbf{G}_t\|_{\infty}\le G_\infty$ almost surely. 
For Algorithm \ref{alg:muoneq}, consider the exact-polar RC variant in \eqref{eq:diagpre_family}, with Nesterov momentum disabled and $\lambda_t=0$.  Suppose further that $\varepsilon \ge \frac{4}{5}G_\infty^2\max\{m,n\}$, and define $\rho_r:=\sqrt{1+\frac{nG_\infty^2}{\varepsilon}},
\rho_c:=\sqrt{1+\frac{mG_\infty^2}{\varepsilon}},
\chi_\varepsilon:= \frac{(\rho_r+1)(\rho_c+1)}{4\rho_r\rho_c}
-\frac{3\rho_r\rho_c-\rho_r-\rho_c-1}{4}.$
Then $\rho_r,\rho_c\le 3/2$ and hence $\chi_\varepsilon\ge 1/144$. Consequently,
\[
\begin{aligned}
\frac1T\sum_{t=1}^T \mathbb E\|\nabla f(\mathbf{X}_t)\|_F
\le
\frac{f(\mathbf{X}_1)-f^{\star}+C_1(1+\ln T)+C_2}
{a\chi_\varepsilon \cdot T^{1/4}},
\end{aligned}
\]
where $C_{1}=\frac{a}{L}\bigl(2\sqrt{2}L^2a^2n+\sigma^2\bigr)
+aL\bigl(\chi_\varepsilon+\rho_r\rho_c\sqrt n\bigr)^2,C_{2}=\frac{a\sigma^2}{L}+\frac{3La^2n}{2}.$
See Appendix~\ref{proof:th_conv_rc} for details.
\end{proposition}

\begin{theorem}
\label{th_conv_r}
Let $\eta_t=t^{-3/4}$, $\beta_t=1-t^{-1/2}$,
$h_t:=1-\beta_t=t^{-1/2}$, $\varepsilon=0$ and $a:=0.2\sqrt{\max\{m,n\}}$.
Assume Assumptions~\ref{ass:1}, \ref{ass:2}, and \ref{ass:3}.
For Algorithm \ref{alg:muoneq}, consider the exact-polar R variant in \eqref{eq:diagpre_family}, with Nesterov momentum disabled and
$0\le \rho<\frac{a}{\sqrt m}, 0\le \lambda_t \le \frac{\rho}{\|\mathbf{X}_1\|_F+4a\sqrt n}\,t^{-1/4}$. Then, for every $T\ge 1$,
\[
\frac1T\sum_{t=1}^T \mathbb E\|\nabla f(\mathbf{X}_t)\|_F
\le \frac{f(\mathbf{X}_1)-f^\star + C_{1,\rho}^{\mathrm{ep}}(1+\ln T) +
C_{2,\rho}^{\mathrm{ep}}}{\bigl(a/\sqrt m-\rho\bigr)\cdot T^{1/4}},
\]
where $C_{1,\rho}^{\mathrm{ep}} = \frac{a}{L}
\Bigl(2\sqrt2L^2(a\sqrt n+\rho)^2+\sigma^2\Bigr) +
aL\Bigl(\frac1{\sqrt m}+\sqrt n\Bigr)^2,C_{2,\rho}^{\mathrm{ep}} = \frac{a\sigma^2}{L} + \frac{3L}{2}(a\sqrt n+\rho)^2.$
\end{theorem}

\begin{corollary}
\label{cor:conv_r_ns}
Under the assumptions of Theorem~\ref{th_conv_r}, replace the exact-polar step by the finite-step NS5 step $\mathbf O_t:=\mathrm{NS5}(\hat{\mathbf M}_t),$
where $\mathrm{NS5}$ is the five-step pre-scaled Newton--Schulz map in
Lemma~\ref{lemma:ns5_traj_error}. For the fixed horizon $T$, we define 
\[
\varepsilon_{\mathrm{ns}}:=\varepsilon_{\mathrm{ns},T}=\max_{1\le t\le T}
\bigl\|\mathrm{NS5}(\hat{\mathbf M}_t)-\operatorname{Orth}(\hat{\mathbf M}_t)
\bigr\|_2.
\]
By Lemma~\ref{lemma:ns5_traj_error}, $\varepsilon_{\mathrm{ns}}<1$. 
Replace the decoupled weight decay condition by
$0\le \rho<\frac{a(1-\varepsilon_{\mathrm{ns}})}{\sqrt m},
0\le \lambda_t \le\frac{\rho}{\|\mathbf{X}_1\|_F+4a(1+\varepsilon_{\mathrm{ns}})\sqrt n}\,t^{-1/4}.
$
Then, for every $T\ge 1$,
\[
\frac1T\sum_{t=1}^T \mathbb E\|\nabla f(\mathbf{X}_t)\|_F
\le
\frac{f(\mathbf{X}_1)-f^\star + C_{1,\rho}^{\mathrm{ns}}(1+\ln T) + C_{2,\rho}^{\mathrm{ns}}}{\bigl(a(1-\varepsilon_{\mathrm{ns}})/\sqrt m-\rho\bigr)\cdot T^{1/4}},
\]
where \scalebox{0.95}{$C_{1,\rho}^{\mathrm{ns}} =
\frac{a}{L} \Bigl( 2\sqrt2L^2 \bigl(a(1+\varepsilon_{\mathrm{ns}})\sqrt n+\rho\bigr)^2 +\sigma^2 \Bigr) + aL \Bigl( \frac{1-\varepsilon_{\mathrm{ns}}}{\sqrt m} + (1+\varepsilon_{\mathrm{ns}})\sqrt n \Bigr)^2$}, $C_{2,\rho}^{\mathrm{ns}} =
\frac{a\sigma^2}{L} + \frac{3L}{2} \bigl(a(1+\varepsilon_{\mathrm{ns}})\sqrt n+\rho\bigr)^2$. See Appendix~\ref{proof:th_conv_r} for details.
\end{corollary}

\begin{remark}
\label{rem:r_main}\label{rem:alg_align}
When \(\varepsilon=0\), \(\mathbf{D}_{r,t}^{-1/2}\) is interpreted as
\((\mathbf{D}_{r,t}^{1/2})^\dagger\), leaving zero rows unchanged. Theorem~\ref{th_conv_r} is the main exact-polar guarantee for the default R variant. Corollary~\ref{cor:conv_r_ns} transfers the same $T^{-1/4}$ dependence to the NS5 implementation, up to the explicit $(1\pm\varepsilon_{\mathrm{ns}})$ constants.
% Figure~\ref{fig:ns5_spec_norm} verifies $\varepsilon_{\mathrm{ns}}<1$ on the MuonEq (R) trajectory considered here. 
In Algorithm~\ref{alg:muoneq}, enabling Nesterov momentum simply amounts to replacing $\mathbf{M}_t$ with $\widetilde{\mathbf{M}}t$ in both $\mathbf{D}{r,t}$ and $\hat{\mathbf{M}}_t$; this should not change the stated $T^{-1/4}$ stationarity result, consistent with~\cite{chen2025muon,chang2025convergence}.
% In Algorithm~\ref{alg:muoneq}, enabling Nesterov simply replaces $\mathbf{M}_t$ by $\widetilde{\mathbf{M}}_t$ in $\mathbf{D}_{r,t}$ and $\hat{\mathbf{M}}_t$, and decoupled weight decay changes the step to $\mathbf{X}_{t+1}=(1-\lambda\eta_t)\mathbf{X}_t-a\eta_t\mathbf O_t$; these standard extensions do not change the stated $T^{-1/4}$ stationarity conclusion, as in~\cite{chang2025convergence,chen2025muon,sfyraki2025lions}. The pseudoinverse formulation above is exactly the $\varepsilon=0$ version of the practical R implementation.
\end{remark}

\begin{remark}
\label{rem:rc_to_r}\label{remark:conv}
Equation~\eqref{eq:error_decomp} separates finite-step orthogonalization error from preprocessing bias. RC is the stronger two-sided spectral corrector, but Proposition~\ref{th_conv_rc} requires the auxiliary condition $\chi_\varepsilon>0$. By contrast, the row-normalized map satisfies
\[
\left\langle \mathbf{M},\operatorname{Orth}(\mathbf P(\mathbf{M})\mathbf{M})\right\rangle
\ge \frac{1}{\sqrt m}\|\mathbf{M}\|_F,
\qquad
\mathbf P(\mathbf{M}):=\bigl(\operatorname{diag}(\operatorname{rowsum}(\mathbf{M}\odot\mathbf{M}))^{1/2}\bigr)^\dagger,
\]
so Theorem~\ref{th_conv_r} needs no $\chi_\varepsilon$-type condition, and Corollary~\ref{cor:conv_r_ns} only degrades this coefficient by the explicit factor $1-\varepsilon_{\mathrm{ns}}$. Hence the default R variant obeys
\[
\frac1T\sum_{t=1}^T \mathbb E\|\nabla f(\mathbf{X}_t)\|_F
=\mathcal O\left(\frac{1+\ln T}{T^{1/4}}\right),
\]
matching the standard Muon-type baseline up to logarithmic~\citep{chen2025muon,chang2025convergence,kim2026convergence,shulgin2025beyond} and constant factors, while C is retained as the embedding-aligned companion and studied empirically~\citep{Xu2026moga}. Improvements beyond this baseline typically require additional mechanisms such as variance reduction~\citep{yuan2024mars,chang2025convergence,liu2025mars,qian2025muon,Fang2018SPIDERNN,zhou2020stochastic,Cutkosky2019MomentumBasedVR,Huang2021SUPERADAMFA}.
\end{remark}
\section{Experiments}
\label{sec:exp}
\subsection{Main Results}
In this section, we evaluate the performance of the proposed {\method} optimizers on LLM pretraining. We use the same 8-node Ascend 910C cluster for all LLaMA experiments; 130M uses 32 NPUs, whereas 350M and 1B use 64 NPUs. Detailed experimental settings are provided in Appendix \ref{appendix:exp}. We will refer to the {\method} with Nesterov acceleration enabled as {\method}-Nes.

\begin{figure}[!htbp]
\centering
\subfloat[Training Loss]{%
\begin{minipage}[b]{0.32\textwidth}
\centering
\includegraphics[width=\linewidth]{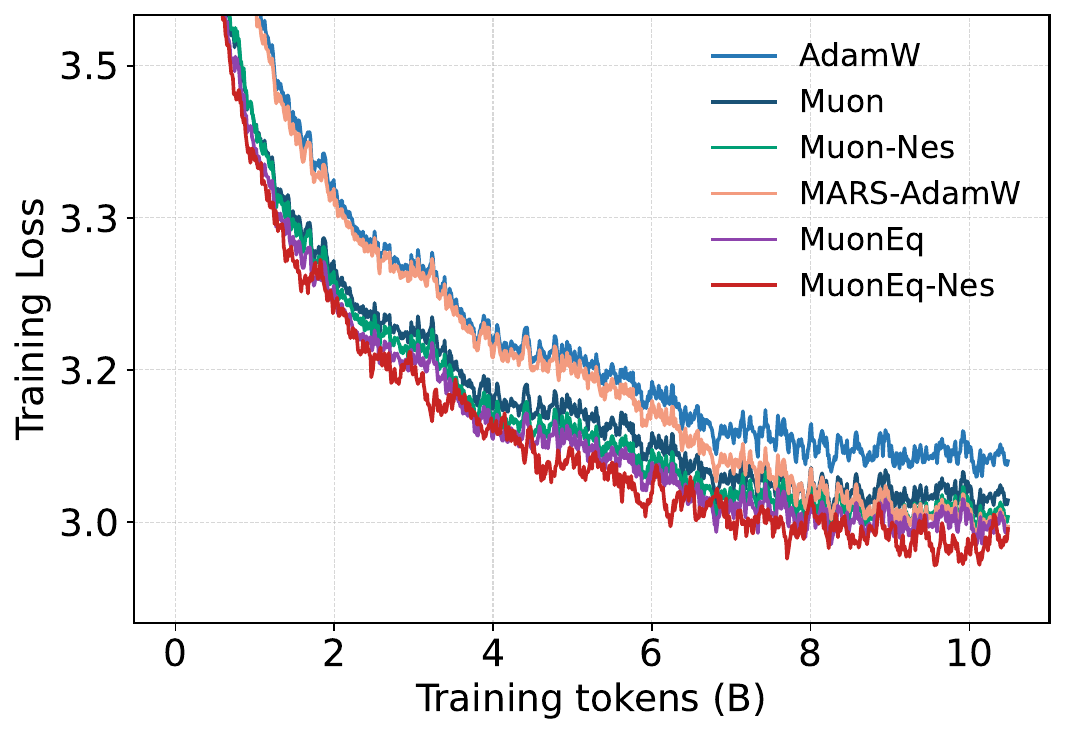}\par\vspace{-3pt}
\includegraphics[width=\linewidth]{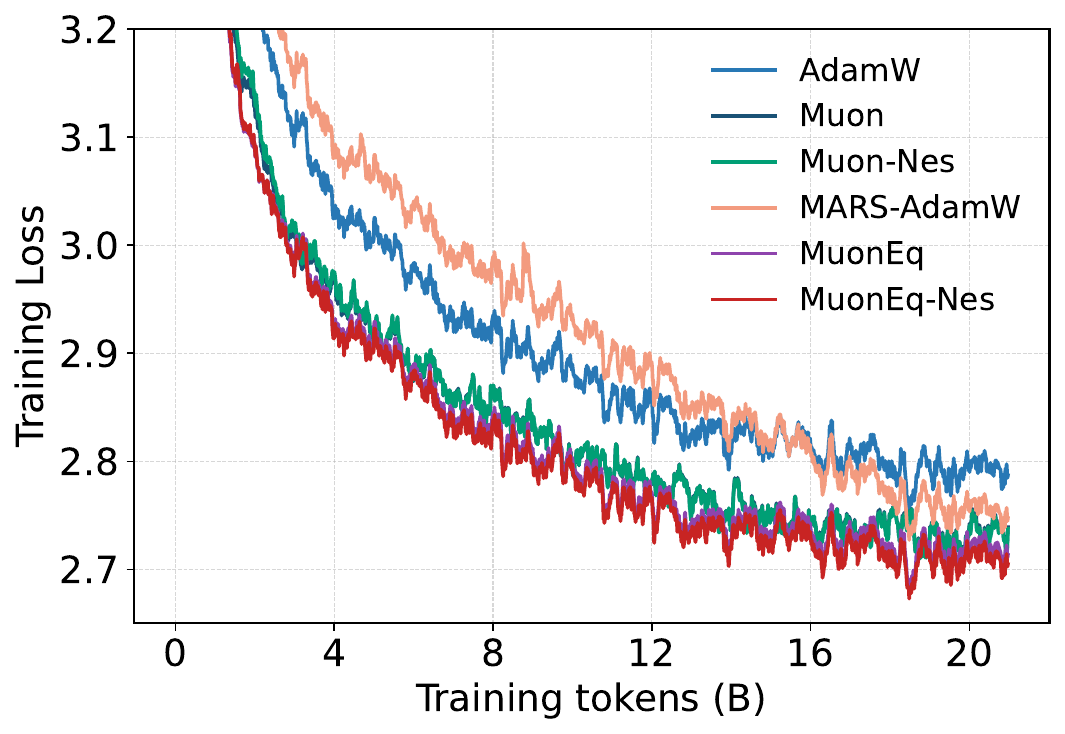}\par\vspace{-3pt}
\includegraphics[width=\linewidth]{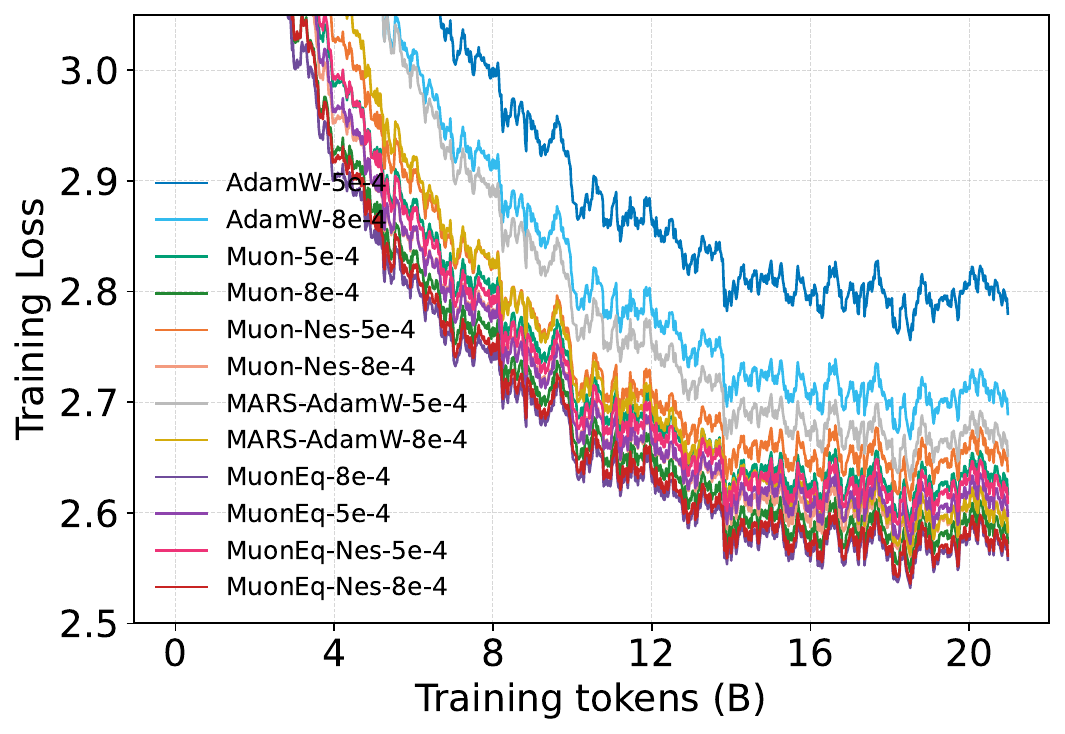}
\end{minipage}
\label{Fig:FG9}%
}
\hfill
\subfloat[Validation Loss]{%
\begin{minipage}[b]{0.32\textwidth}
\centering
\includegraphics[width=\linewidth]{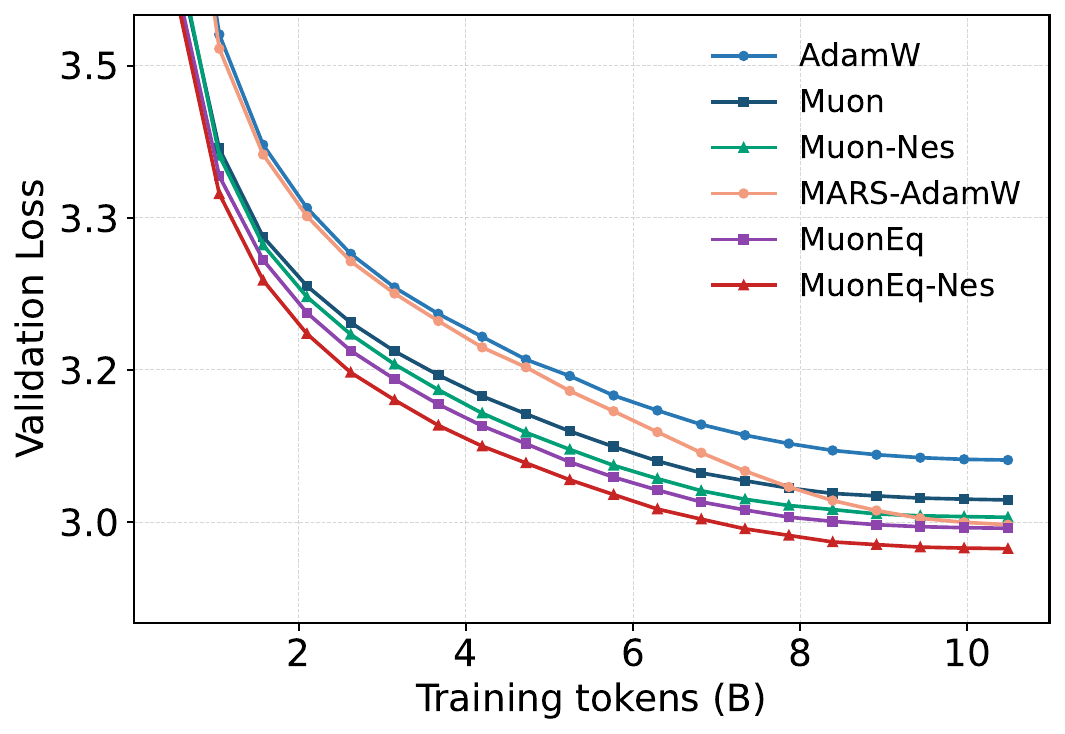}\par\vspace{-2pt}
\includegraphics[width=\linewidth]{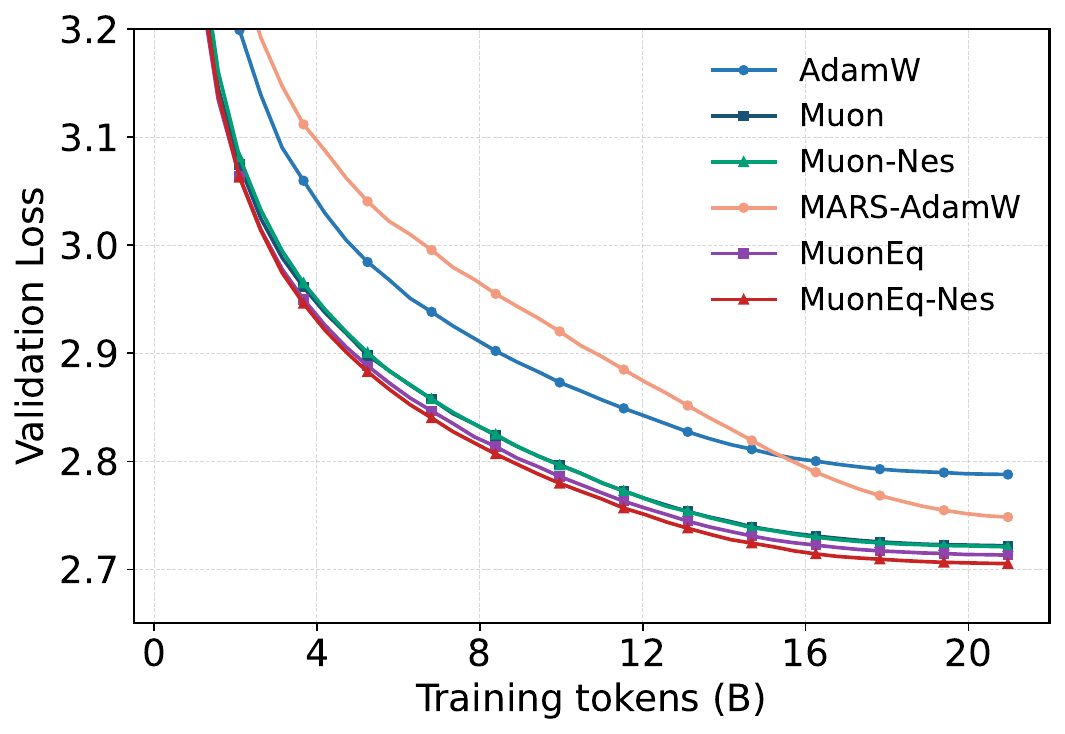}\par\vspace{-2pt}
\includegraphics[width=\linewidth]{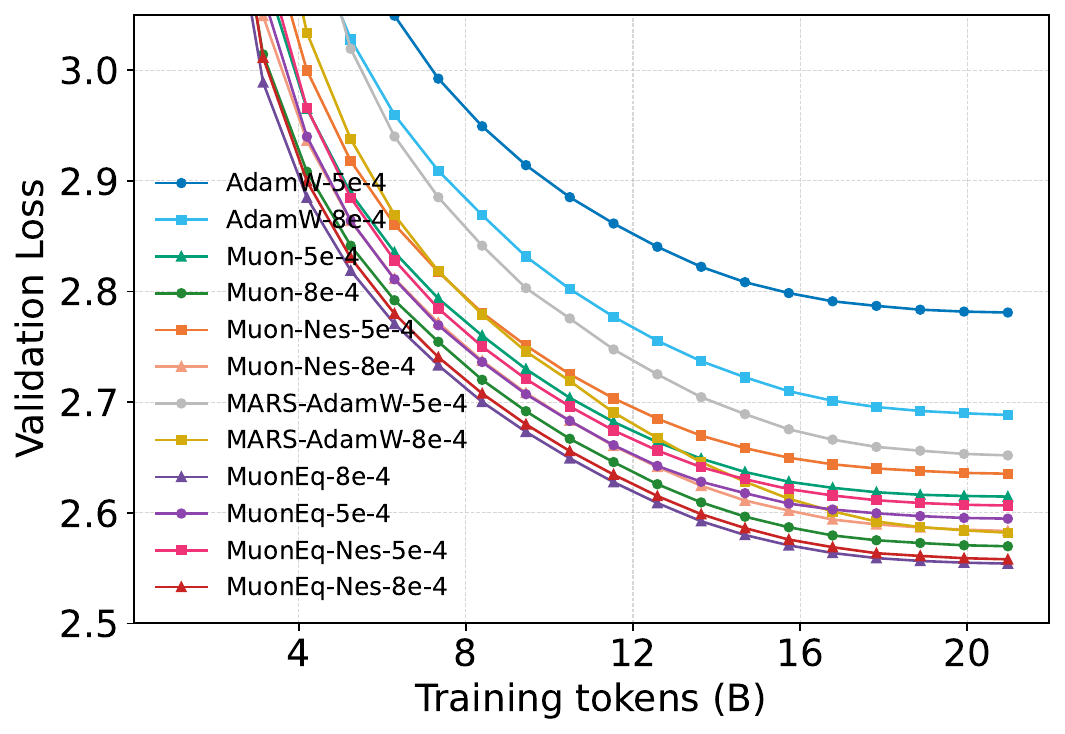}
\end{minipage}
\label{Fig:FG10}%
}
\hfill
\subfloat[Wall-clock Time]{%
\begin{minipage}[b]{0.32\textwidth}
\centering
\includegraphics[width=\linewidth]{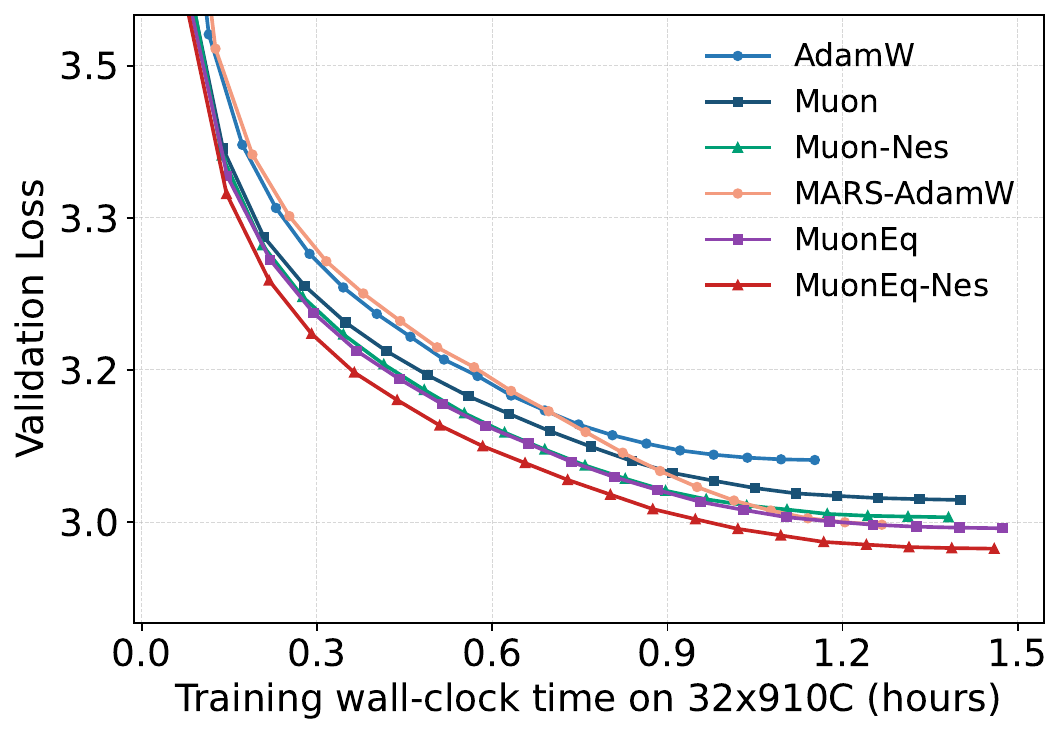}\par\vspace{-2pt}
\includegraphics[width=\linewidth]{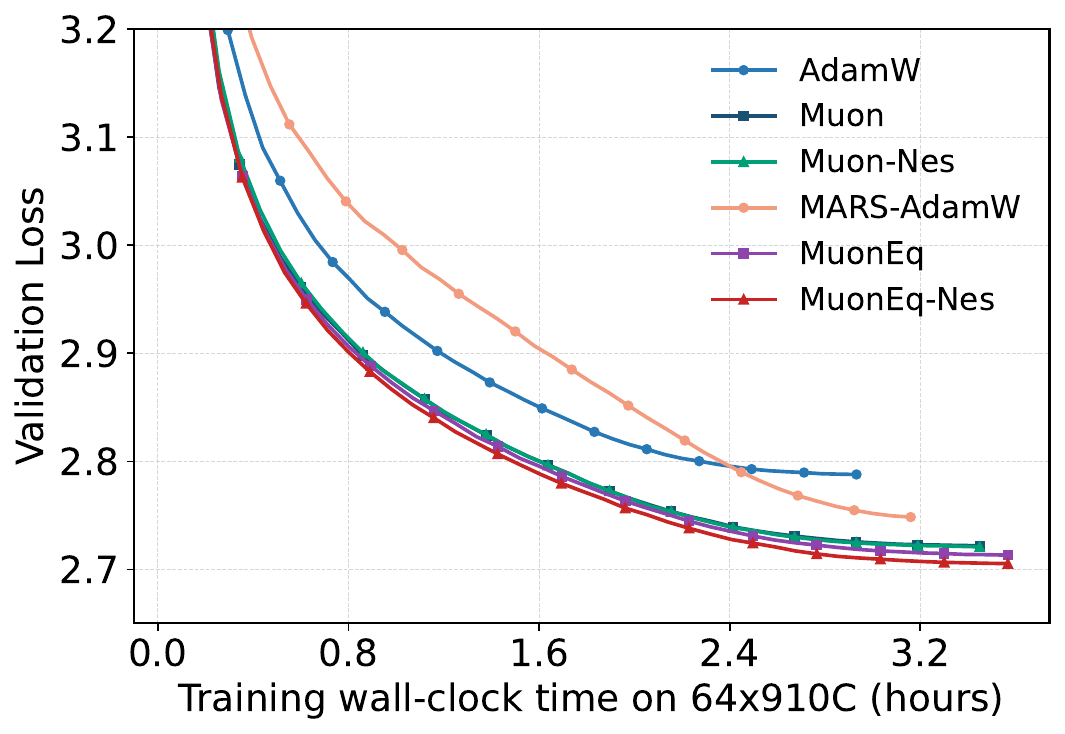}\par\vspace{-2pt}
\includegraphics[width=\linewidth]{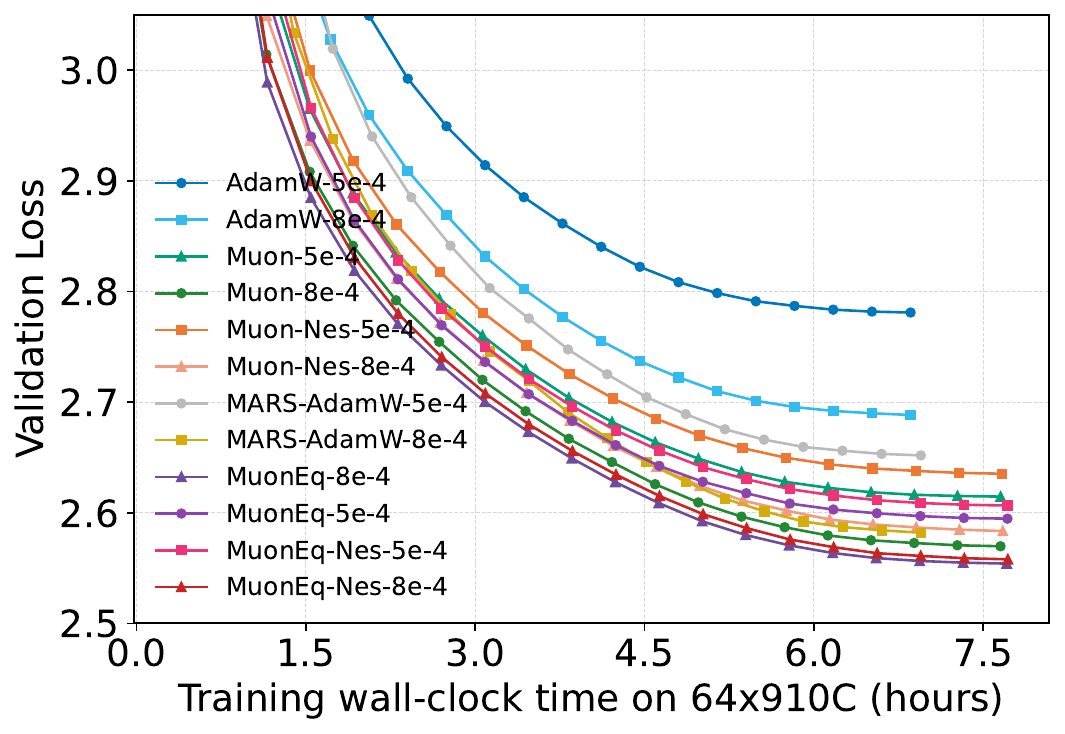}
\end{minipage}
\label{Fig:FG11}%
}

\caption{Comparison of LLaMA2 130M, 350M, and 1B from three perspectives. In each column, the top panel corresponds to the 130M model, the middle panel corresponds to the 350M model, and the bottom panel corresponds to the 1B model. From left to right, the panels show training loss versus tokens, validation loss versus tokens, and validation loss versus wall-clock time.}
\label{fig:train_val_loss_130m_350m_1b}
\end{figure}

$\blacktriangleright$ \textbf{LLaMA2 on C4 Dataset.}
We compare {\method}/{\method}-Nes with AdamW, MARS-AdamW, Muon, and Muon-Nes on LLaMA2-130M, 350M and 1B trained on C4. In the main comparison, 130M is trained on 10.5B tokens, while 350M and 1B are trained on 21.0B tokens. Learning rates are selected through shorter pilot runs using search grids $\{5\mathrm{e}{-4}, 1\mathrm{e}{-3}, 2\mathrm{e}{-3}, 3\mathrm{e}{-3}, 5\mathrm{e}{-3}, 8\mathrm{e}{-3}, 1\mathrm{e}{-2}\}$ for 130M, $\{5\mathrm{e}{-4}, 1\mathrm{e}{-3}, 1.5\mathrm{e}{-3}, 2\mathrm{e}{-3}, 3\mathrm{e}{-3}, 5\mathrm{e}{-3}\}$ for 350M, and $\{5\mathrm{e}{-4}, 8\mathrm{e}{-4}\}$ for 1B. Appendix~\ref{appendix:exp}. Figure~\ref{fig:train_val_loss_130m_350m_1b} shows training loss versus tokens, validation loss versus tokens, and validation loss versus wall-clock time for all three model sizes, while Table~\ref{table-compare} summarizes the best validation perplexity, peak memory usage (GB), and training time (s/step). Overall, {\method}/{\method}-Nes consistently outperforms Muon/Muon-Nes.

\begin{table}[!h]
\vspace{-5pt}
\caption{Best validation perplexity, peak memory usage, and wall-clock training time on LLaMA2/C4.}
\label{table-compare}
\begin{center}
\begin{small}
\resizebox{1.0\textwidth}{!}{%
\begin{tabular}{lccccccccc}
\toprule
\multicolumn{1}{c}{\multirow{2}{*}{\textbf{Optimizer}}}
& \multicolumn{3}{c}{\textbf{130M}} 
& \multicolumn{3}{c}{\textbf{350M}}
& \multicolumn{3}{c}{\textbf{1B}} \\
\cmidrule(lr){2-4} \cmidrule(lr){5-7} \cmidrule(lr){8-10}
& \makecell{Perplexity} 
& \makecell{Memory} 
& \makecell{Time}
& \makecell{Perplexity} 
& \makecell{Memory} 
& \makecell{Time} 
& \makecell{Perplexity} 
& \makecell{Memory} 
& \makecell{Time} \\
\midrule
AdamW~\cite{ilya2019adamw}  & 21.35 & 12.94 & 0.21 & 16.25 & 29.43 & 0.53 & 14.71 & 39.56 & 2.47 \\
MARS-AdamW~\cite{yuan2024mars}  & 20.03 & 12.99 & 0.23 & 15.62 & 29.47 & 0.57 & 13.22 & 39.57  & 2.50\\
Muon~\cite{jordan6muon,liu2025muon}  & 20.52 & 12.52 & 0.25 & 15.21 & 28.18 & 0.61 & 13.06 & 34.81 & 2.76\\
Muon-Nes~\cite{jordan6muon,liu2025muon,chang2025convergence}    & 20.18 & 12.52 & 0.25 & 15.19 & 28.18 & 0.61 & 13.24 & 34.81 & 2.76\\
\oursrow 
{\method}    & \textbf{19.96}\dec{0.56} & 12.52 & 0.26 & \textbf{15.07}\dec{0.14} & 28.18 & 0.63 & \textbf{12.86}\dec{0.20} & 34.81 & 2.78\\
\oursrow 
{\method}-Nes  & \textbf{19.56}\dec{0.62} & 12.52 & 0.26 & \textbf{14.96}\dec{0.23} & 28.18 & 0.63 & \textbf{12.91}\dec{0.33} & 34.81 & 2.78\\
\midrule
Training Tokens 
& \multicolumn{3}{c}{10.5B} 
& \multicolumn{3}{c}{21.0B}
& \multicolumn{3}{c}{21.0B} \\
\bottomrule
\end{tabular}
}
\end{small}
\end{center}
\vspace{-5pt}
\end{table}

\subsection{Ablation Study}
We ablate the three static {\method}-Nes variants, RC, R, and C, to isolate the effect of pre-orthogonalization geometry. All ablations run on 4 RTX Pro6000 (96GB) GPUs. Detailed experimental settings are provided in Appendix~\ref{app:ablation_details}.

Tables~\ref{tab:cifar10_resnet18_ablation_left} and~\ref{tab:cifar10_resnet18_ablation_right} show that, with results averaged over three random seeds, R is consistently the strongest static variant across both benchmarks, achieving higher test accuracy than Muon-Nes on CIFAR-10 with ResNet-18 and lower validation perplexity on FineWeb-10.5B tokens with GPT2-small.

\begin{figure}[!htbp]
    \centering
    \captionsetup{font=small}
    \setlength{\abovecaptionskip}{4pt}
    \setlength{\belowcaptionskip}{-2pt}
    
    \subfloat[Validation Perplexity\label{fig:muon-val}]{
        \includegraphics[width=0.24\textwidth]{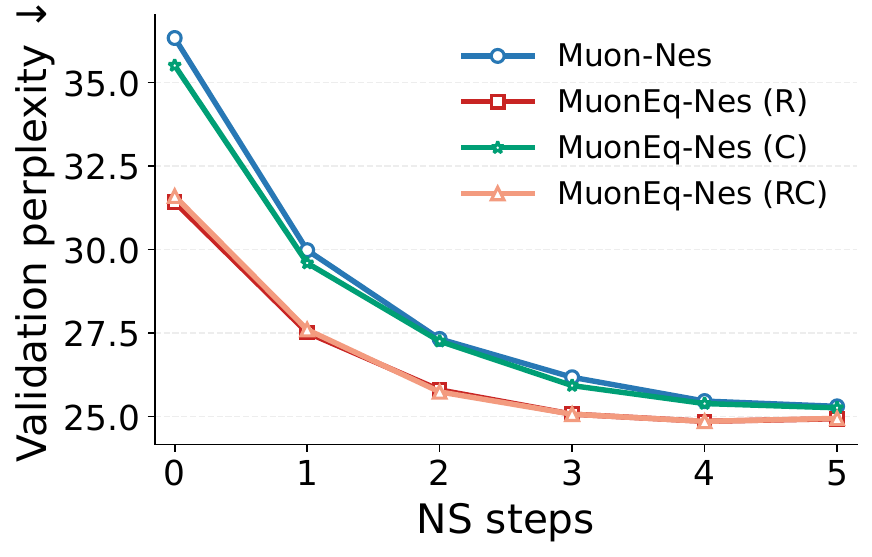}
    }
    \subfloat[Step Time\label{fig:muon-speed}]{
        \includegraphics[width=0.24\textwidth]{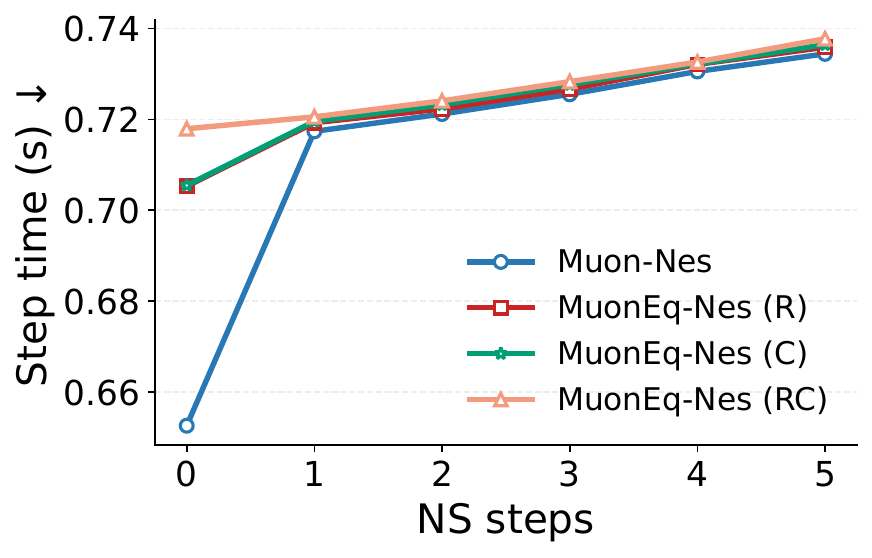}
    }
    \subfloat[Muon-Nes | K=5\label{fig:heatmap-muon-k5}]{
        \includegraphics[width=0.24\textwidth]{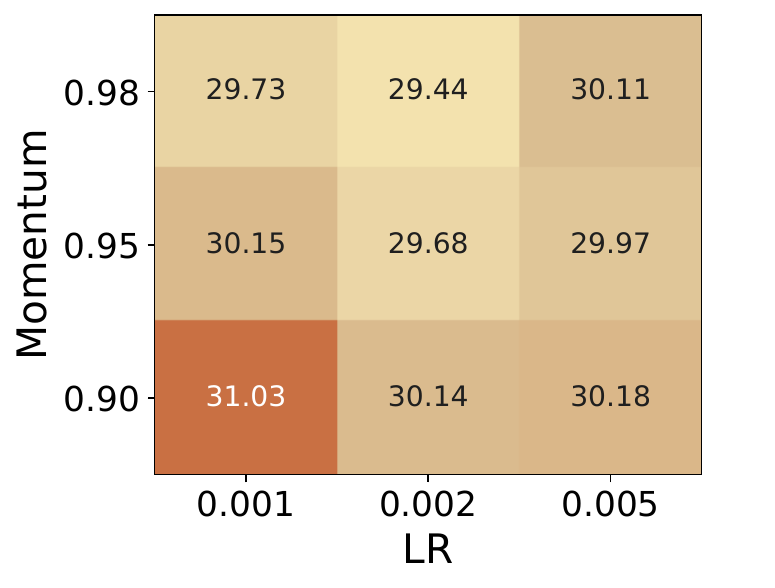}
    }
    \subfloat[MuonEq-Nes (R) | K=5\label{fig:heatmap-muoneq-k5}]{
        \includegraphics[width=0.24\textwidth]{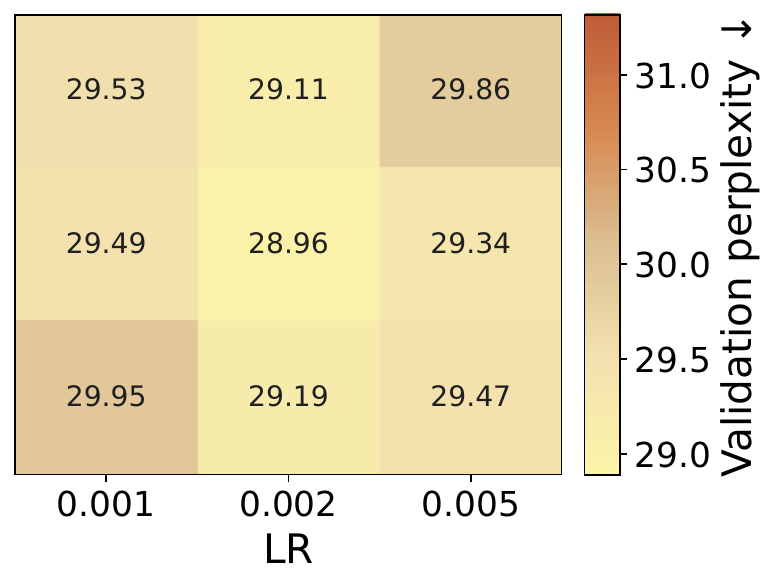}
    }
    \caption{(a,b) Validation perplexity and step time vs. Newton--Schulz iterations $K$ on GPT2-small/FineWeb (10.5B tokens); (c,d) validation perplexity heatmaps over learning rate and momentum at $K=5$ (2.6B tokens).}
    \label{fig:muon_compact}
\end{figure}

\begin{table*}[h]
\centering

\begin{minipage}[t]{0.48\textwidth}
\centering
\caption{Ablation on CIFAR-10 with ResNet-18.}
\label{tab:cifar10_resnet18_ablation_left}
\resizebox{0.92\textwidth}{!}{%
\begin{tabular}{lccc}
\toprule
\multirow{2}{*}{Method} & \multicolumn{2}{c}{Component} & \multirow{2}{*}{\makecell{Test Acc. (\%)}} \\
\cmidrule(lr){2-3}
& Row & Col & \\
\midrule
\multicolumn{4}{l}{\textit{Baseline Experiments}} \\
SGD      &            &            & 92.82$_{\pm0.04}$ \\
AdamW    &            &            & 93.36$_{\pm0.10}$ \\
Muon &            &            & 94.39$_{\pm0.04}$ \\
Muon-Nes &            &            & 94.44$_{\pm0.05}$ \\
\midrule
\multicolumn{4}{l}{\textit{Ablation Experiments}} \\
MuonEq-Nes (RC)   & \checkmark & \checkmark & 94.29$_{\pm0.10}$ \\
MuonEq-Nes (C) &            & \checkmark & 94.45$_{\pm0.02}$ \\
\oursrow 
MuonEq-Nes (R) & \checkmark &            & \textbf{94.57}$_{\pm0.07}$\inc{0.13} \\
\bottomrule
\end{tabular}%
}
\end{minipage}
\hfill
\begin{minipage}[t]{0.48\textwidth}
\centering
\caption{Ablation on FineWeb with GPT2-small.}
\label{tab:cifar10_resnet18_ablation_right}
\resizebox{0.9\textwidth}{!}{%
\begin{tabular}{lccc}
\toprule
\multirow{2}{*}{Method} & \multicolumn{2}{c}{Component} & \multirow{2}{*}{\makecell{Validation \\ Perplexity}} \\
\cmidrule(lr){2-3}
& Row & Col & \\
\midrule
\multicolumn{4}{l}{\textit{Baseline Experiments}} \\
AdamW    &            &            & 26.48$_{\pm0.12}$ \\
Muon-Nes &            &            & 25.23$_{\pm0.08}$ \\
AdaMuon &            &            & 24.99$_{\pm0.13}$ \\
Muon+ &            &            & 25.12$_{\pm0.08}$ \\
% FISMO &            &            & 25.27$_{\pm0.10}$ \\
Mousse &            &            & 25.16$_{\pm0.10}$ \\
\midrule
\multicolumn{4}{l}{\textit{Ablation Experiments}} \\
MuonEq-Nes (RC)   & \checkmark & \checkmark & 24.94$_{\pm0.04}$ \\
MuonEq-Nes (C) &            & \checkmark &  25.23$_{\pm0.06}$ \\
\oursrow
MuonEq-Nes (R) & \checkmark &            & \textbf{24.88}$_{\pm0.02}$\dec{0.35} \\
\bottomrule
\end{tabular}%
}
\end{minipage}
\end{table*}

Figure~\ref{fig:muon-val} and \ref{fig:muon-speed} show that the default R variant of {\method}-Nes is markedly less sensitive to the NS iteration count $K$ than Muon-Nes: its validation perplexity remains stable over a wider range of $K$, while step time stays comparable throughout the sweep. Complementary robustness evidence is given in Figure~\ref{fig:heatmap-muon-k5} and~\ref{fig:heatmap-muoneq-k5}, where {\method}-Nes (R) exhibits a broader low-perplexity region over learning-rate and momentum sweeps than Muon-Nes, especially at $K=5$. This pattern is consistent with Eq.~\eqref{eq:error_decomp} and Section~\ref{sec:analysis}. RC applies stronger two-sided spectral correction and can reduce finite-step NS error more aggressively, but at the cost of larger preconditioning bias. 
% For the hidden weights targeted by {\method}-Nes, the geometry is primarily row-sided, making R the best trade-off between spectral correction and geometric alignment, consistent with \citep{Xu2026moga}. 
For the hidden-weight settings we tested, R provides a favorable trade-off between spectral correction and preprocessing bias, consistent with \citep{Xu2026moga}. 
Comparisons with \citep{pethick2025training} and \citep{glentis2025minimalist} require caution, as row/column terms vary by layout and task; see Appendix~\ref{app:ablation_term_note}.

\section{Conclusion}
We improve the matrix geometry encountered by finite-step orthogonalization in Muon-style optimization without resorting to full whitening. {\method} reshapes the input to Newton--Schulz using lightweight pre-orthogonalization equilibration in three forms, RC, R, and C; R is used by default for hidden weights. Theoretically, these schemes serve as zeroth-order whitening surrogates: RC provides stronger two-sided correction, whereas R yields cleaner one-sided geometry and supports the main $\widetilde{\mathcal O}(T^{-1/4})$ guarantee. Empirically, R consistently outperforms Muon in LLaMA2 pretraining on C4 across all tested scales and budgets.

\bibliographystyle{unsrt}
\bibliography{manuscript}

%%%%%%%%%%%%%%%%%%%%%%%%%%%%%%%%%%%%%%%%%%%%%%%%%%%%%%%%%%%%%%%%%%%%%%%%%%%%%%%
%%%%%%%%%%%%%%%%%%%%%%%%%%%%%%%%%%%%%%%%%%%%%%%%%%%%%%%%%%%%%%%%%%%%%%%%%%%%%%%
% APPENDIX
%%%%%%%%%%%%%%%%%%%%%%%%%%%%%%%%%%%%%%%%%%%%%%%%%%%%%%%%%%%%%%%%%%%%%%%%%%%%%%%
%%%%%%%%%%%%%%%%%%%%%%%%%%%%%%%%%%%%%%%%%%%%%%%%%%%%%%%%%%%%%%%%%%%%%%%%%%%%%%%
\newpage
\appendix

\onecolumn
\section*{\LARGE Appendix}
\section{Related Work}
\label{sec:rw}

{\method} is viewed as a lightweight pre-balanced Muon family rather than as a fully whitened optimizer. It inserts diagonal equilibration before Newton--Schulz orthogonalization and includes three forms: two-sided row/column normalization (RC), row normalization (R), and column normalization (C), with R as the default choice for the hidden matrix weights considered in this paper. 
% All three use only row/column squared-norm statistics and preserve the $\mathcal O(m+n)$ state profile.
All three compute row/column squared norms from the current momentum and add no persistent optimizer state beyond Muon.

\textbf{Muon and orthogonalized-update optimizers.} Muon~\cite{jordan6muon} introduced the core design of orthogonalizing matrix-valued momentum with a small fixed number of Newton--Schulz iterations. Subsequent work has expanded this idea into a broader family of orthogonalized-update optimizers, including post-orthogonalization normalization and rescaling methods such as Muon+, NorMuon, and AdaMuon~\cite{zhang2026muon+,li2025normuon,si2025adamuon}; adaptive stepsize variants such as AdaGO and NAMO~\cite{Zhang2025AdaGradMM,zhang2026namo}; distributed and communication-aware extensions such as Dion and MuLoCo~\cite{ahn2025dion,therien2025muloco}; reduced-state implementations~\cite{gupta2025_8bitmuon}; and fine-tuning-oriented variants~\cite{page2025muonall}. These methods mainly alter the orthogonalized update itself or the surrounding optimization machinery. By contrast, {\method} modifies the matrix fed into orthogonalization through lightweight pre-orthogonalization equilibration.

\textbf{Theory of Muon-style methods.} Recent theory interprets Muon as steepest descent under spectral-norm or nuclear-norm geometry~\cite{bernstein2024old,chen2025muon}, establishes stochastic and nonconvex convergence guarantees~\cite{shen2025convergence,chang2025convergence,nagashima2026improved,zhang2026convergence}, and analyzes practical ingredients such as weight decay, critical batch size, and inexact orthogonalization~\cite{Sato2025ConvergenceBA,kim2026convergence,shulgin2025beyond}. This literature is directly relevant to {\method} because we retain the same finite-step orthogonalization primitive as Muon and change only the geometry of its input. Our analysis further highlights the role of input spectrum in finite-step Newton--Schulz and clarifies how diagonal equilibration trades stronger spectral correction against preprocessing bias.

\textbf{Structured preconditioning, factorized statistics, and whitening-inspired methods.} A separate line of work studies matrix-aware preconditioning, including Shampoo~\cite{gupta2018shampoo}, K-FAC~\cite{martens2015optimizing,eschenhagen2023kronecker}, SOAP~\cite{vyas2024soap}, Sketchy~\cite{feinberg2023sketchy}, and ASGO~\cite{an2025asgo}. Adafactor~\cite{Shazeer2018AdafactorAL} is especially relevant from a systems perspective because it shows how row- and column-factorized second-moment statistics reduce optimizer state to $\mathcal O(m+n)$ for matrix-valued parameters. Recent Muon-adjacent whitening or preconditioned-polar methods, including FISMO and Mousse~\cite{Xu2026FISMOFM,Zhang2026MousseRT}, pursue a similar high-level goal: improve the geometry before imposing spectral constraints. The key difference is that these methods rely on richer matrix- or Kronecker-structured preconditioners, whereas {\method} restricts the intervention to diagonal row/column equilibration. In this sense, {\method} is closer to low-cost equilibration than to full whitening.

\textbf{Orthogonalization algorithms.} There is also a complementary numerical literature on faster and more accurate polar/orthogonalization routines, including improved Newton--Schulz variants and matrix-sign methods such as CANS and Polar Express~\cite{grishina2025accelerating,amsel2025polar}. This literature matters because Muon-style optimizers operate in the regime of approximate orthogonalization, where only a small number of iterations is affordable; {\method} is designed explicitly for this setting.
\section{Proofs of Theorem \ref{th_ns}}
\label{proof:th_ns}
\begin{proof}
Let $\alpha>0$. For $k=0$, the desired form follows directly from
\[
\mathbf{X}_0=\alpha^{-1}\mathbf{G} = \mathbf{U}\operatorname{diag}\Bigl(\frac{\sigma_1}{\alpha},\dots,\frac{\sigma_r}{\alpha}\Bigr)\mathbf{V}^\top.
\]
Assume that, for some $k\ge0$,
\[
\mathbf{X}_k=\mathbf{U}\operatorname{diag}\bigl(s_1^{(k)},\dots,s_r^{(k)}\bigr)\mathbf{V}^\top.
\]
Then
\[
\mathbf{X}_k\mathbf{X}_k^\top = \mathbf{U}\operatorname{diag}\bigl((s_1^{(k)})^2,\dots,(s_r^{(k)})^2\bigr)\mathbf{U}^\top.
\]
Inserting this expression into the iteration and using $\mathbf{U}^\top\mathbf{U}=\mathbf{I}_r$ gives
\[
\mathbf{X}_{k+1} = \mathbf{U}\operatorname{diag}\bigl(\phi(s_1^{(k)}),\dots,\phi(s_r^{(k)})\bigr)\mathbf{V}^\top.
\]
By induction,
\[
\mathbf{X}_k = \mathbf{U}\operatorname{diag}\bigl(s_1^{(k)},\dots,s_r^{(k)}\bigr)\mathbf{V}^\top,
\qquad s_i^{(k+1)}=\phi\left(s_i^{(k)}\right).
\]
We also have
\[
\operatorname{Orth}(\mathbf{G})=\mathbf{U}\mathbf{V}^\top =\mathbf{U} \mathbf{I}_r\mathbf{V}^\top,
\quad \mathbf{X}_k-\operatorname{Orth}(\mathbf{G}) =\mathbf{U}\Bigl(\operatorname{diag}(s_1^{(k)},\dots,s_r^{(k)})-\mathbf{I}_r\Bigr)\mathbf{V}^\top.
\]
Left and right multiplication by matrices with orthonormal columns preserves the Frobenius norm, so
\[
\bigl\|\mathbf{X}_k-\operatorname{Orth}(\mathbf{G})\bigr\|_F^2 = \sum_{i=1}^r (s_i^{(k)}-1)^2.
\]
This proves the exact error representation.

We now establish the lower bound. Recall that
\[
q(t):=a+bt+ct^2.
\]
By assumption,
\[0<q(t)\le a, \qquad t\in[0,1].\]
Fix $i\in[r]$ and $k\ge0$. If $a^k\frac{\sigma_i}{\alpha}\ge1,$ then
\[
\bigl|1-s_i^{(k)}\bigr|\ge0=\left(1-a^k\frac{\sigma_i}{\alpha}\right)_+.
\]
Now suppose that $a^k\sigma_i/\alpha<1$. We claim that, for every $0\le j\le k$,
\[
0<s_i^{(j)}\le a^j\frac{\sigma_i}{\alpha}<1.
\]
The statement is immediate for $j=0$. Suppose it holds for some $j<k$. Then $s_i^{(j)}\in(0,1)$, and hence $(s_i^{(j)})^2\in(0,1)$. Therefore,
\[
0<q\bigl((s_i^{(j)})^2\bigr)\le a.
\]
It follows that
\[
0<s_i^{(j+1)}=s_i^{(j)}q \bigl((s_i^{(j)})^2\bigr) \le a s_i^{(j)}\le a^{j+1}\frac{\sigma_i}{\alpha}.
\]
Because $j+1\le k$,
\[
a^{j+1}\frac{\sigma_i}{\alpha}\le a^k\frac{\sigma_i}{\alpha}<1.
\]
This proves the claim. In particular,
\[
0<s_i^{(k)}\le a^k\frac{\sigma_i}{\alpha}<1,
\]
and thus
\[
|1-s_i^{(k)}| =1-s_i^{(k)} \ge 1-a^k\frac{\sigma_i}{\alpha} = \Bigl(1-a^k\frac{\sigma_i}{\alpha}\Bigr)_+.
\]
Combining the two cases gives, for every $i$,
\[
|1-s_i^{(k)}| \ge \Bigl(1-a^k\frac{\sigma_i}{\alpha}\Bigr)_+.
\]
Finally, set $\alpha = \|\mathbf{G}\|_F$. Squaring both sides, summing over $i$, and using the exact error representation together with
\[
\frac{\sigma_i}{\|\mathbf{G}\|_F} =\frac{\sigma_i/\sigma_1}{\|\mathbf{G}\|_F/\sigma_1} = \frac{1}{\kappa_i\sqrt{\operatorname{sr}(\mathbf{G})}},
\]
we obtain
\[
\frac{1}{\sqrt r}\bigl\|\mathbf{X}_k-\operatorname{Orth}(\mathbf{G})\bigr\|_F \ge \frac{1}{\sqrt r}\left(\sum_{i=1}^r\left(1-\frac{a^k}{\kappa_i\sqrt{\operatorname{sr}(\mathbf{G})}}\right)_+^2\right)^{1/2}.
\]
Moreover,
\[
\tau_i:=\log_a\frac{\|\mathbf{G}\|_F}{\sigma_i}=\log_a\left(\kappa_i\sqrt{\operatorname{sr}(\mathbf{G})}\right),
\]
so
\[
\tau_1=\log_a\sqrt{\operatorname{sr}(\mathbf{G})},
\qquad
\tau_r=\log_a\left(
\kappa(\mathbf{G})\sqrt{\operatorname{sr}(\mathbf{G})}
\right),
\qquad
\tau_r-\tau_1=\log_a\kappa(\mathbf{G}).
\]
This completes the proof.
\end{proof}

\section{Lemmas for Proposition \ref{th_zero_rc}}
\subsection{Lemma \ref{ana_lemma:a1}}
\begin{lemma}
\label{ana_lemma:a1}
Let $\mathbf{A}\in\mathbb{R}^{n\times n}$ be symmetric positive definite, and let
$\mathbf{E}\in\mathbb{R}^{n\times n}$ be symmetric. Define $g(\mathbf{X})=\mathbf{X}^{1/2},f(\mathbf{X})=\mathbf{X}^{-1/2}$. Then $g$ and $f$ are Fr\'echet differentiable at $\mathbf{A}$. We use $L_h(\mathbf{A};\mathbf{E})$ to denote the Fr\'echet derivative of
a matrix function $h$ at $\mathbf{A}$ applied to $\mathbf{E}$, i.e., $h(\mathbf{A}+\mathbf{E})=h(\mathbf{A})+L_h(\mathbf{A};\mathbf{E}) +o(\|\mathbf{E}\|_2)$. If $\mathbf{S}=\mathbf{A}^{1/2},\mathbf{L}=L_g(\mathbf{A};\mathbf{E})$, then $\mathbf{L}$ is the unique symmetric solution of the Sylvester equation
\[
\mathbf{S}\mathbf{L}+\mathbf{L}\mathbf{S}=\mathbf{E}.
\]
Moreover,
\[
L_f(\mathbf{A};\mathbf{E})=-\mathbf{A}^{-1/2}\mathbf{L}\mathbf{A}^{-1/2}.
\]
Consequently, for $\|\mathbf{E}\|_2$ sufficiently small,
\[
(\mathbf{A}+\mathbf{E})^{-1/2}
=\mathbf{A}^{-1/2}-\mathbf{A}^{-1/2}\mathbf{L}\mathbf{A}^{-1/2}+\mathcal{O}(\|\mathbf{E}\|_2^2).
\]
\end{lemma}

\begin{proof}
For completeness, differentiate the identity $
g(\mathbf{X})^2=\mathbf{X} $ at $\mathbf{A}$ in the direction $\mathbf{E}$. Writing $\mathbf{L}=L_g(\mathbf{A};\mathbf{E})$ and $\mathbf{S}=\mathbf{A}^{1/2}$ gives
\[
\mathbf{S}\mathbf{L}+\mathbf{L}\mathbf{S}=\mathbf{E}.
\]
Since $\mathbf{S}\succ0$, the Sylvester operator $\mathbf{X}\mapsto\mathbf{S}\mathbf{X}+\mathbf{X}\mathbf{S}$ is invertible, so the solution $\mathbf{L}$ is unique. Next, since $f=\mathrm{inv}\circ g$, the chain rule and the Fr\'echet derivative of the matrix inverse,
\[
L_{\mathrm{inv}}(\mathbf{Z};\mathbf{W})=-\mathbf{Z}^{-1}\mathbf{W}\mathbf{Z}^{-1},
\]
yield
\[
L_f(\mathbf{A};\mathbf{E})
=-\mathbf{S}^{-1}\mathbf{L}\mathbf{S}^{-1}
=-\mathbf{A}^{-1/2}\mathbf{L}\mathbf{A}^{-1/2}.
\]
The stated first-order expansion follows from the Fr\'echet differentiability of $f$ at $\mathbf{A}$.
\end{proof}

\subsection{Lemma \ref{ana_lemma:a2}}
\begin{lemma}
\label{ana_lemma:a2}
Let $\mathbf{A}\succ0$ and set $\mathbf{S}=\mathbf{A}^{1/2}\succ0$. Consider the Sylvester equation
\[
\mathbf{SL}+\mathbf{LS}=\mathbf{E}.
\]
Its unique solution is given by \cite{bhatia2009positive}
\[
\mathbf{L}=\int_0^\infty e^{-t\mathbf{S}}\mathbf{E}e^{-t\mathbf{S}}dt.
\]
Moreover, in the operator norm,
\[
\|\mathbf{L}\|_2
\le\frac{1}{2\lambda_{\min}(\mathbf{S})}\|\mathbf{E}\|_2
=\frac{1}{2\sqrt{\lambda_{\min}(\mathbf{A})}}\|\mathbf{E}\|_2.
\]
The same bound also holds for the Frobenius norm, and more generally for any unitarily invariant norm.
\end{lemma}

\begin{proof}
Differentiate the integrand:
\[
\frac{d}{dt}\bigl(e^{-t\mathbf{S}}\mathbf{E}e^{-t\mathbf{S}}\bigr)=-\mathbf{S}e^{-t\mathbf{S}}\mathbf{E}e^{-t\mathbf{S}}-e^{-t\mathbf{S}}\mathbf{E}e^{-t\mathbf{S}}\mathbf{S}.
\]
Integrating from $0$ to $\infty$ yields
\[
\mathbf{S}\mathbf{L}+\mathbf{L}\mathbf{S}=\mathbf{E},
\]
which proves the integral representation.

For the norm bound, using submultiplicativity and the fact that
\[
\|e^{-t\mathbf{S}}\|_2=e^{-t\lambda_{\min}(\mathbf{S})},
\]
we have
\[
\|\mathbf{L}\|_2
\le \int_0^\infty \|e^{-t\mathbf{S}}\|_2^2\|\mathbf{E}\|_2\,dt
= \|\mathbf{E}\|_2\int_0^\infty e^{-2t\lambda_{\min}(\mathbf{S})}dt
= \frac{1}{2\lambda_{\min}(\mathbf{S})}\|\mathbf{E}\|_2.
\]
Since $\lambda_{\min}(\mathbf{S})=\sqrt{\lambda_{\min}(\mathbf{A})}$, the stated bound follows.
\end{proof}

\section{Proofs of Proposition \ref{th_zero_rc}}
\label{proof:th_zero_rc}
\begin{proof}
For the column/right statement, we have
\[
\mathbf{M}^\top\mathbf{M} =\mathbf{D}_c(\mathbf{I}_n+\mathbf{C}_c)\mathbf{D}_c =\mathbf{D}_c^2+\mathbf{D}_c\mathbf{C}_c\mathbf{D}_c.
\]
Apply Lemma \ref{ana_lemma:a1} with $
\mathbf{A}=\mathbf{D}_c^2,\mathbf{E}=\mathbf{D}_c\mathbf{C}_c\mathbf{D}_c$. Then
\[
(\mathbf{M}^\top\mathbf{M})^{-1/2} =\mathbf{D}_c^{-1}-\mathbf{D}_c^{-1}\mathbf{L}_c\mathbf{D}_c^{-1}+\mathcal{O}(\|\mathbf{D}_c\mathbf{C}_c\mathbf{D}_c\|_2^2),
\]
where $\mathbf{L}_c$ solves
$\mathbf{D}_c\mathbf{L}_c+\mathbf{L}_c\mathbf{D}_c=\mathbf{D}_c\mathbf{C}_c\mathbf{D}_c.$

Multiplying on the left by $\mathbf{M}$ yields
\[
\begin{aligned}
\operatorname{Orth}_r(\mathbf{M})&=\mathbf{M}\mathbf{D}_c^{-1}-\mathbf{M}\mathbf{D}_c^{-1}\mathbf{L}_c\mathbf{D}_c^{-1}+\mathcal{O}(\|\mathbf{D}_c\mathbf{C}_c\mathbf{D}_c\|_2^2)\\
&=\mathbf{N}_c-\mathbf{N}_c\mathbf{L}_c\mathbf{D}_c^{-1}
+\mathcal{O}(\|\mathbf{D}_c\mathbf{C}_c\mathbf{D}_c\|_2^2).
\end{aligned}
\]
Moreover, by Lemma \ref{ana_lemma:a2}, we have
\[
\|\mathbf{L}_c\|_2 \le \frac{1}{2\lambda_{\min}(\mathbf{D}_c)}
\|\mathbf{D}_c\mathbf{C}_c\mathbf{D}_c\|_2 \le \frac{\|\mathbf{D}_c\|_2^2}{2\lambda_{\min}(\mathbf{D}_c)}\|\mathbf{C}_c\|_2.
\]
Hence
\[
\|\mathbf{L}_c\mathbf{D}_c^{-1}\|_2 \le \|\mathbf{L}_c\|_2\|\mathbf{D}_c^{-1}\|_2 \le \frac{\kappa(\mathbf{D}_c)^2}{2}\|\mathbf{C}_c\|_2.
\]
Therefore the first-order correction is $\mathcal{O}(\|\mathbf{C}_c\|_2)$, and
\[
\operatorname{Orth}_r(\mathbf{M})=\mathbf{N}_c+\mathcal{O}(\|\mathbf{C}_c\|_2).
\]
For the row/left statement, we have
\[
\mathbf{M}\mathbf{M}^\top = \mathbf{D}_r(\mathbf{I}_m+\mathbf{C}_r)\mathbf{D}_r = \mathbf{D}_r^2+\mathbf{D}_r\mathbf{C}_r\mathbf{D}_r.
\]
Applying Lemma \ref{ana_lemma:a1} with $\mathbf{A}=\mathbf{D}_r^2,
\mathbf{E}=\mathbf{D}_r\mathbf{C}_r\mathbf{D}_r$
gives
\[
(\mathbf{M}\mathbf{M}^\top)^{-1/2} = \mathbf{D}_r^{-1} - \mathbf{D}_r^{-1}\mathbf{L}_r\mathbf{D}_r^{-1} + \mathcal{O}(\|\mathbf{D}_r\mathbf{C}_r\mathbf{D}_r\|_2^2),
\]
where $\mathbf{L}_r$ solves $\mathbf{D}_r\mathbf{L}_r+\mathbf{L}_r\mathbf{D}_r
=\mathbf{D}_r\mathbf{C}_r\mathbf{D}_r.$

Multiplying on the right by $\mathbf{M}$ yields
\[
\begin{aligned}
\operatorname{Orth}_\ell(\mathbf{M})
&=\mathbf{D}_r^{-1}\mathbf{M}-\mathbf{D}_r^{-1}\mathbf{L}_r\mathbf{D}_r^{-1}\mathbf{M}+\mathcal{O}(\|\mathbf{D}_r\mathbf{C}_r\mathbf{D}_r\|_2^2)\\
&=\mathbf{N}_r-\mathbf{D}_r^{-1}\mathbf{L}_r\mathbf{N}_r+\mathcal{O}(\|\mathbf{D}_r\mathbf{C}_r\mathbf{D}_r\|_2^2).
\end{aligned}
\]
Again by Lemma \ref{ana_lemma:a2}, we have
\[
\|\mathbf{L}_r\|_2 \le \frac{1}{2\lambda_{\min}(\mathbf{D}_r)} \|\mathbf{D}_r\mathbf{C}_r\mathbf{D}_r\|_2 \le \frac{\|\mathbf{D}_r\|_2^2}{2\lambda_{\min}(\mathbf{D}_r)}
\|\mathbf{C}_r\|_2,
\]
so
\[
\|\mathbf{D}_r^{-1}\mathbf{L}_r\|_2 \le\|\mathbf{D}_r^{-1}\|_2\|\mathbf{L}_r\|_2 \le\frac{\kappa(\mathbf{D}_r)^2}{2}\|\mathbf{C}_r\|_2.
\]
Therefore
\[
\operatorname{Orth}_\ell(\mathbf{M})=\mathbf{N}_r+\mathcal{O}(\|\mathbf{C}_r\|_2).
\]
\end{proof}

\section{Normalization and whitening}
\label{appendix:normalization}
This appendix collects the normalization--whitening estimates that are omitted from the main text. We use the notation of Proposition~\ref{th_zero_rc}. All $\mathcal O(\cdot)$ terms are taken in operator norm; the implied constants may depend on the fixed matrices in Proposition~\ref{th_zero_rc}, but not on the perturbation terms.

\begin{proposition}
\label{prop:one_sided_norm}
The following estimates hold.

In the column/right setting of Proposition~\ref{th_zero_rc}, if $\|\mathbf C_c\|_2$ is sufficiently small, then
$\operatorname{Orth}_r(\mathbf N_c)=\mathbf N_c(\mathbf I_n+\mathbf C_c)^{-1/2}
=\mathbf N_c-\tfrac12\mathbf N_c\mathbf C_c+\mathcal O(\|\mathbf C_c\|_2^2)$, and therefore
$\|\operatorname{Orth}_r(\mathbf N_c)-\mathbf N_c\|_2
\le \tfrac12\|\mathbf N_c\|_2\|\mathbf C_c\|_2+\mathcal O(\|\mathbf C_c\|_2^2)$.
Moreover,
\[
\begin{aligned}
\|\operatorname{Orth}_r(\mathbf M)-\mathbf M\|_2
\le{}& \|\mathbf N_c\|_2\|\mathbf D_c-\mathbf I_n\|_2 \\
&+ \frac{\kappa(\mathbf D_c)^2}{2}\|\mathbf N_c\|_2\|\mathbf C_c\|_2
+\mathcal O(\|\mathbf D_c\mathbf C_c\mathbf D_c\|_2^2).
\end{aligned}
\]
In the row/left setting of Proposition~\ref{th_zero_rc}, if $\|\mathbf C_r\|_2$ is sufficiently small, then
$\operatorname{Orth}_\ell(\mathbf N_r)=(\mathbf I_m+\mathbf C_r)^{-1/2}\mathbf N_r
=\mathbf N_r-\tfrac12\mathbf C_r\mathbf N_r+\mathcal O(\|\mathbf C_r\|_2^2)$, and therefore
$\|\operatorname{Orth}_\ell(\mathbf N_r)-\mathbf N_r\|_2
\le \tfrac12\|\mathbf N_r\|_2\|\mathbf C_r\|_2+\mathcal O(\|\mathbf C_r\|_2^2)$.
Moreover,
\[
\begin{aligned}
\|\operatorname{Orth}_\ell(\mathbf M)-\mathbf M\|_2
\le{}& \|\mathbf D_r-\mathbf I_m\|_2\|\mathbf N_r\|_2 \\
&+ \frac{\kappa(\mathbf D_r)^2}{2}\|\mathbf N_r\|_2\|\mathbf C_r\|_2
+\mathcal O(\|\mathbf D_r\mathbf C_r\mathbf D_r\|_2^2).
\end{aligned}
\]
\end{proposition}

\begin{proof}
Consider first the column/right case. Applying Lemma~\ref{ana_lemma:a1} with $\mathbf A=\mathbf I_n$ and $\mathbf E=\mathbf C_c$ gives the Sylvester equation $\mathbf L+\mathbf L=\mathbf C_c$. Thus $\mathbf L=\tfrac12\mathbf C_c$, and
$(\mathbf I_n+\mathbf C_c)^{-1/2}=\mathbf I_n-\tfrac12\mathbf C_c+\mathcal O(\|\mathbf C_c\|_2^2)$.
Multiplying on the left by $\mathbf N_c$ gives the expansion of $\operatorname{Orth}_r(\mathbf N_c)$, and the norm bound follows from submultiplicativity.

To compare with $\mathbf M$, Proposition~\ref{th_zero_rc} gives
$\operatorname{Orth}_r(\mathbf M)=\mathbf N_c-\mathbf N_c\mathbf L_c\mathbf D_c^{-1}
+\mathcal O(\|\mathbf D_c\mathbf C_c\mathbf D_c\|_2^2)$.
Since $\mathbf M=\mathbf N_c\mathbf D_c$,
\[
\operatorname{Orth}_r(\mathbf M)-\mathbf M=\mathbf N_c(\mathbf I_n-\mathbf D_c)-\mathbf N_c\mathbf L_c\mathbf D_c^{-1}
+\mathcal O(\|\mathbf D_c\mathbf C_c\mathbf D_c\|_2^2).
\]
Lemma~\ref{ana_lemma:a2} yields
$\|\mathbf L_c\mathbf D_c^{-1}\|_2\le \tfrac12\kappa(\mathbf D_c)^2\|\mathbf C_c\|_2$, which gives the stated estimate.

The row/left case follows by the same argument. Applying Lemma~\ref{ana_lemma:a1} with $\mathbf A=\mathbf I_m$ and $\mathbf E=\mathbf C_r$ gives
$(\mathbf I_m+\mathbf C_r)^{-1/2}=\mathbf I_m-\tfrac12\mathbf C_r+\mathcal O(\|\mathbf C_r\|_2^2)$, and hence
$\operatorname{Orth}_\ell(\mathbf N_r)=(\mathbf I_m+\mathbf C_r)^{-1/2}\mathbf N_r
=\mathbf N_r-\tfrac12\mathbf C_r\mathbf N_r+\mathcal O(\|\mathbf C_r\|_2^2)$.
The norm bound again follows from submultiplicativity. Proposition~\ref{th_zero_rc} also gives
$\operatorname{Orth}_\ell(\mathbf M)=\mathbf N_r-\mathbf D_r^{-1}\mathbf L_r\mathbf N_r
+\mathcal O(\|\mathbf D_r\mathbf C_r\mathbf D_r\|_2^2)$, so
\[
\operatorname{Orth}_\ell(\mathbf M)-\mathbf M = (\mathbf I_m-\mathbf D_r)\mathbf N_r-\mathbf D_r^{-1}\mathbf L_r\mathbf N_r +\mathcal O(\|\mathbf D_r\mathbf C_r\mathbf D_r\|_2^2).
\]
The bound $\|\mathbf D_r^{-1}\mathbf L_r\|_2\le \tfrac12\kappa(\mathbf D_r)^2\|\mathbf C_r\|_2$ completes the proof.
\end{proof}

\begin{remark}
Proposition~\ref{prop:one_sided_norm} shows that one-sided normalization removes the zeroth-order marginal scale terms
$\|\mathbf N_c\|_2\|\mathbf D_c-\mathbf I_n\|_2$ and
$\|\mathbf D_r-\mathbf I_m\|_2\|\mathbf N_r\|_2$.
Whitening then acts only on the first-order Gram error. Since
$\operatorname{diag}(\mathbf N_c^\top\mathbf N_c)=\mathbf I_n$ and
$\operatorname{diag}(\mathbf N_r\mathbf N_r^\top)=\mathbf I_m$,
the leading correction is purely off-diagonal.
\end{remark}

\begin{corollary}
\label{cor:two_sided_norm}
Assume that every row and every column of $\mathbf M$ is nonzero, and set
$\hat{\mathbf M}:=\mathbf D_r^{-1/2}\mathbf M\mathbf D_c^{-1/2}$.

If $\mathbf M$ has full column rank and
$\hat{\mathbf C}_c:=\hat{\mathbf M}^\top\hat{\mathbf M}-\mathbf I_n$, then for $\|\hat{\mathbf C}_c\|_2$ sufficiently small,
$\operatorname{Orth}_r(\hat{\mathbf M})
=\hat{\mathbf M}(\mathbf I_n+\hat{\mathbf C}_c)^{-1/2}
=\hat{\mathbf M}-\tfrac12\hat{\mathbf M}\hat{\mathbf C}_c
+\mathcal O(\|\hat{\mathbf C}_c\|_2^2)$, and therefore
$\|\operatorname{Orth}_r(\hat{\mathbf M})-\hat{\mathbf M}\|_2
\le \tfrac12\|\hat{\mathbf M}\|_2\|\hat{\mathbf C}_c\|_2
+\mathcal O(\|\hat{\mathbf C}_c\|_2^2)$.

If $\mathbf M$ has full row rank and
$\hat{\mathbf C}_r:=\hat{\mathbf M}\hat{\mathbf M}^\top-\mathbf I_m$, then for $\|\hat{\mathbf C}_r\|_2$ sufficiently small,
$\operatorname{Orth}_\ell(\hat{\mathbf M})
=(\mathbf I_m+\hat{\mathbf C}_r)^{-1/2}\hat{\mathbf M}
=\hat{\mathbf M}-\tfrac12\hat{\mathbf C}_r\hat{\mathbf M}
+\mathcal O(\|\hat{\mathbf C}_r\|_2^2)$, and therefore
$\|\operatorname{Orth}_\ell(\hat{\mathbf M})-\hat{\mathbf M}\|_2
\le \tfrac12\|\hat{\mathbf M}\|_2\|\hat{\mathbf C}_r\|_2
+\mathcal O(\|\hat{\mathbf C}_r\|_2^2)$.
\end{corollary}

\begin{proof}
The same expansion applies to the two-sided normalized matrix $\hat{\mathbf M}$.

For the column-rank case, Lemma~\ref{ana_lemma:a1} with $\mathbf A=\mathbf I_n$ and $\mathbf E=\hat{\mathbf C}_c$ gives
$(\mathbf I_n+\hat{\mathbf C}_c)^{-1/2}
=\mathbf I_n-\tfrac12\hat{\mathbf C}_c+\mathcal O(\|\hat{\mathbf C}_c\|_2^2)$.
This yields the expansion of $\operatorname{Orth}_r(\hat{\mathbf M})$ and the corresponding norm bound.

For the row-rank case, Lemma~\ref{ana_lemma:a1} with $\mathbf A=\mathbf I_m$ and $\mathbf E=\hat{\mathbf C}_r$ gives
$(\mathbf I_m+\hat{\mathbf C}_r)^{-1/2}
=\mathbf I_m-\tfrac12\hat{\mathbf C}_r+\mathcal O(\|\hat{\mathbf C}_r\|_2^2)$.
This yields the expansion of $\operatorname{Orth}_\ell(\hat{\mathbf M})$ and the corresponding norm bound.
\end{proof}

\begin{remark}
After two-sided normalization, whitening is centered at the identity Gram matrix of $\hat{\mathbf M}$. It no longer corrects marginal row or column scales; it only removes the residual mismatch left after diagonal equilibration.
\end{remark}
\section{Lemmas for Proposition \ref{th_conv_rc}}
\subsection{Lemma \ref{conv_lemma:a1}}

\begin{lemma}
\label{conv_lemma:a1}
Let the momentum iterate satisfy $\mathbf{M}_t=\beta \mathbf{M}_{t-1}+(1-\beta)\mathbf{G}_t,\beta\in[0,1)$, and assume there exists a constant $G_\infty>0$ such that $\|\mathbf{G}_t\|_\infty \le G_\infty,\|\mathbf{M}_0\|_\infty \le G_\infty$. For $\varepsilon>0$, define
$\mathbf{D}_{r,t}:=\operatorname{diag}\bigl(\operatorname{rowsum}(\mathbf{M}_t\odot \mathbf{M}_t)+\varepsilon\bigr),
\mathbf{D}_{c,t}:=\operatorname{diag}\bigl(\operatorname{colsum}(\mathbf{M}_t\odot \mathbf{M}_t)+\varepsilon\bigr)$, and $\mathbf{P}_t:=\mathbf{D}_{r,t}^{-1/2},\mathbf{Q}_t:=\mathbf{D}_{c,t}^{-1/2}$. Then, for all $t$,
\[
(nG_\infty^2+\varepsilon)^{-1/2}\mathbf{I}_m
\preceq\mathbf{P}_t\preceq\varepsilon^{-1/2}\mathbf{I}_m,\quad
(mG_\infty^2+\varepsilon)^{-1/2}\mathbf{I}_n\preceq\mathbf{Q}_t\preceq\varepsilon^{-1/2}\mathbf{I}_n.
\]
Moreover,
\[
\sqrt{\varepsilon}\mathbf{I}_m\preceq\mathbf{P}_t^{-1}\preceq\sqrt{nG_\infty^2+\varepsilon}\mathbf{I}_m,\quad \sqrt{\varepsilon}\mathbf{I}_n
\preceq\mathbf{Q}_t^{-1}\preceq\sqrt{mG_\infty^2+\varepsilon}\mathbf{I}_n.
\]
\end{lemma}

\begin{proof}
We first show that the momentum iterate $\mathbf{M}_t$ is uniformly bounded entrywise.

Since
\[
\mathbf{M}_t=\beta \mathbf{M}_{t-1}+(1-\beta)\mathbf{G}_t,
\]
we have
\[
\|\mathbf{M}_t\|_\infty
\le \beta \|\mathbf{M}_{t-1}\|_\infty + (1-\beta)\|\mathbf{G}_t\|_\infty \le \beta \|\mathbf{M}_{t-1}\|_\infty + (1-\beta)G_\infty.
\]
Using $\|\mathbf{M}_0\|_{\infty}=0\le G_{\infty}$, an induction on $t$ yields
\[
\|\mathbf{M}_t\|_\infty \le G_\infty,
\qquad \forall t.
\]
Therefore, every entry of $\mathbf{M}_t$ satisfies
\[
|(\mathbf{M}_t)_{ij}| \le G_\infty,
\qquad \forall i,j,t.
\]
Now fix any row $i$. Since that row has $n$ entries,
\[
\sum_{j=1}^n (\mathbf{M}_t)_{ij}^2 \le n G_\infty^2.
\]
Hence each diagonal entry of $\mathbf{D}_{r,t}$ satisfies
\[
\varepsilon \le (\mathbf{D}_{r,t})_{ii} = \sum_{j=1}^n (\mathbf{M}_t)_{ij}^2+\varepsilon \le nG_\infty^2+\varepsilon.
\]
Similarly, for any column $j$,
\[
\sum_{i=1}^m (\mathbf{M}_t)_{ij}^2 \le m G_\infty^2,
\]
so each diagonal entry of $\mathbf{D}_{c,t}$ satisfies
\[
\varepsilon \le (\mathbf{D}_{c,t})_{jj}
=\sum_{i=1}^m (\mathbf{M}_t)_{ij}^2+\varepsilon
\le mG_\infty^2+\varepsilon.
\]
Since $\mathbf{P}_t=\mathbf{D}_{r,t}^{-1/2}$ and $\mathbf{Q}_t=\mathbf{D}_{c,t}^{-1/2}$ are diagonal matrices, their diagonal entries are
\[
(\mathbf{P}_t)_{ii}
=\left(\sum_{j=1}^n (\mathbf{M}_t)_{ij}^2+\varepsilon\right)^{-1/2},\quad (\mathbf{Q}_t)_{jj}
=\left(\sum_{i=1}^m(\mathbf{M}_t)_{ij}^2+\varepsilon\right)^{-1/2}.
\]
Using the bounds above, we obtain
\[
(nG_\infty^2+\varepsilon)^{-1/2}\le(\mathbf{P}_t)_{ii}\le\varepsilon^{-1/2},\quad (mG_\infty^2+\varepsilon)^{-1/2}\le(\mathbf{Q}_t)_{jj}\le\varepsilon^{-1/2}.
\]
Since positive diagonal matrices are ordered entrywise in the Loewner sense, it follows that
\[
(nG_\infty^2+\varepsilon)^{-1/2}\mathbf{I}_m
\preceq\mathbf{P}_t\preceq\varepsilon^{-1/2}\mathbf{I}_m,\quad (mG_\infty^2+\varepsilon)^{-1/2}\mathbf{I}_n
\preceq\mathbf{Q}_t\preceq\varepsilon^{-1/2}\mathbf{I}_n.
\]
Because $\mathbf{P}_t$ and $\mathbf{Q}_t$ are symmetric positive definite diagonal matrices, their singular values are exactly their diagonal entries. Therefore,
\[
(nG_\infty^2+\varepsilon)^{-1/2}\le\sigma_{\min}(\mathbf{P}_t)\le\|\mathbf{P}_t\|_2\le\varepsilon^{-1/2},\quad (mG_\infty^2+\varepsilon)^{-1/2}\le\sigma_{\min}(\mathbf{Q}_t)\le\|\mathbf{Q}_t\|_2
\le\varepsilon^{-1/2}.
\]
Finally, inverting the above matrix inequalities gives
\[
\sqrt{\varepsilon}\mathbf{I}_m\preceq\mathbf{P}_t^{-1}\preceq\sqrt{nG_\infty^2+\varepsilon}\mathbf{I}_m,\quad \sqrt{\varepsilon}\mathbf{I}_n\preceq\mathbf{Q}_t^{-1}\preceq\sqrt{mG_\infty^2+\varepsilon}\mathbf{I}_n.
\]
This completes the proof.
\end{proof}

\subsection{Lemma \ref{conv_lemma:a2}}
\begin{lemma}
\label{conv_lemma:a2}
Suppose that $\{E_t,A_t\}$ are nonnegative sequences. Assume
$
E_{t+1}\le (1-\alpha_{t+1})E_t + A_{t+1}
$
where $\alpha_t=t^{-p}$, $p\in(0,1]$. Then
$
\alpha_t E_t \le 2(E_t-E_{t+1}+A_{t+1}).
$
\end{lemma}
\begin{proof}
See \cite{chang2025convergence}, Lemma A.3.
\end{proof}

\subsection{Lemma \ref{conv_lemma:a3}}
\begin{lemma}
\label{conv_lemma:a3}
For Algorithm \ref{alg:muoneq}, the accumulated error between the momentum term and the true gradient is bounded for $t\ge1$:
\[
\begin{aligned}
\mathbb{E}\left[\|\mathbf{M}_{t+1}-\nabla f(\mathbf{X}_{t+1})\|_F^2\right] \leq &\beta_{t+1}\mathbb{E}\left[\|\mathbf{M}_{t}-\nabla f(\mathbf{X}_{t})\|_F^2\right]\\
&+\frac{\beta_{t+1}^2}{1-\beta_{t+1}}L^2\eta_t^2\|a\mathbf{O}_{t} + \lambda_t \mathbf{X}_t\|_F^2+(1-\beta_{t+1})^2\sigma^2.
\end{aligned}
\]
\end{lemma}

\begin{proof}
First, we have
\[
\begin{aligned}
&\| \mathbf{M}_{t+1}-\nabla f(\mathbf{X}_{t+1})\|_{F}^{2} \\
& =\|\beta_{t+1} \mathbf{M}_{t}+(1-\beta_{t+1})\nabla f(\mathbf{X}_{t+1};\xi_{t+1})-\nabla f(\mathbf{X}_{t+1})\|_{F}^{2} \\
& =\|\beta_{t+1}(\mathbf{M}_{t} - \nabla f(\mathbf{X}_{t})) + (1 - \beta_{t+1})(\nabla f(\mathbf{X}_{t+1};\xi_{t+1}) - \nabla f(\mathbf{X}_{t+1}) )\\
&\quad+\beta_{t+1} (\nabla f(\mathbf{X}_{t}) - \nabla f(\mathbf{X}_{t+1}))\|_F^2 \\
& =\beta_{t+1}^2\|\mathbf{M}_{t}-\nabla f(\mathbf{X}_{t})\|_F^2+\beta_{t+1}^2\|\nabla f(\mathbf{X}_{t})-\nabla f(\mathbf{X}_{t+1})\|_F^2 \\
&\quad+(1-\beta_{t+1})^2\|\nabla f(\mathbf{X}_{t+1};\xi_{t+1}) - \nabla f(\mathbf{X}_{t+1})\|_F^2 \\
& \quad +2\beta_{t+1}^2\langle \mathbf{M}_{t}-\nabla f(\mathbf{X}_{t}),\nabla f(\mathbf{X}_{t})-\nabla f(\mathbf{X}_{t+1})\rangle_F \\
& \quad +2\beta_{t+1}(1-\beta_{t+1})\langle \mathbf{M}_{t}-\nabla f(\mathbf{X}_{t}),\nabla f(\mathbf{X}_{t+1};\xi_{t+1})-\nabla f(\mathbf{X}_{t+1})\rangle_F \\
& \quad +2\beta_{t+1}(1-\beta_{t+1})\langle\nabla f(\mathbf{X}_{t})-\nabla f(\mathbf{X}_{t+1}),\nabla f(\mathbf{X}_{t+1};\xi_{t+1})-\nabla f(\mathbf{X}_{t+1})\rangle_{F}.
\end{aligned}
\]
According to Assumption \ref{ass:3}. Taking the expectation of its squared norm, and  using the unbiasedness and independence of the stochastic gradient, we obtain:
\[
\begin{aligned}
\mathbb{E}[\|\mathbf{M}_{t+1}-\nabla f(\mathbf{X}_{t+1})\|_F^2] = &\beta_{t+1}^2\mathbb{E}[\|\mathbf{M}_{t}-\nabla f(\mathbf{X}_{t})\|_F^2] \\
&\quad+ (1-\beta_{t+1})^2\mathbb{E}[\|\nabla f(\mathbf{X}_{t+1};\xi_{t+1}) - \nabla f(\mathbf{X}_{t+1})\|_F^2] \\
&+ \beta_{t+1}^2\mathbb{E}[\|\nabla f(\mathbf{X}_{t}) - \nabla f(\mathbf{X}_{t+1})\|_F^2] \\
&+ 2\beta_{t+1}^2\mathbb{E}[\langle \mathbf{M}_{t}-\nabla f(\mathbf{X}_{t}), \nabla f(\mathbf{X}_{t}) - \nabla f(\mathbf{X}_{t+1}) \rangle].
\end{aligned}
\]
Applying Young's inequality with a parameter ($ab \le \frac{\epsilon}{2}a^2 + \frac{1}{2\epsilon}b^2$), we have
\[
\langle \mathbf{M}_{t}-\nabla f(\mathbf{X}_{t}),\nabla f(\mathbf{X}_{t})-\nabla f(\mathbf{X}_{t+1})\rangle_F\leq\frac{\epsilon}{2}\|\mathbf{M}_{t}-\nabla f(\mathbf{X}_{t})\|_F^2+\frac{1}{2\epsilon}\|\nabla f(\mathbf{X}_{t})-\nabla f(\mathbf{X}_{t+1})\|_F^2.
\]
Thus, we have:
\[
\begin{aligned}
\mathbb{E}\left[\|\mathbf{M}_{t+1}-\nabla f(\mathbf{X}_{t+1})\|_{F}^{2}\right] & \le \beta_{t+1}^{2}(1+\epsilon)\mathbb{E}\left[\|\mathbf{M}_{t}-\nabla f(\mathbf{X}_{t})\|_{F}^{2}\right]\\
&\quad +\beta_{t+1}^{2}\left(1+\frac{1}{\epsilon}\right)\mathbb{E}\left[\|\nabla f(\mathbf{X}_{t})-\nabla f(\mathbf{X}_{t+1})\|_{F}^{2}\right] \\
& \quad +(1-\beta_{t+1})^{2}\mathbb{E}\left[\|\nabla f(\mathbf{X}_{t+1};\xi_{t+1})-\nabla f(\mathbf{X}_{t+1})\|_{F}^{2}\right].
\end{aligned}
\]
According to Assumption \ref{ass:2},
\[
\begin{aligned}
\|\nabla f(\mathbf{X}_{t})-\nabla f(\mathbf{X}_{t+1})\|_F^2&\leq L^2\|\mathbf{X}_{t}-\mathbf{X}_{t+1}\|_F^2\\
&=L^2\eta_t^2\|a\mathbf{O}_{t} + \lambda_t \mathbf{X}_t\|_F^2\\
\end{aligned}
\]
Therefore:
\[
\begin{aligned}
\mathbb{E}\left[\left\|\mathbf{M}_{t+1}-\nabla f(\mathbf{X}_{t+1})\right\|_F^2\right]&\leq\beta_{t+1}^2(1+\epsilon)\mathbb{E}\left[\left\|\mathbf{M}_{t}-\nabla f(\mathbf{X}_{t})\right\|_F^2\right]\\
&+\beta_{t+1}^2\left(1+\frac{1}{\epsilon}\right)L^2\eta_t^2\|a\mathbf{O}_{t} + \lambda_t \mathbf{X}_t\|_F^2+(1-\beta_{t+1})^2\sigma^2.
\end{aligned}
\]
Then, by letting $\epsilon := \frac{1 - \beta_{t+1}}{\beta_{t+1}},t\ge1$, we have
\begin{equation}
\label{eq1:a}
\begin{aligned}
\mathbb{E}\left[\|\mathbf{M}_{t+1}-\nabla f(\mathbf{X}_{t+1})\|_F^2\right] \leq &\beta_{t+1}\mathbb{E}\left[\|\mathbf{M}_{t}-\nabla f(\mathbf{X}_{t})\|_F^2\right]\\
&+\frac{\beta_{t+1}^2}{1-\beta_{t+1}}L^2\eta_t^2\|a\mathbf{O}_{t} + \lambda_t \mathbf{X}_t\|_F^2+(1-\beta_{t+1})^2\sigma^2.
\end{aligned}
\end{equation}
\end{proof}

\section{Proofs of Proposition \ref{th_conv_rc}}
\label{proof:th_conv_rc}
\begin{proof}
By Assumption~\ref{ass:2} and the descent lemma,
\[
\begin{aligned}
f(\mathbf{X}_{t+1})
&\le f(\mathbf{X}_{t}) + \langle \nabla f(\mathbf{X}_{t}), \mathbf{X}_{t+1}-\mathbf{X}_{t} \rangle + \frac{L}{2}\|\mathbf{X}_{t+1}-\mathbf{X}_{t}\|_F^2\\
&= f(\mathbf{X}_t)-a\eta_t\langle \nabla f(\mathbf{X}_t),\mathbf{O}_t\rangle +\frac{La^2\eta_t^2}{2}\|\mathbf{O}_t\|_F^2\\
&\le f(\mathbf{X}_t)-a\eta_t\langle \nabla f(\mathbf{X}_t),\mathbf{O}_t\rangle+\frac{La^2n}{2}\eta_t^2,
\end{aligned}
\]
where the last step uses $\|\mathbf{O}_t\|_F^2\le n$.

Let $
\mathbf{S}_t:=\nabla f(\mathbf{X}_t)-\mathbf{M}_t,
r_\varepsilon:=\sqrt{\varepsilon},
R_\varepsilon:=\sqrt{nG_\infty^2+\varepsilon},
C_\varepsilon:=\sqrt{mG_\infty^2+\varepsilon}.$

Define the midpoint-rescaled matrices
\[
\widetilde{\mathbf{P}}_t:=\frac{R_\varepsilon+r_\varepsilon}{2}\mathbf{P}_t,
\widetilde{\mathbf{Q}}_t:=\frac{C_\varepsilon+r_\varepsilon}{2}\mathbf Q_t,
\widetilde{\mathbf{A}}_t:=\widetilde{\mathbf{P}}_t^{-1},
\widetilde{\mathbf{B}}_t:=\widetilde{\mathbf{Q}}_t^{-1}.
\]
Since the rescaling is by positive scalars,
\[
\operatorname{polar}(\widetilde{\mathbf{P}}_t\mathbf{M}_t\widetilde{\mathbf{Q}}_t)=\operatorname{polar}(\mathbf{P}_t\mathbf{M}_t\mathbf{Q}_t)=\mathbf{O}_t.
\]
Hence
\[
\langle \nabla f(\mathbf{X}_t),\mathbf{O}_t\rangle
=\left\langle
\widetilde{\mathbf{A}}_t\nabla f(\mathbf{X}_t)\widetilde{\mathbf{B}}_t,
\widetilde{\mathbf{P}}_t\mathbf{O}_t\widetilde{\mathbf{Q}}_t
\right\rangle.
\]
By Lemma~\ref{conv_lemma:a1}, we have
\[
R_\varepsilon^{-1}\mathbf{I}_m \preceq \mathbf{P}_t
\preceq r_\varepsilon^{-1}\mathbf{I}_m,
\qquad
C_\varepsilon^{-1}\mathbf{I}_n \preceq \mathbf{Q}_t\preceq r_\varepsilon^{-1}\mathbf{I}_n.
\]
Since $R_\varepsilon=r_\varepsilon\rho_r, C_\varepsilon=r_\varepsilon\rho_c$, we obtain
\[
\frac{\rho_r+1}{2\rho_r}\mathbf I_m \preceq \widetilde{\mathbf P}_t \preceq\frac{\rho_r+1}{2}\mathbf I_m,
\qquad
\frac{\rho_c+1}{2\rho_c}\mathbf I_n\preceq\widetilde{\mathbf Q}_t\preceq\frac{\rho_c+1}{2}\mathbf I_n.
\]
Therefore, we have
\[
\|\widetilde{\mathbf A}_t\|_2
\le \frac{2\rho_r}{\rho_r+1},
\quad
\|\widetilde{\mathbf B}_t\|_2\le\frac{2\rho_c}{\rho_c+1},
\quad
\|\widetilde{\mathbf A}_t-\mathbf I_m\|_2\le\frac{\rho_r-1}{\rho_r+1},
\quad
\|\widetilde{\mathbf B}_t-\mathbf I_n\|_2\le\frac{\rho_c-1}{\rho_c+1}.
\]
Using $\nabla f(\mathbf{X}_t)=\mathbf{M}_t+\mathbf{S}_t$, we split
\[
\begin{aligned}
\langle \nabla f(\mathbf{X}_t),\mathbf{O}_t\rangle
&=\left\langle
\widetilde{\mathbf{A}}_t\mathbf{M}_t\widetilde{\mathbf{B}}_t,
\widetilde{\mathbf{P}}_t\mathbf{O}_t\widetilde{\mathbf{Q}}_t
\right\rangle+\left\langle \widetilde{\mathbf{A}}_t\mathbf{S}_t\widetilde{\mathbf{B}}_t,
\widetilde{\mathbf{P}}_t\mathbf{O}_t\widetilde{\mathbf{Q}}_t
\right\rangle\\
&=\|\widetilde{\mathbf{P}}_t\mathbf{M}_t\widetilde{\mathbf{Q}}_t\|_*+\left\langle
\widetilde{\mathbf{A}}_t\mathbf{S}_t\widetilde{\mathbf{B}}_t,
\widetilde{\mathbf{P}}_t\mathbf{O}_t\widetilde{\mathbf{Q}}_t
\right\rangle+\left\langle
\widetilde{\mathbf{A}}_t\mathbf{M}_t\widetilde{\mathbf{B}}_t-\mathbf{M}_t, \widetilde{\mathbf{P}}_t\mathbf{O}_t\widetilde{\mathbf{Q}}_t
\right\rangle.
\end{aligned}
\]
\textit{(i)} First, by the singular-value inequality
\[
\sigma_i(\mathbf{A}\mathbf{X}\mathbf{B})\ge\sigma_{\min}(\mathbf{A})\sigma_i(\mathbf{X})\sigma_{\min}(\mathbf{B}),
\]
we have
\[
\|\widetilde{\mathbf{P}}_t\mathbf{M}_t\widetilde{\mathbf{Q}}_t\|_*\ge\frac{(\rho_r+1)(\rho_c+1)}{4\rho_r\rho_c}\|\mathbf{M}_t\|_*.
\]
\textit{(ii)} Second, using Hölder's inequality, $\|\mathbf{O}_t\|_2=1$, and the bounds above,
\[
\begin{aligned}
\left|
\left\langle
\widetilde{\mathbf{A}}_t\mathbf{S}_t\widetilde{\mathbf{B}}_t,
\widetilde{\mathbf{P}}_t\mathbf{O}_t\widetilde{\mathbf{Q}}_t
\right\rangle
\right|
&\le\|\widetilde{\mathbf{A}}_t\mathbf{S}_t\widetilde{\mathbf{B}}_t\|_*
\|\widetilde{\mathbf{P}}_t\mathbf{O}_t\widetilde{\mathbf{Q}}_t\|_2\\
&\le\|\widetilde{\mathbf{A}}_t\|_2
\|\widetilde{\mathbf{B}}_t\|_2
\|\mathbf{S}_t\|_*
\|\widetilde{\mathbf{P}}_t\|_2
\|\widetilde{\mathbf{Q}}_t\|_2\\
&\le\rho_r\rho_c \|\mathbf{S}_t\|_*\\
&\le \rho_r\rho_c \sqrt n\|\mathbf{S}_t\|_F.
\end{aligned}
\]
\textit{(iii)} Third, we use the symmetric decomposition
\[
\widetilde{\mathbf{A}}_t\mathbf{M}_t\widetilde{\mathbf{B}}_t-\mathbf{M}_t
=(\widetilde{\mathbf{A}}_t-\mathbf{I}_m)\mathbf{M}_t+\mathbf{M}_t(\widetilde{\mathbf{B}}_t-\mathbf{I}_n)+(\widetilde{\mathbf{A}}_t-\mathbf{I}_m)\mathbf{M}_t(\widetilde{\mathbf{B}}_t-\mathbf{I}_n).
\]
Hence
\[
\begin{aligned}
\|\widetilde{\mathbf{A}}_t\mathbf{M}_t\widetilde{\mathbf{B}}_t-\mathbf{M}_t\|_*
&\le\left(\frac{\rho_r - 1}{\rho_r + 1}+\frac{\rho_c - 1}{\rho_c + 1}+\frac{(\rho_r-1)(\rho_c -1)}{(\rho_r + 1)(\rho_c + 1)}\right)\|\mathbf{M}_t\|_*.
\end{aligned}
\]
Therefore, we have
\[
\begin{aligned}
\left|
\left\langle
\widetilde{\mathbf{A}}_t\mathbf{M}_t\widetilde{\mathbf{B}}_t-\mathbf{M}_t,
\widetilde{\mathbf{P}}_t\mathbf{O}_t\widetilde{\mathbf{Q}}_t
\right\rangle
\right|
&\le\|\widetilde{\mathbf{A}}_t\mathbf{M}_t\widetilde{\mathbf{B}}_t-\mathbf{M}_t\|_*
\|\widetilde{\mathbf{P}}_t\mathbf{O}_t\widetilde{\mathbf{Q}}_t\|_2\\
&\le\frac{3\rho_r\rho_c - \rho_r - \rho_c - 1}{4}\|\mathbf{M}_t\|_*.
\end{aligned}
\]
Combining the above bounds, we obtain
\[
\langle \nabla f(\mathbf{X}_t),\mathbf{O}_t\rangle
\ge\chi_\varepsilon \|\mathbf{M}_t\|_*-\rho_r\rho_c\sqrt n \|\mathbf{S}_t\|_F.
\]
Under the setting $\varepsilon \ge \frac45 G_\infty^2\max\{m,n\}$, we have $\rho_r=\sqrt{1+\frac{nG_\infty^2}{\varepsilon}}
\le\sqrt{1+\frac{5n}{4\max\{m,n\}}}
\le\frac32, \rho_c=\sqrt{1+\frac{mG_\infty^2}{\varepsilon}}
\le\sqrt{1+\frac{5m}{4\max\{m,n\}}}
\le\frac32$.

Let $\psi(x,y):=
\frac{(x+1)(y+1)}{4xy}-\frac{3xy-x-y-1}{4},x,y\ge 1.$ A direct calculation gives $\partial_x\psi(x,y)=\frac{1-x^{-2}-x^{-2}y^{-1}-3y}{4}<0,
\partial_y\psi(x,y)=\frac{1-y^{-2}-x^{-1}y^{-2}-3x}{4}<0$, for all $x,y\ge1$. Hence $\psi$ is decreasing in each argument on $[1,\infty)^2$. Therefore, we have $\chi_\varepsilon=\psi(\rho_r,\rho_c)\ge\psi\left(\frac32,\frac32\right)=\frac1{144}>0$. Since $\chi_\varepsilon>0$, $\|\mathbf{M}_t\|_*\ge \|\mathbf{M}_t\|_F$, and
\[
\|\mathbf{M}_t\|_F\ge\|\nabla f(\mathbf{X}_t)\|_F-\|\mathbf{S}_t\|_F,
\]
it follows that
\[
\langle \nabla f(\mathbf{X}_t),\mathbf{O}_t\rangle\ge\chi_\varepsilon \|\nabla f(\mathbf{X}_t)\|_F
-\bigl(\chi_\varepsilon+\rho_r\rho_c\sqrt n\bigr)\|\mathbf{S}_t\|_F.
\]
Substituting the above inequality into the descent bound yields
\[
\begin{aligned}
f(\mathbf{X}_{t+1})
\le&
f(\mathbf{X}_t)
-a\chi_\varepsilon\eta_t\|\nabla f(\mathbf{X}_t)\|_F\\
&\quad
+a\bigl(\chi_\varepsilon+\rho_r\rho_c\sqrt n\bigr)\eta_t\|\mathbf{S}_t\|_F
+\frac{La^2n}{2}\eta_t^2.
\end{aligned}
\]
Apply Young's inequality with parameter $h_{t+1}/L$:
\[
a\bigl(\chi_\varepsilon+\rho_r\rho_c\sqrt n\bigr)\eta_t\|\mathbf{S}_t\|_F
\le \frac{ah_{t+1}}{2L}\|\mathbf{S}_t\|_F^2
+ \frac{aL}{2}
\bigl(\chi_\varepsilon+\rho_r\rho_c\sqrt n\bigr)^2
\frac{\eta_t^2}{h_{t+1}}.
\]
Hence
\[
\begin{aligned}
f(\mathbf{X}_{t+1})
\le& f(\mathbf{X}_t)-a\chi_\varepsilon\eta_t\|\nabla f(\mathbf{X}_t)\|_F+\frac{ah_{t+1}}{2L}\|\mathbf{S}_t\|_F^2\\
&+\frac{aL}{2}\bigl(\chi_\varepsilon+\rho_r\rho_c\sqrt n\bigr)^2
\frac{\eta_t^2}{h_{t+1}}+\frac{La^2n}{2}\eta_t^2.
\end{aligned}
\]
Taking expectations and summing the above inequality over $t=1,\dots,T$, then using $f(\mathbf{X}_{T+1})\ge f^{\star}$, we have
\[
\begin{aligned}
a\chi_\varepsilon \sum_{t=1}^T \eta_t\mathbb E\|\nabla f(\mathbf{X}_t)\|_F
&\le f(\mathbf{X}_1)-f^{\star}
+ \frac{a}{2L}\sum_{t=1}^T h_{t+1}\mathbb E\|\mathbf{S}_t\|_F^2 \\
& + \frac{aL}{2}\bigl(\chi_\varepsilon+\rho_r\rho_c\sqrt n\bigr)^2
\sum_{t=1}^T \frac{\eta_t^2}{h_{t+1}} +\frac{La^2n}{2}\sum_{t=1}^T \eta_t^2.
\end{aligned}
\]
By Lemma~\ref{conv_lemma:a3} with $\lambda=0$,
\[
\mathbb{E}\|\mathbf{S}_{t+1}\|_F^2
\le(1-h_{t+1})\mathbb{E}\|\mathbf{S}_{t}\|_F^2
+\frac{(1-h_{t+1})^2}{h_{t+1}}L^2\eta_t^2 a^2n
+h_{t+1}^2\sigma^2.
\]
Since $(1-h_{t+1})^2\le 1$ and
\[
\frac{\eta_t^2}{h_{t+1}}
=\frac{t^{-3/2}}{(t+1)^{-1/2}}
=\frac{\sqrt{t+1}}{t^{3/2}}
\le\frac{2\sqrt2}{t+1}
=2\sqrt2h_{t+1}^2,
\]
we have
\[
\mathbb{E}\|\mathbf{S}_{t+1}\|_F^2
\le (1-h_{t+1})\mathbb{E}\|\mathbf{S}_{t}\|_F^2 +h_{t+1}^2(2\sqrt{2}L^2a^2n+\sigma^2).
\]
According to Lemma~\ref{conv_lemma:a2}, by letting $E_t=\mathbb E\|\mathbf S_t\|_F^2, A_{t+1}=h_{t+1}^2(2\sqrt{2}L^2a^2n+\sigma^2)$, we have
\[
h_t\mathbb E\|\mathbf{S}_t\|_F^2 \le2\left(\mathbb E\|\mathbf{S}_t\|_F^2 -\mathbb E\|\mathbf{S}_{t+1}\|_F^2+A_{t+1}\right).
\]
Moreover, since $\beta_1=0$,
\[
\begin{aligned}
\mathbb{E}\|\mathbf{S}_1\|_F^2
&=\mathbb{E}\|\nabla f(\mathbf{X}_1)-\mathbf{M}_1\|_F^2 \\
&=\mathbb{E}\|\nabla f(\mathbf{X}_1)-\nabla f(\mathbf{X}_1;\xi_1)\|_F^2
\le\sigma^2.
\end{aligned}
\]
It follows that
\[
\begin{aligned}
\sum_{t=1}^T h_t\mathbb E\|\mathbf{S}_t\|_F^2
&\le2\sum_{t=1}^T
\left(\mathbb E\|\mathbf{S}_t\|_F^2-\mathbb E\|\mathbf{S}_{t+1}\|_F^2+A_{t+1}\right) \\
&\le2\mathbb E\|\mathbf{S}_1\|_F^2+2(2\sqrt{2}L^2a^2n+\sigma^2)\sum_{t=1}^T \frac1{t+1} \\
&\le2\sigma^2+2(2\sqrt{2}L^2a^2n+\sigma^2)(1+\ln T).
\end{aligned}
\]
Hence, since $h_{t+1}\le h_t$,
\[
\sum_{t=1}^T h_{t+1}\mathbb E\|\mathbf{S}_t\|_F^2 \le 2\sigma^2+2(2\sqrt{2}L^2a^2n+\sigma^2)(1+\ln T).
\]
Also,
\[
\sum_{t=1}^T \frac{\eta_t^2}{h_{t+1}}
=\sum_{t=1}^T \frac{t^{-3/2}}{(t+1)^{-1/2}} \le 2(1+\ln T),
\qquad
\sum_{t=1}^T \eta_t^2=\sum_{t=1}^T t^{-3/2}\le3.
\]
Substituting these bounds into the previous inequality, we get
\[
a\chi_\varepsilon \sum_{t=1}^T \eta_t\mathbb E\|\nabla f(\mathbf{X}_t)\|_F \le f(\mathbf{X}_1)-f^{\star} + C_{1}(1+\ln T)+ C_{2},
\]
where
\[
C_{1} =\frac{a}{L}\bigl(2\sqrt{2}L^2a^2n+\sigma^2\bigr)+ aL\bigl(\chi_\varepsilon+\rho_r\rho_c\sqrt n\bigr)^2,
\quad 
C_{2}=\frac{a\sigma^2}{L}+\frac{3La^2n}{2}.
\]
Finally, since $\eta_t=t^{-3/4}\ge T^{-3/4}$ for all $1\le t\le T$,
\[
T^{-3/4}\sum_{t=1}^T \mathbb E\|\nabla f(\mathbf{X}_t)\|_F
\le \sum_{t=1}^T \eta_t\mathbb E\|\nabla f(\mathbf{X}_t)\|_F.
\]
Therefore, we have
\[
\frac1T\sum_{t=1}^T \mathbb E\|\nabla f(\mathbf{X}_t)\|_F \le \frac{f(\mathbf{X}_1)-f^{\star} + C_{1}(1+\ln T)+ C_{2}}{a\chi_\varepsilon \cdot T^{1/4}}.
\]
This completes the proof.
\end{proof}
\section{Lemmas for Theorem \ref{th_conv_r} and Corollary \ref{cor:conv_r_ns}}
\subsection{Lemma \ref{conv_r_lemma:a1}}

% \begin{lemma}
% \label{conv_r_lemma:a1}
% Let $\mathbf{D}_{r,t}:=\operatorname{diag}\bigl(\operatorname{rowsum}(\mathbf{M}_t\odot\mathbf{M}_t)\bigr),
% \mathbf{R}_t:=\mathbf{D}_{r,t}^{-1/2}\mathbf{M}_t,\mathbf{O}_t:=\operatorname{Orth}(\mathbf{R}_t)$. Then, for all $t$,
% \[
% \langle \mathbf{M}_t,\mathbf{O}_t\rangle\ge \frac1{\sqrt m}\|\mathbf{M}_t\|_F,
% \qquad
% \|\mathbf{O}_t\|_F^2=\operatorname{rank}(\mathbf{R}_t)\le n.
% \]
% \end{lemma}
\begin{lemma}
\label{conv_r_lemma:a1}
Let
\[
\mathbf{D}_{r,t}:=
\operatorname{diag}\bigl(\operatorname{rowsum}(\mathbf{M}_t\odot\mathbf{M}_t)\bigr),
\quad
\mathbf{P}_{r,t}:=(\mathbf{D}_{r,t}^{1/2})^\dagger,
\quad
\mathbf{R}_t:=\mathbf{P}_{r,t}\mathbf{M}_t,
\quad
\mathbf{O}_t:=\operatorname{Orth}(\mathbf{R}_t),
\]
where \(^{\dagger}\) denotes the Moore--Penrose pseudoinverse, so zero rows
of \(\mathbf M_t\) remain zero after row normalization, and
\(\operatorname{Orth}(\mathbf 0):=\mathbf 0\). Then, for all \(t\),
\[
\langle \mathbf{M}_t,\mathbf{O}_t\rangle\ge \frac1{\sqrt m}\|\mathbf{M}_t\|_F,
\qquad
\|\mathbf{O}_t\|_F^2=\operatorname{rank}(\mathbf{R}_t)\le n.
\]
\end{lemma}

\begin{proof}
\textit{(i)} If $\mathbf{M}_t=\mathbf{0}$, then $\mathbf{R}_t=\mathbf{O}_t=\mathbf{0}$ and the claim is immediate. 
% When $\varepsilon=0$ in the theorem~\ref{th_conv_r}, DiagPre is interpreted via the Moore–Penrose pseudoinverse, with zero rows mapped to zero; experiments use $\varepsilon=10^{-8}$ for numerical stability.

\textit{(ii)} Assume $\mathbf{M}_t\neq \mathbf{0}$. Let
\[
d_{i,t}:=\| (\mathbf{M}_t)_{i,:}\|_2,
\qquad
\mathbf{D}_t:=\operatorname{diag}(d_{1,t},\ldots,d_{m,t}),
\]
so that $\mathbf{D}_t=\mathbf{D}_{r,t}^{1/2}$ and $\mathbf{M}_t=\mathbf{D}_t\mathbf{R}_t$. Let $\mathbf{H}_t:=(\mathbf{R}_t\mathbf{R}_t^\top)^{1/2}$. If $\mathbf{R}_t=\mathbf{U}_t\boldsymbol\Sigma_t\mathbf{V}_t^\top$ is a compact SVD, then $\mathbf{O}_t=\mathbf{U}_t\mathbf{V}_t^\top$ and therefore $\mathbf{O}_t\mathbf{R}_t^\top=\mathbf{H}_t$. Hence
\[
\langle \mathbf{M}_t,\mathbf{O}_t\rangle
=\operatorname{tr}(\mathbf{O}_t^\top\mathbf{D}_t\mathbf{R}_t)
=\operatorname{tr}(\mathbf{D}_t\mathbf{O}_t\mathbf{R}_t^\top)
=\operatorname{tr}(\mathbf{D}_t\mathbf{H}_t)
=\sum_{i=1}^m d_{i,t}(\mathbf{H}_t)_{ii}.
\]
If $d_{i,t}>0$, then the $i$th row of $\mathbf{R}_t$ has unit Euclidean norm, so
\[
1=\mathbf e_i^\top \mathbf{R}_t\mathbf{R}_t^\top \mathbf e_i
=\mathbf e_i^\top \mathbf{H}_t^2 \mathbf e_i
\le \|\mathbf{H}_t\|_2\mathbf e_i^\top \mathbf{H}_t \mathbf e_i
=\|\mathbf{R}_t\|_2(\mathbf{H}_t)_{ii}.
\]
Thus $(\mathbf{H}_t)_{ii}\ge \|\mathbf{R}_t\|_2^{-1}$ whenever $d_{i,t}>0$. If $d_{i,t}=0$, the corresponding term vanishes. Therefore,
\[
\langle \mathbf{M}_t,\mathbf{O}_t\rangle
\ge \frac1{\|\mathbf{R}_t\|_2}\sum_{i=1}^m d_{i,t}.
\]
Now each nonzero row of $\mathbf{R}_t$ has norm $1$, so $\|\mathbf{R}_t\|_F^2\le m$, hence $\|\mathbf{R}_t\|_2\le \sqrt m$. Also, $\sum_{i=1}^m d_{i,t}\ge \|\mathbf{M}_t\|_F$. It follows that
\[
\langle \mathbf{M}_t,\mathbf{O}_t\rangle\ge \frac1{\sqrt m}\|\mathbf{M}_t\|_F.
\]
Finally, if $r_t:=\operatorname{rank}(\mathbf{R}_t)$, then the nonzero singular values of $\mathbf{O}_t$ are all equal to $1$, so $\|\mathbf{O}_t\|_F^2=r_t\le n$.
\end{proof}

%%%%%

\subsection{Lemma \ref{lemma:ns5_traj_error}}
\begin{lemma}
\label{lemma:ns5_traj_error}
Let $\{\hat{\mathbf M}_t\}_{t=1}^T\subset\mathbb R^{m\times n}$ be any finite
algorithmic trajectory, and set $\operatorname{Orth}(\mathbf 0):=\mathbf 0$.
Consider the five-step pre-scaled Newton--Schulz map with the degree-two
Taylor coefficients analyzed in \citep[Definition 2]{kim2026convergence}.
For any $\mathbf A\in\mathbb R^{m\times n}$, define
\[
\alpha(\mathbf A):=\max\{1,\|\mathbf A\|_F\},
\qquad
\mathbf Y^{(0)}(\mathbf A):=\frac{\mathbf A}{\alpha(\mathbf A)},
\]
and, for $k=0,1,2,3,4$,
\[
\mathbf Y^{(k+1)}(\mathbf A)=p_{\mathrm{ns}}\bigl(\mathbf Y^{(k)}(\mathbf A)\mathbf Y^{(k)}(\mathbf A)^\top\bigr)
\mathbf Y^{(k)}(\mathbf A),
\]
where
\[
p_{\mathrm{ns}}(\mathbf Z):=\frac{15}{8}\mathbf I_m
-\frac{5}{4}\mathbf Z+\frac{3}{8}\mathbf Z^2.
\]
We define
\[
\mathrm{NS5}(\mathbf A):=\mathbf Y^{(5)}(\mathbf A).
\]
For each $t$, let $\Pi_t$ be the orthogonal projector onto
$\operatorname{range}(\hat{\mathbf M}_t)$, and define
\[
\delta_{t,0}:=\bigl\|\Pi_t-\mathbf Y^{(0)}(\hat{\mathbf M}_t)
\mathbf Y^{(0)}(\hat{\mathbf M}_t)^\top \bigr\|_2,
\qquad
\delta_{0,T}:=\max_{1\le t\le T}\delta_{t,0}.
\]
Then $\delta_{0,T}<1$. Moreover, the trajectory-wise NS5 polar-approximation
error
\[
\varepsilon_{\mathrm{ns},T}:=\max_{1\le t\le T} \bigl\|\mathrm{NS5}(\hat{\mathbf M}_t)-\operatorname{Orth}(\hat{\mathbf M}_t)\bigr\|_2<1.
\]
\end{lemma}

\begin{proof}
If $\hat{\mathbf M}_t=\mathbf 0$, then $\Pi_t=\mathbf 0$ and $\delta_{t,0}=0$.
Otherwise, the pre-scaling by
$\alpha(\hat{\mathbf M}_t)=\max\{1,\|\hat{\mathbf M}_t\|_F\}$ gives
\[
0<\lambda_{\min}^{+}\bigl(\mathbf Y^{(0)}(\hat{\mathbf M}_t)
\mathbf Y^{(0)}(\hat{\mathbf M}_t)^\top\bigr)\le 1
\]
on $\operatorname{range}(\hat{\mathbf M}_t)$, and hence $\delta_{t,0}<1$.
Since the horizon $T$ is finite, $\delta_{0,T}=\max_{1\le t\le T}\delta_{t,0}<1$.

The polynomial above is the degree-two Taylor Newton--Schulz polynomial in
\citep[Definition 2]{kim2026convergence}. Applying
\citep[Theorem 2]{kim2026convergence} with $\kappa=2$ and $q=5$ yields
\[
\varepsilon_{\mathrm{ns},T}\le1-\sqrt{1-\delta_{0,T}^{(\kappa+1)^q}}=1-\sqrt{1-\delta_{0,T}^{3^5}}.
\]
Finally, since $1-\sqrt{1-x}\le x$ for $x\in[0,1]$ and $\delta_{0,T}<1$, we obtain
\[
\varepsilon_{\mathrm{ns},T}\le\delta_{0,T}^{3^5}<1.
\]
\end{proof}

\subsection{Lemma \ref{lemma:ns5_vs_orth}}
\begin{lemma}
\label{lemma:ns5_vs_orth}
Let $\mathbf A\in\mathbb R^{m\times n}$. If $\mathbf A=\mathbf 0$, define
$\operatorname{Orth}(\mathbf A)=\mathbf 0$ and $\mathrm{NS5}(\mathbf A)=\mathbf 0$.
Otherwise, let $\mathbf A=\mathbf U\boldsymbol\Sigma\mathbf V^\top$ be the compact SVD, and set
$\mathbf Q:=\operatorname{Orth}(\mathbf A)=\mathbf U\mathbf V^\top$ and
$\mathbf O:=\mathrm{NS5}(\mathbf A)$.
Then there exists a diagonal matrix $\widetilde{\boldsymbol\Sigma}$ such that $\mathbf O=\mathbf U\widetilde{\boldsymbol\Sigma}\mathbf V^\top.$ Moreover, for every row-normalized input $\mathbf A=\hat{\mathbf M}_t$ along the algorithmic
trajectory,
\[
\|\mathrm{NS5}(\mathbf A)-\operatorname{Orth}(\mathbf A)\|_2 \le \varepsilon_{\mathrm{ns}}.
\]
Consequently, every diagonal entry $\widetilde\sigma_i$ of $\widetilde{\boldsymbol\Sigma}$ satisfies
\[
1-\varepsilon_{\mathrm{ns}} \le \widetilde\sigma_i
\le 1+\varepsilon_{\mathrm{ns}},
\qquad
i\in[\operatorname{rank}(\mathbf A)].
\]
In particular, whenever $\mathbf A=\hat{\mathbf M}_t\neq \mathbf 0$, we have $\operatorname{Orth}(\mathrm{NS5}(\mathbf A))=\operatorname{Orth}(\mathbf A).$
\end{lemma}

\begin{proof}
If $\mathbf A=\mathbf 0$, the conclusion is immediate by definition. Assume $\mathbf A\neq \mathbf 0$.

The internal scaling multiplies $\mathbf A$ by a positive scalar and hence does not change its left or
right singular vectors. Thus
\[
\mathbf Y^{(0)}(\mathbf A)=\mathbf U\boldsymbol\Sigma^{(0)}\mathbf V^\top
\]
for some diagonal matrix $\boldsymbol\Sigma^{(0)}$.
Assume inductively that
\[
\mathbf Y^{(k)}(\mathbf A)=\mathbf U\boldsymbol\Sigma^{(k)}\mathbf V^\top
\]
with diagonal $\boldsymbol\Sigma^{(k)}$. Then
\[
\mathbf Y^{(k)}(\mathbf A)\mathbf Y^{(k)}(\mathbf A)^\top=\mathbf U\bigl(\boldsymbol\Sigma^{(k)}\bigr)^2\mathbf U^\top,
\]
and therefore
\[
\mathbf Y^{(k+1)}(\mathbf A)=\mathbf U
p_{\mathrm{ns}}\bigl((\boldsymbol\Sigma^{(k)})^2\bigr)
\boldsymbol\Sigma^{(k)}
\mathbf V^\top.
\]
Hence every iterate has the same left and right singular vectors as $\mathbf A$, and in particular
\[
\mathrm{NS5}(\mathbf A)=\mathbf U\widetilde{\boldsymbol\Sigma}\mathbf V^\top
\]
for some diagonal matrix $\widetilde{\boldsymbol\Sigma}$.

Now, if $\mathbf A=\hat{\mathbf M}_t$ is a row-normalized input on the algorithmic trajectory, then
by the definition of $\varepsilon_{\mathrm{ns}}$,
\[
\|\mathrm{NS5}(\mathbf A)-\operatorname{Orth}(\mathbf A)\|_2
=\|\mathbf U(\widetilde{\boldsymbol\Sigma}-\mathbf I)\mathbf V^\top\|_2 =\|\widetilde{\boldsymbol\Sigma}-\mathbf I\|_2
\le \varepsilon_{\mathrm{ns}}.
\]
Hence
\[
\max_i|\widetilde\sigma_i-1| \le \varepsilon_{\mathrm{ns}},
\]
which is equivalent to
\[
1-\varepsilon_{\mathrm{ns}}
\le \widetilde\sigma_i \le 1+\varepsilon_{\mathrm{ns}}
\qquad
\text{for all }i.
\]
Since $\varepsilon_{\mathrm{ns}}<1$, these diagonal entries are strictly positive. Therefore
\[
\operatorname{Orth}(\mathrm{NS5}(\mathbf A))=\operatorname{Orth}(\mathbf A)
\]
whenever $\mathbf A=\hat{\mathbf M}_t\neq \mathbf 0$.
\end{proof}

\subsection{Lemma \ref{conv_r_lemma:a1_ns}}
\begin{lemma}
\label{conv_r_lemma:a1_ns}
Under Lemma~\ref{lemma:ns5_vs_orth}, for all $t$,
\[
\langle \mathbf M_t,\mathbf O_t\rangle
\ge \frac{1-\varepsilon_{\mathrm{ns}}}{\sqrt m}\|\mathbf M_t\|_F, \qquad \|\mathbf O_t\|_F^2 \le (1+\varepsilon_{\mathrm{ns}})^2 n.
\]
\end{lemma}

\begin{proof}
If $\mathbf M_t=\mathbf 0$, then $\hat{\mathbf M}_t=\mathbf O_t=\mathbf 0$ and the claim is immediate.
Assume $\mathbf M_t\neq \mathbf 0$. We have
\[
d_{i,t}:=\|(\mathbf M_t)_{i,:}\|_2,
\qquad
\mathbf D_t:=\mathbf D_{r,t}^{1/2} = \operatorname{diag}(d_{1,t},\dots,d_{m,t}),
\]
so that
\[
\mathbf M_t=\mathbf D_t\hat{\mathbf M}_t.
\]
Let
\[
\hat{\mathbf M}_t=\mathbf U_t\boldsymbol\Sigma_t\mathbf V_t^\top,
\qquad
\mathbf H_t:=(\hat{\mathbf M}_t\hat{\mathbf M}_t^\top)^{1/2} = \mathbf U_t\boldsymbol\Sigma_t\mathbf U_t^\top.
\]
By Lemma~\ref{lemma:ns5_vs_orth},
\[
\mathbf O_t=\mathbf U_t\widetilde{\boldsymbol\Sigma}_t\mathbf V_t^\top
\]
with diagonal entries of $\widetilde{\boldsymbol\Sigma}_t$ lying in
$[1-\varepsilon_{\mathrm{ns}},1+\varepsilon_{\mathrm{ns}}]$.
Therefore,
\[
\mathbf O_t\hat{\mathbf M}_t^\top = \mathbf U_t\widetilde{\boldsymbol\Sigma}_t\boldsymbol\Sigma_t\mathbf U_t^\top \succeq (1-\varepsilon_{\mathrm{ns}})\mathbf U_t\boldsymbol\Sigma_t\mathbf U_t^\top = (1-\varepsilon_{\mathrm{ns}})\mathbf H_t.
\]
Hence
\[
\langle \mathbf M_t,\mathbf O_t\rangle =
\operatorname{tr}(\mathbf O_t^\top\mathbf D_t\hat{\mathbf M}_t)
= \operatorname{tr}(\mathbf D_t\mathbf O_t\hat{\mathbf M}_t^\top) \ge (1-\varepsilon_{\mathrm{ns}})\operatorname{tr}(\mathbf D_t\mathbf H_t) = (1-\varepsilon_{\mathrm{ns}})\sum_{i=1}^m d_{i,t}(\mathbf H_t)_{ii}.
\]
If $d_{i,t}>0$, then the $i$-th row of $\hat{\mathbf M}_t$ has unit Euclidean norm, so
\[
1 = \mathbf e_i^\top \hat{\mathbf M}_t\hat{\mathbf M}_t^\top \mathbf e_i = \mathbf e_i^\top \mathbf H_t^2 \mathbf e_i \le \|\mathbf H_t\|_2\mathbf e_i^\top \mathbf H_t \mathbf e_i = \|\hat{\mathbf M}_t\|_2 (\mathbf H_t)_{ii}.
\]
Thus $(\mathbf H_t)_{ii}\ge \|\hat{\mathbf M}_t\|_2^{-1}$ whenever $d_{i,t}>0$.
If $d_{i,t}=0$, the corresponding term vanishes.
Consequently,
\[
\operatorname{tr}(\mathbf D_t\mathbf H_t) \ge \|\hat{\mathbf M}_t\|_2^{-1}\sum_{i=1}^m d_{i,t}.
\]
Now each nonzero row of $\hat{\mathbf M}_t$ has norm $1$, so
$\|\hat{\mathbf M}_t\|_F^2\le m$, hence $\|\hat{\mathbf M}_t\|_2\le \sqrt m$.
Also,
\[
\sum_{i=1}^m d_{i,t} \ge \Bigl(\sum_{i=1}^m d_{i,t}^2\Bigr)^{1/2} = \|\mathbf M_t\|_F.
\]
Combining the above bounds gives
\[
\langle \mathbf M_t,\mathbf O_t\rangle \ge \frac{1-\varepsilon_{\mathrm{ns}}}{\sqrt m}\|\mathbf M_t\|_F.
\]
For the norm bound, let $r_t:=\operatorname{rank}(\hat{\mathbf M}_t)$.
Again by Lemma~\ref{lemma:ns5_vs_orth}, every nonzero singular value of $\mathbf O_t$
is at most $1+\varepsilon_{\mathrm{ns}}$. Therefore
\[
\|\mathbf O_t\|_F^2 = \sum_{i=1}^{r_t}\widetilde\sigma_{i,t}^2
\le (1+\varepsilon_{\mathrm{ns}})^2 r_t \le (1+\varepsilon_{\mathrm{ns}})^2 n.
\]
This completes the proof.
\end{proof}

\subsection{Lemma \ref{conv_r_lemma:wd_envelope}}

\begin{lemma}
\label{conv_r_lemma:wd_envelope}
Consider the row-normalized update with decoupled weight decay
\[
\mathbf X_{t+1} = \mathbf X_t-\eta_t(a\mathbf O_t+\lambda_t\mathbf X_t),
\]
where $\mathbf O_t=\mathrm{NS5}(\hat{\mathbf M}_t)$ and Lemma~\ref{conv_r_lemma:a1_ns} holds.
Let $\gamma:=1-\varepsilon_{\mathrm{ns}},
0\le \rho<\frac{a\gamma}{\sqrt m},$
and assume
\begin{equation}
\label{eq:row-ns-lambda-schedule}
0\le \lambda_t \le \frac{\rho}{\|\mathbf X_1\|_F+4a(1+\varepsilon_{\mathrm{ns}})\sqrt n}t^{-1/4}.
\end{equation}
Then, along every realization,
\[
\lambda_t\|\mathbf X_t\|_F\le \rho,
\qquad
\|\mathbf X_{t+1}-\mathbf X_t\|_F
\le (a(1+\varepsilon_{\mathrm{ns}})\sqrt n+\rho)\cdot\eta_t.
\]
\end{lemma}

\begin{proof}
Let $
\lambda_0:=
\frac{\rho}{\|\mathbf X_1\|_F+4a(1+\varepsilon_{\mathrm{ns}})\sqrt n}.$ Since $\rho<a\gamma_{\mathrm{ns}}/\sqrt m\le a/\sqrt m$, we have
\[
\lambda_0 \le \frac{1}{4(1+\varepsilon_{\mathrm{ns}})\sqrt{mn}} \le 1.
\]
Hence
\[
0\le \eta_t\lambda_t\le \lambda_0t^{-1}\le 1.
\]
Using Lemma~\ref{conv_r_lemma:a1_ns},
\[
\|\mathbf O_t\|_F\le (1+\varepsilon_{\mathrm{ns}})\sqrt n.
\]
Therefore,
\[
\begin{aligned}
\|\mathbf X_{t+1}\|_F
&= \|(1-\eta_t\lambda_t)\mathbf X_t-a\eta_t\mathbf O_t\|_F  \\
&\le (1-\eta_t\lambda_t)\|\mathbf X_t\|_F
+a\eta_t\|\mathbf O_t\|_F  \\
&\le \|\mathbf X_t\|_F +a(1+\varepsilon_{\mathrm{ns}})\sqrt n\eta_t .
\end{aligned}
\]
By induction and $\sum_{s=1}^{t-1}s^{-3/4}\le 4t^{1/4}$,
\[
\|\mathbf X_t\|_F \le \|\mathbf X_1\|_F +4a(1+\varepsilon_{\mathrm{ns}})\sqrt nt^{1/4}.
\]
Thus
\[
\lambda_t\|\mathbf X_t\|_F \le
\lambda_0t^{-1/4}
\Bigl( \|\mathbf X_1\|_F +4a(1+\varepsilon_{\mathrm{ns}})\sqrt nt^{1/4} \Bigr)
\le \rho.
\]
Finally,
\[
\begin{aligned}
\|\mathbf X_{t+1}-\mathbf X_t\|_F
&= \eta_t\|a\mathbf O_t+\lambda_t\mathbf X_t\|_F  \\
&\le \eta_t\Bigl(a(1+\varepsilon_{\mathrm{ns}})\sqrt n+\rho\Bigr)
\end{aligned}
\]
This completes the proof.
\end{proof}

\section{Proofs of Theorem \ref{th_conv_r} and Corollary \ref{cor:conv_r_ns}}
\label{proof:th_conv_r}

\begin{proof}
Let $\mathbf S_t:=\nabla f(\mathbf X_t)-\mathbf M_t,
\gamma:=1-\varepsilon_{\mathrm{ns}},$
and define $c_{m,n}^{\mathrm{ns}} := \frac{\gamma}{\sqrt m}  +(1+\varepsilon_{\mathrm{ns}})\sqrt n.$

By \eqref{eq:row-ns-lambda-schedule}, we have $\frac{a\gamma}{\sqrt m}-\rho>0$.
By Assumption~\ref{ass:2} and the descent lemma,
\begin{equation}
\label{eq:row-ns-wd-descent}
\begin{aligned}
f(\mathbf X_{t+1})
&\le f(\mathbf X_t) + \langle \nabla f(\mathbf X_t),\mathbf X_{t+1}-\mathbf X_t\rangle + \frac{L}{2}\|\mathbf X_{t+1}-\mathbf X_t\|_F^2
\\
&= f(\mathbf X_t) -a\eta_t\langle \nabla f(\mathbf X_t),\mathbf O_t\rangle -\lambda_t\eta_t\langle \nabla f(\mathbf X_t),\mathbf X_t\rangle + \frac{L\eta_t^2}{2} \|a\mathbf O_t+\lambda_t\mathbf X_t\|_F^2
\\
&\le f(\mathbf X_t) -a\eta_t\langle \nabla f(\mathbf X_t),\mathbf O_t\rangle +\rho\eta_t\|\nabla f(\mathbf X_t)\|_F
+ \frac{L}{2}\bigl(a(1+\varepsilon_{\mathrm{ns}})\sqrt n+\rho\bigr)^2\eta_t^2,
\end{aligned}
\end{equation}
where the last step uses Lemma~\ref{conv_r_lemma:wd_envelope} and Cauchy's inequality.

Again by Lemma~\ref{conv_r_lemma:a1_ns}, we have
\begin{equation}
\label{eq:row-ns-wd-align}
\begin{aligned}
\langle \nabla f(\mathbf X_t),\mathbf O_t\rangle
&= \langle \mathbf M_t,\mathbf O_t\rangle
+ \langle \mathbf S_t,\mathbf O_t\rangle
\\
&\ge \frac{\gamma}{\sqrt m}\|\mathbf M_t\|_F - \|\mathbf S_t\|_F\|\mathbf O_t\|_F
\\
&\ge \frac{\gamma}{\sqrt m}\|\mathbf M_t\|_F - (1+\varepsilon_{\mathrm{ns}})\sqrt n\|\mathbf S_t\|_F
\\
&\ge \frac{\gamma}{\sqrt m}\|\nabla f(\mathbf X_t)\|_F - \Bigl(
\frac{\gamma}{\sqrt m} + (1+\varepsilon_{\mathrm{ns}})\sqrt n \Bigr)
\|\mathbf S_t\|_F.
\end{aligned}
\end{equation}
Substituting \eqref{eq:row-ns-wd-align} into \eqref{eq:row-ns-wd-descent}, we have
\begin{equation}
\label{eq:row-ns-wd-mainstep}
f(\mathbf X_{t+1})
\le f(\mathbf X_t) - \left(\frac{a\gamma}{\sqrt m}-\rho\right)\eta_t\|\nabla f(\mathbf X_t)\|_F
+ ac_{m,n}^{\mathrm{ns}}\eta_t\|\mathbf S_t\|_F
+ \frac{L}{2}\bigl(a(1+\varepsilon_{\mathrm{ns}})\sqrt n+\rho\bigr)^2\eta_t^2 .
\end{equation}
Apply Young's inequality with parameter $h_{t+1}/L$:
\[
ac_{m,n}^{\mathrm{ns}}\eta_t\|\mathbf S_t\|_F
\le \frac{ah_{t+1}}{2L}\|\mathbf S_t\|_F^2 + \frac{aL}{2}\bigl(c_{m,n}^{\mathrm{ns}}\bigr)^2
\frac{\eta_t^2}{h_{t+1}}.
\]
Let $d_{\rho}^{\mathrm{ns}} :=
a(1+\varepsilon_{\mathrm{ns}})\sqrt n+\rho$ and $ \kappa_{\rho}^{\mathrm{ns}} := \frac{a\gamma}{\sqrt m}-\rho.$ Hence
\[
\begin{aligned}
f(\mathbf X_{t+1})
\le& f(\mathbf X_t) - \kappa_{\rho}^{\mathrm{ns}}\eta_t\|\nabla f(\mathbf X_t)\|_F + \frac{ah_{t+1}}{2L}\|\mathbf S_t\|_F^2 \\
&\quad + \frac{aL}{2}\bigl(c_{m,n}^{\mathrm{ns}}\bigr)^2
\frac{\eta_t^2}{h_{t+1}} + \frac{L}{2}\bigl(d_{\rho}^{\mathrm{ns}}\bigr)^2\eta_t^2.
\end{aligned}
\]
Taking expectations and summing from $t=1$ to $T$, then using
$f(\mathbf X_{T+1})\ge f^\star$, we obtain
\begin{equation}
\label{eq:row-ns-wd-main}
\begin{aligned}
\kappa_{\rho}^{\mathrm{ns}}
\sum_{t=1}^T
\eta_t\mathbb E\|\nabla f(\mathbf X_t)\|_F
&\le f(\mathbf X_1)-f^\star + \frac{a}{2L}
\sum_{t=1}^T h_{t+1}\mathbb E\|\mathbf S_t\|_F^2
\\
&\quad
+ \frac{aL}{2}\bigl(c_{m,n}^{\mathrm{ns}}\bigr)^2
\sum_{t=1}^T \frac{\eta_t^2}{h_{t+1}} + \frac{L}{2}\bigl(d_{\rho}^{\mathrm{ns}}\bigr)^2
\sum_{t=1}^T
\eta_t^2 .
\end{aligned}
\end{equation}
By Lemma~\ref{conv_lemma:a3}, using
$\|\mathbf X_{t+1}-\mathbf X_t\|_F\le d_{\rho}^{\mathrm{ns}}\eta_t$
from Lemma~\ref{conv_r_lemma:wd_envelope}, we have
\[
\mathbb E\|\mathbf S_{t+1}\|_F^2
\le (1-h_{t+1})\mathbb E\|\mathbf S_t\|_F^2 + \frac{(1-h_{t+1})^2}{h_{t+1}} L^2\bigl(d_{\rho}^{\mathrm{ns}}\bigr)^2\eta_t^2 + h_{t+1}^2\sigma^2.
\]
Since $(1-h_{t+1})^2\le 1$ and
\[
\frac{\eta_t^2}{h_{t+1}} = \frac{t^{-3/2}}{(t+1)^{-1/2}}
= \frac{\sqrt{t+1}}{t^{3/2}}
\le \frac{2\sqrt2}{t+1} = 2\sqrt2 h_{t+1}^2,
\]
we have
\[
\mathbb E\|\mathbf S_{t+1}\|_F^2
\le (1-h_{t+1})\mathbb E\|\mathbf S_t\|_F^2
+ h_{t+1}^2 \Bigl(
2\sqrt2L^2\bigl(d_{\rho}^{\mathrm{ns}}\bigr)^2
+\sigma^2\Bigr).
\]
Applying Lemma~\ref{conv_lemma:a2} with
\[
E_t=\mathbb E\|\mathbf S_t\|_F^2,
\qquad
A_{t+1} = h_{t+1}^2
\Bigl( 2\sqrt2L^2\bigl(d_{\rho}^{\mathrm{ns}}\bigr)^2
+\sigma^2\Bigr),
\]
we get
\[
h_t\mathbb E\|\mathbf S_t\|_F^2 \le 2\Bigl(\mathbb E\|\mathbf S_t\|_F^2 - \mathbb E\|\mathbf S_{t+1}\|_F^2
+ A_{t+1} \Bigr).
\]
Moreover, since $\beta_1=0$,
\[
\mathbb E\|\mathbf S_1\|_F^2 = \mathbb E\|\nabla f(\mathbf X_1)-\mathbf M_1\|_F^2 = \mathbb E\|\nabla f(\mathbf X_1)-\nabla f(\mathbf X_1;\xi_1)\|_F^2 \le \sigma^2.
\]
Thus,
\[
\begin{aligned}
\sum_{t=1}^T h_t\mathbb E\|\mathbf S_t\|_F^2
&\le 2\sum_{t=1}^T \Bigl( \mathbb E\|\mathbf S_t\|_F^2
- \mathbb E\|\mathbf S_{t+1}\|_F^2 + A_{t+1} \Bigr) \\
&\le 2\mathbb E\|\mathbf S_1\|_F^2 + 2\Bigl(
2\sqrt2L^2\bigl(d_{\rho}^{\mathrm{ns}}\bigr)^2 +\sigma^2 \Bigr) \sum_{t=1}^T\frac1{t+1} \\
&\le 2\sigma^2 + 2\Bigl( 2\sqrt2L^2\bigl(d_{\rho}^{\mathrm{ns}}\bigr)^2 +\sigma^2 \Bigr)(1+\ln T).
\end{aligned}
\]
Hence, since $h_{t+1}\le h_t$, we have
\[
\sum_{t=1}^T h_{t+1}\mathbb E\|\mathbf S_t\|_F^2
\le 2\sigma^2 + 2\Bigl( 2\sqrt2L^2\bigl(d_{\rho}^{\mathrm{ns}}\bigr)^2 +\sigma^2 \Bigr)(1+\ln T).
\]
Also,
\[
\sum_{t=1}^T\frac{\eta_t^2}{h_{t+1}}
= \sum_{t=1}^T \frac{t^{-3/2}}{(t+1)^{-1/2}} \le 2(1+\ln T), \qquad
\sum_{t=1}^T\eta_t^2
= \sum_{t=1}^Tt^{-3/2} \le 3.
\]
Substituting the above bounds into \eqref{eq:row-ns-wd-main},
\[
\kappa_{\rho}^{\mathrm{ns}}
\sum_{t=1}^T
\eta_t\mathbb E\|\nabla f(\mathbf X_t)\|_F
\le f(\mathbf X_1)-f^\star + C_{1,\rho}^{\mathrm{ns}}(1+\ln T) + C_{2,\rho}^{\mathrm{ns}}.
\]
where 
\[
C_{1,\rho}^{\mathrm{ns}} := \frac{a}{L}\Bigl(2\sqrt2L^2\bigl(d_{\rho}^{\mathrm{ns}}\bigr)^2
+\sigma^2\Bigr) + aL\bigl(c_{m,n}^{\mathrm{ns}}\bigr)^2,C_{2,\rho}^{\mathrm{ns}} := \frac{a\sigma^2}{L} + \frac{3L}{2}\bigl(d_{\rho}^{\mathrm{ns}}\bigr)^2.
\]
Finally, since $\eta_t=t^{-3/4}\ge T^{-3/4}$ for all $1\le t\le T$,
\[
T^{-3/4} \sum_{t=1}^T
\mathbb E\|\nabla f(\mathbf X_t)\|_F
\le \sum_{t=1}^T \eta_t\mathbb E\|\nabla f(\mathbf X_t)\|_F.
\]
Therefore,
\[
\frac1T
\sum_{t=1}^T
\mathbb E\|\nabla f(\mathbf X_t)\|_F
\le \frac{
f(\mathbf X_1)-f^\star + C_{1,\rho}^{\mathrm{ns}}(1+\ln T) + C_{2,\rho}^{\mathrm{ns}}}{\kappa_{\rho}^{\mathrm{ns}}\cdot T^{1/4}}.
\]
Setting $\varepsilon_{\mathrm{ns}}$ to zero yields the corresponding exact-polar result for Theorem~\ref{th_conv_r} with decoupled weight decay. Setting additionally $\rho=0$ recovers the zero weight decay proof.
This completes the proof.
\end{proof}

\section{More Results}
\label{appendix:more_res}
$\blacktriangleright$ \textbf{Learning Rate Search.} Figure \ref{fig:lr_sweep} shows the learning-rate sweeps for LLaMA2-130M and LLaMA2-350M trained on C4 for 2.6B and 7.5B tokens, respectively. 
\begin{figure}[!htbp]
	\centering
	\subfloat[130M]{
    \label{Fig:FG4-1}%
	\begin{minipage}[h]{0.48\textwidth}
	\centering
    \includegraphics[width=\textwidth]{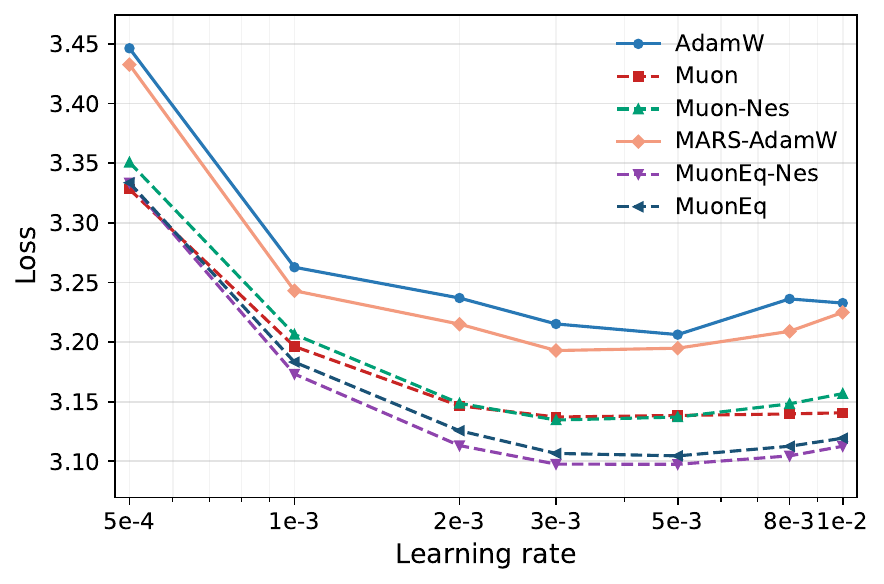} 
    \end{minipage}
    }
    \subfloat[350M]{
    \label{Fig:FG4-2}%
	\begin{minipage}[h]{0.48\textwidth}
	\centering
    \includegraphics[width=\textwidth]{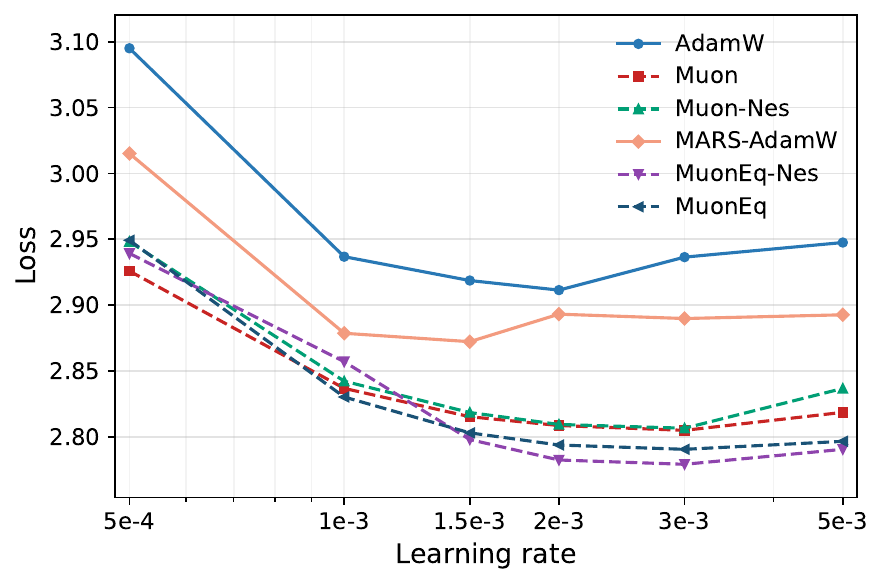} 
    \end{minipage}
    }
    \caption{Learning-rate sweeps for LLaMA2-130M and LLaMA2-350M trained on C4 for 2.6B and 7.5B tokens, respectively.}
\label{fig:lr_sweep}
\end{figure}

$\blacktriangleright$ \textbf{Module-wise Pre-NS Momentum Analysis.}
Since {\method} acts directly on the matrix fed into Newton--Schulz, we analyze pre-NS momentum matrices along a 2.6B-token LLaMA2-130M/C4 training trajectory, tracking the \texttt{query}, \texttt{key}, \texttt{value}, \texttt{output}, \texttt{gate}, \texttt{down}, and \texttt{up} weights in layers 1, 5, 9, and 12 at 1\%, 10\%, 50\%, and 100\% of training. Figure~\ref{fig:spectral_metrics} reports the singular-value entropy and stable rank of these matrices; empirically, \texttt{RowColNorm} consistently improves both metrics over \texttt{direct}, \texttt{ColNorm}, and \texttt{RowNorm}. Figure~\ref{fig:e_b_metrics} provides the empirical counterpart of Eq.~\eqref{eq:error_decomp}: the top row measures the finite-step NS5 approximation error after preprocessing, and the bottom row the bias introduced relative to the unpreconditioned polar direction. At 1\% and 10\% of training, two-sided \texttt{RowColNorm} is usually lowest or near-lowest in the top panels across layers, indicating that early RC equilibration most effectively improves the geometry seen by NS5, but it also yields the largest preconditioning bias. By 50\% and 100\%, the top-panel differences shrink substantially while the extra bias of the two-sided map remains visible. These plots illustrate the same trade-off as Eq.~\eqref{eq:error_decomp}: RC provides stronger two-sided spectral correction but also larger preprocessing bias. For the hidden-weight setting studied here, R remains the more suitable default one-sided variant.

% $\blacktriangleright$ \textbf{From Exact Polar to Practical NS5 for the Default R Variant.} To connect Corollary~\ref{cor:conv_r_ns} with the practical NS5 implementation, Figure~\ref{fig:ns5_spec_norm} reports the training-time evolution of $\varepsilon_{\mathrm{ns}}$ along the {\method} (R) trajectory, averaged over the selected \texttt{RowNorm} modules. Here, $\varepsilon_{\mathrm{ns}}$ is exactly the spectral-norm error term appearing in Corollary~\ref{cor:conv_r_ns}$,$ measuring the discrepancy between the NS5 approximation and the exact polar factor. We observe that $\varepsilon_{\mathrm{ns}}$ remains below $1$ throughout training, so the multiplicative constants $(1\pm\varepsilon_{\mathrm{ns}})$ in the corollary remain well-defined and controlled on the trajectory considered here. This provides empirical support that replacing the exact polar factor by NS5 preserves the practical relevance of the $T^{-1/4}$ stationarity guarantee for the default R variant.

$\blacktriangleright$ \textbf{Sensitivity Analysis.} Figure~\ref{fig:heatmap} analyzes hyperparameter sensitivity on GPT2-small trained on FineWeb for 2.6B tokens by sweeping the learning rate (x-axis) and momentum (y-axis) at $K=4$ and $K=5$. Across both NS budgets, especially at $K=5$, the default R variant of {\method}-Nes shows a broader low-perplexity region than Muon-Nes, indicating reduced sensitivity and stronger robustness to hyperparameter choice.

\begin{figure}[!htbp]
	\centering
	\subfloat[Step 1\%]{
    \label{Fig:FG3-1}%
	\begin{minipage}[h]{0.48\textwidth}
	\centering
    \includegraphics[width=\textwidth]{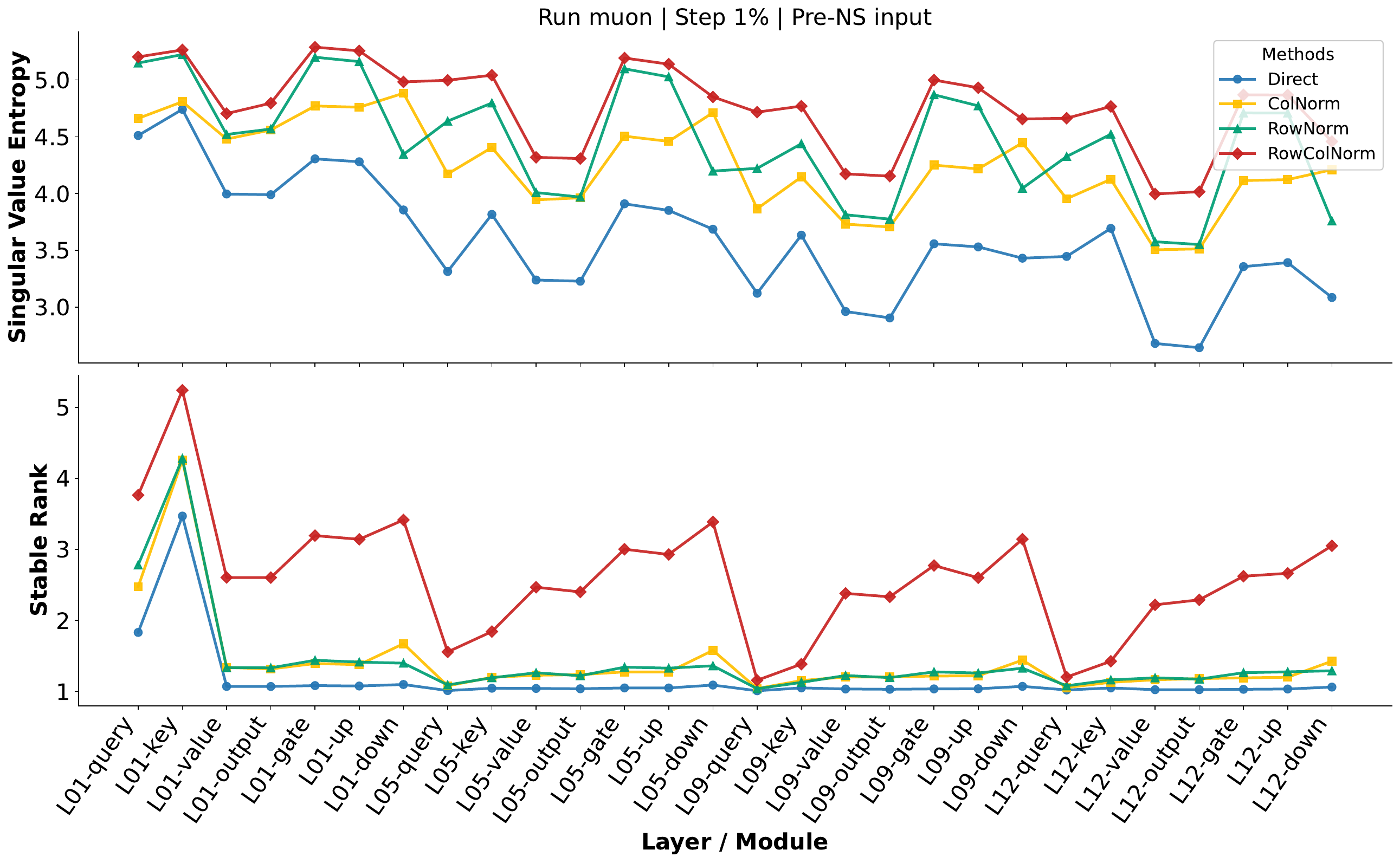} 
    \end{minipage}
    }
    \subfloat[Step 10\%]{
    \label{Fig:FG3-2}%
	\begin{minipage}[h]{0.48\textwidth}
	\centering
    \includegraphics[width=\textwidth]{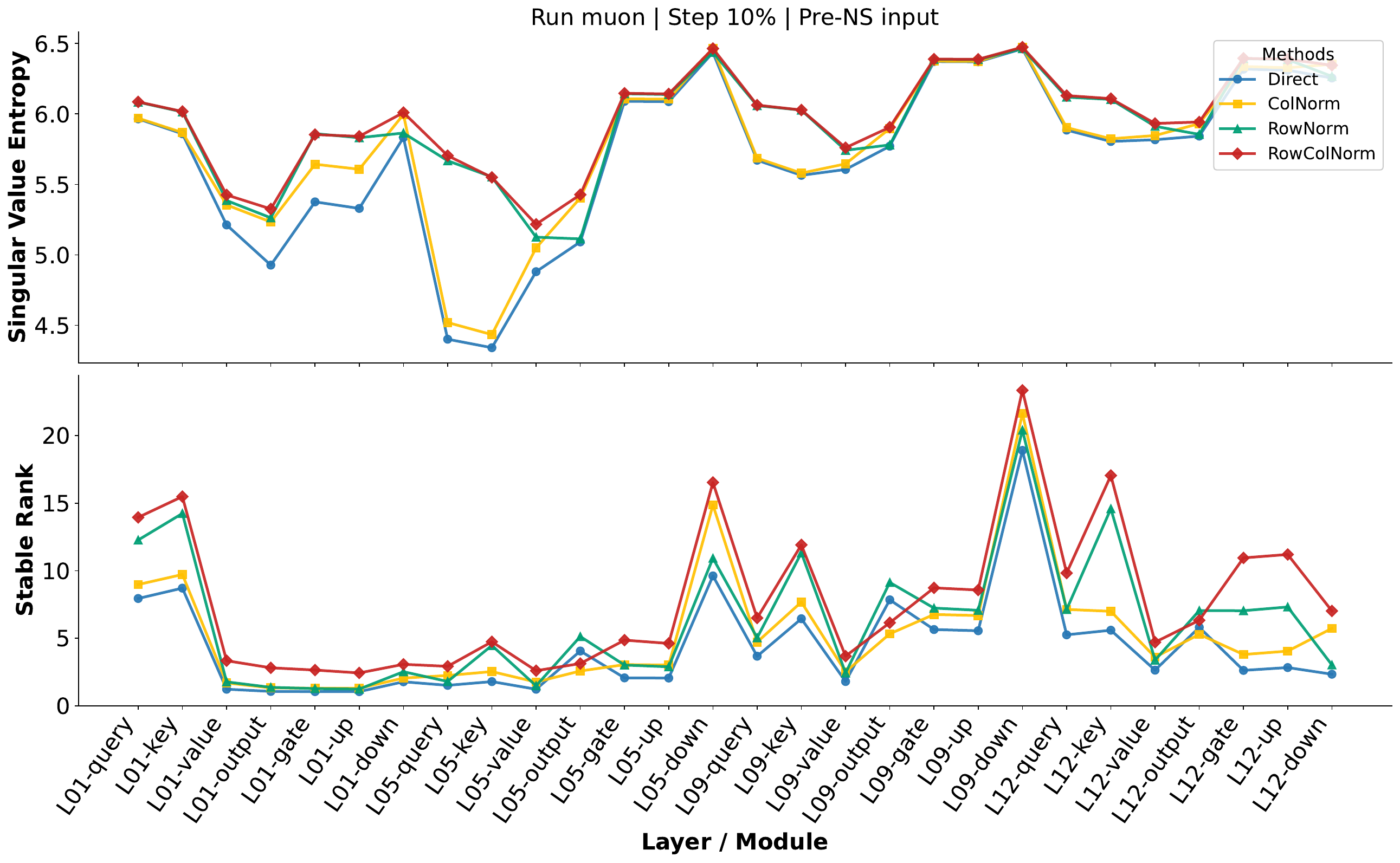} 
    \end{minipage}
    }
    \\
    \subfloat[Step 50\%]{
    \label{Fig:FG3-3}%
	\begin{minipage}[h]{0.48\textwidth}
	\centering
    \includegraphics[width=\textwidth]{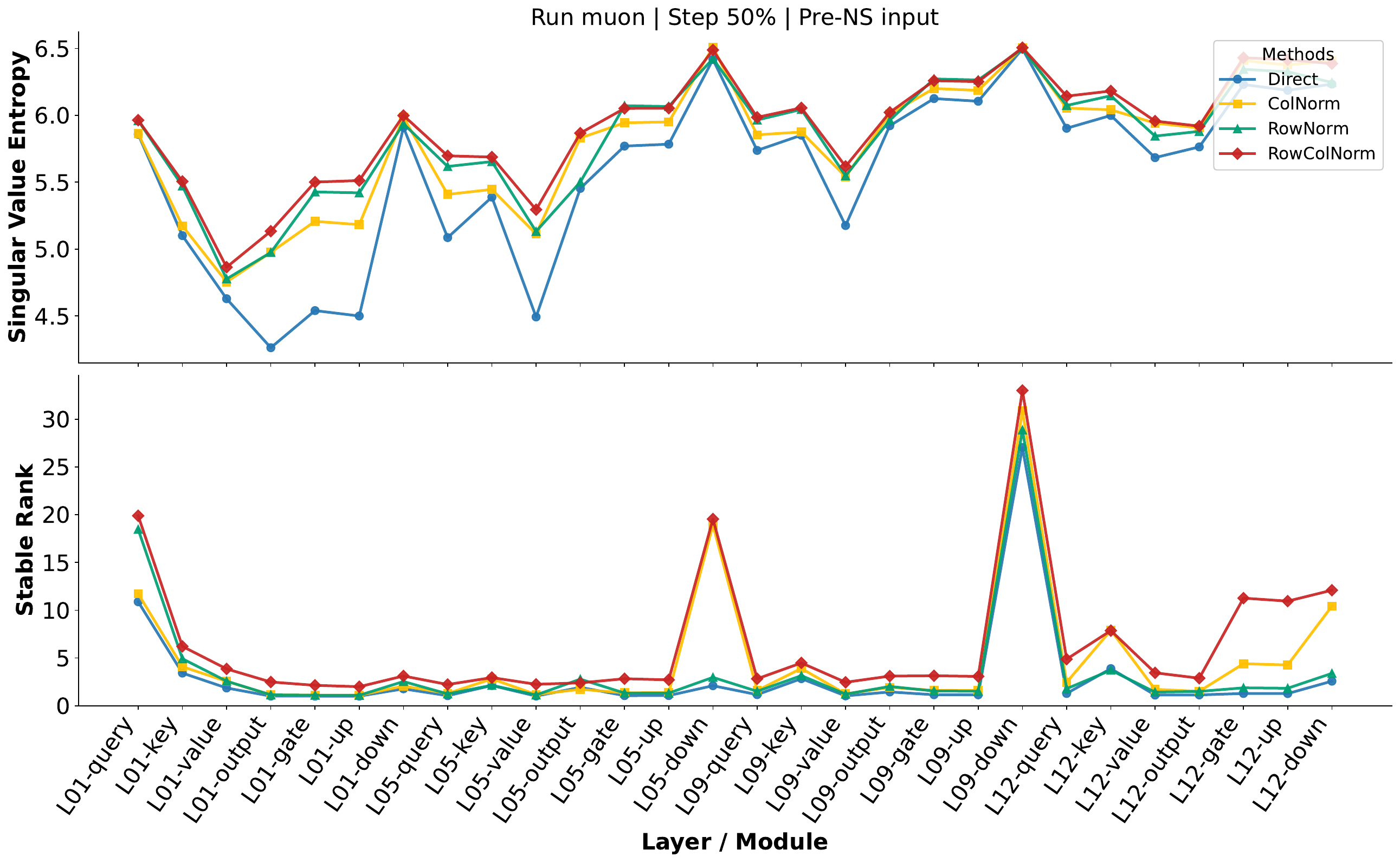} 
    \end{minipage}
    }
    \subfloat[Step 100\%]{
    \label{Fig:FG3-4}%
	\begin{minipage}[h]{0.48\textwidth}
	\centering
    \includegraphics[width=\textwidth]{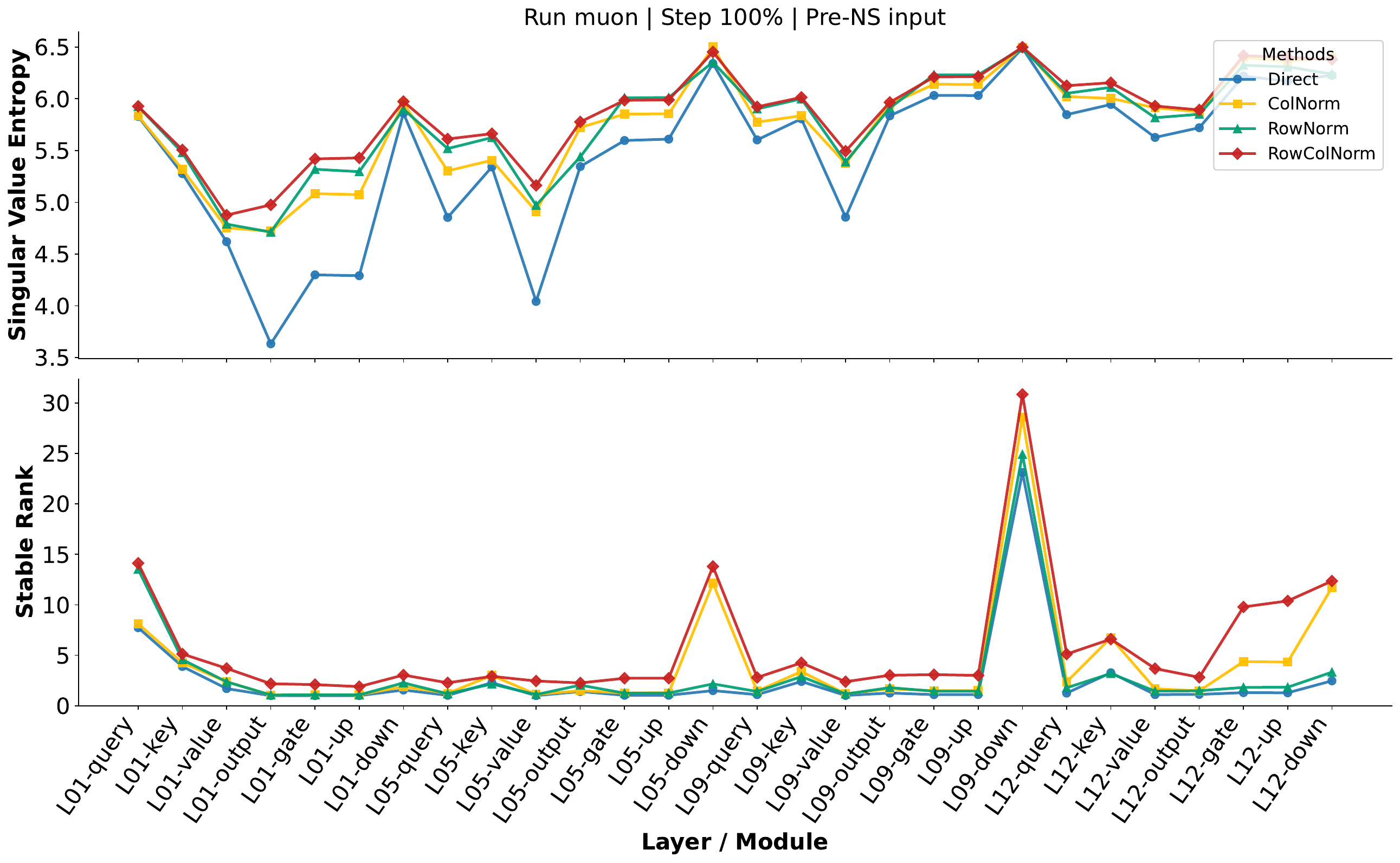} 
    \end{minipage}
    }
    \caption{Singular-value entropy and stable rank of Muon momentum matrices under different normalization schemes during LLaMA2-130M training on C4 for 2.6B tokens. \texttt{RowColNorm} consistently improves both metrics over \texttt{Direct}, \texttt{ColNorm}, and \texttt{RowNorm}.}
\label{fig:spectral_metrics}
\end{figure}

\section{Experimental Details}
\label{appendix:exp}
\subsection{Pretraining on C4}
$\blacktriangleright$ \textbf{Experimental setup.}  We use the same 8-node Ascend 910C cluster for all LLaMA experiments; 130M uses 32 NPUs, whereas 350M and 1B use 64 NPUs.  We compare AdamW, MARS-AdamW, Muon, Muon-Nes, and {\method}/{\method}-Nes on LLaMA2-130M, LLaMA2-350M and LLaMA2-1B~\cite{Touvron2023Llama2O} trained on C4~\cite{raffel2020exploring}. We keep the model architecture, dataset, and training recipe fixed across optimizers. The global batch size is 128 for LLaMA2-130M, 256 for LLaMA2-350M and 512 for LLaMA2-1B, and the maximum sequence length is 4096 for both models. Hyperparameters are selected according to validation performance. 

$\blacktriangleright$ \textbf{Hyperparameter search.} The selected hyperparameters are summarized in Tables \ref{tab:llama130m_c4_hyperparameters} , \ref{tab:llama350m_c4_hyperparameters} and \ref{tab:llama1b_c4_hyperparameters}. In the main comparison, LLaMA2-130M, LLaMA2-350M, and LLaMA2-1B are trained on 10.5B, 21.0B, and 21.0B tokens, respectively. Learning-rate sweeps are conducted on shorter pilot runs of 2.6B tokens for LLaMA2-130M and 7.5B tokens for LLaMA2-350M. AdamW and MARS-AdamW use $(\beta_1,\beta_2)=(0.9,0.95)$, $\epsilon=10^{-8}$, and weight decay 0.1. Muon, Muon-Nes, and {\method}/{\method}-Nes use Muon momentum 0.95 and weight decay 0.1. We set $\gamma=0.025$ for MARS-AdamW and $\gamma=0.05$ for Muon-Nes and {\method}/{\method}-Nes. We set $\varepsilon=10^{-8}$ for {\method}/{\method}-Nes. For all Muon-type optimizers, the number of NS iterations is set to 5 by default. For Muon, Muon-Nes, and {\method}/{\method}-Nes, we use the Muon implementation from \texttt{Moonlight}\footnote{\url{https://github.com/MoonshotAI/Moonlight}}. All 2D weight matrices except embeddings are optimized with the corresponding Muon-type optimizer, while 1D parameters and embeddings are optimized with AdamW.

$\blacktriangleright$ \textbf{Clarification of $\gamma$.}
The symbol $\gamma$ in the hyperparameter tables denotes an
algorithm-dependent correction coefficient. For MARS-AdamW, $\gamma$ is the
variance-reduction correction coefficient used by MARS. For Muon-Nes and
MuonEq-Nes, $\gamma$ denotes the correction coefficient in the following
one-batch Nesterov/MVR parameterization~\cite{yuan2024mars,chang2025convergence}. Let
$\mathbf G_t:=\nabla f(\mathbf X_t;\xi_t)$, and let
$\mathbf M_t^{\rm Nes}$ denote the momentum matrix passed to the diagonal
preconditioning and Newton--Schulz steps. In the constant-momentum setting used
in our experiments, the Nesterov buffer can be written as
\[
\mathbf M_t^{\rm Nes}
= \beta \mathbf M_{t-1}^{\rm Nes}
+ (1-\beta)\mathbf G_t
+ \gamma\beta\bigl(\mathbf G_t-\mathbf G_{t-1}\bigr),
\qquad
\mathbf G_{t-1}:=\nabla f(\mathbf X_{t-1};\xi_{t-1}).
\]
The {\method}-Nes convention used in Algorithm~\ref{alg:muoneq} corresponds to $\gamma = 1-\beta$.
Indeed, if the EMA buffer is
$\mathbf M_t^{\rm EMA}
=\beta \mathbf M_{t-1}^{\rm EMA}+(1-\beta)\mathbf G_t$
and the Nesterov input is
$\widetilde{\mathbf M}_t
=\beta \mathbf M_t^{\rm EMA}+(1-\beta)\mathbf G_t$,
then $\widetilde{\mathbf M}_t$ satisfies the recurrence above with
$\gamma=1-\beta$. Thus the reported Muon momentum $\beta=0.95$ corresponds to
$\gamma=0.05$; the actual coefficient multiplying
$\mathbf G_t-\mathbf G_{t-1}$ is $\gamma\beta=0.0475$.

\begin{figure}[!htbp]
	\centering
	\subfloat[Step 1\%]{
    \label{Fig:FG3-5}%
	\begin{minipage}[h]{0.48\textwidth}
	\centering
    \includegraphics[width=\textwidth]{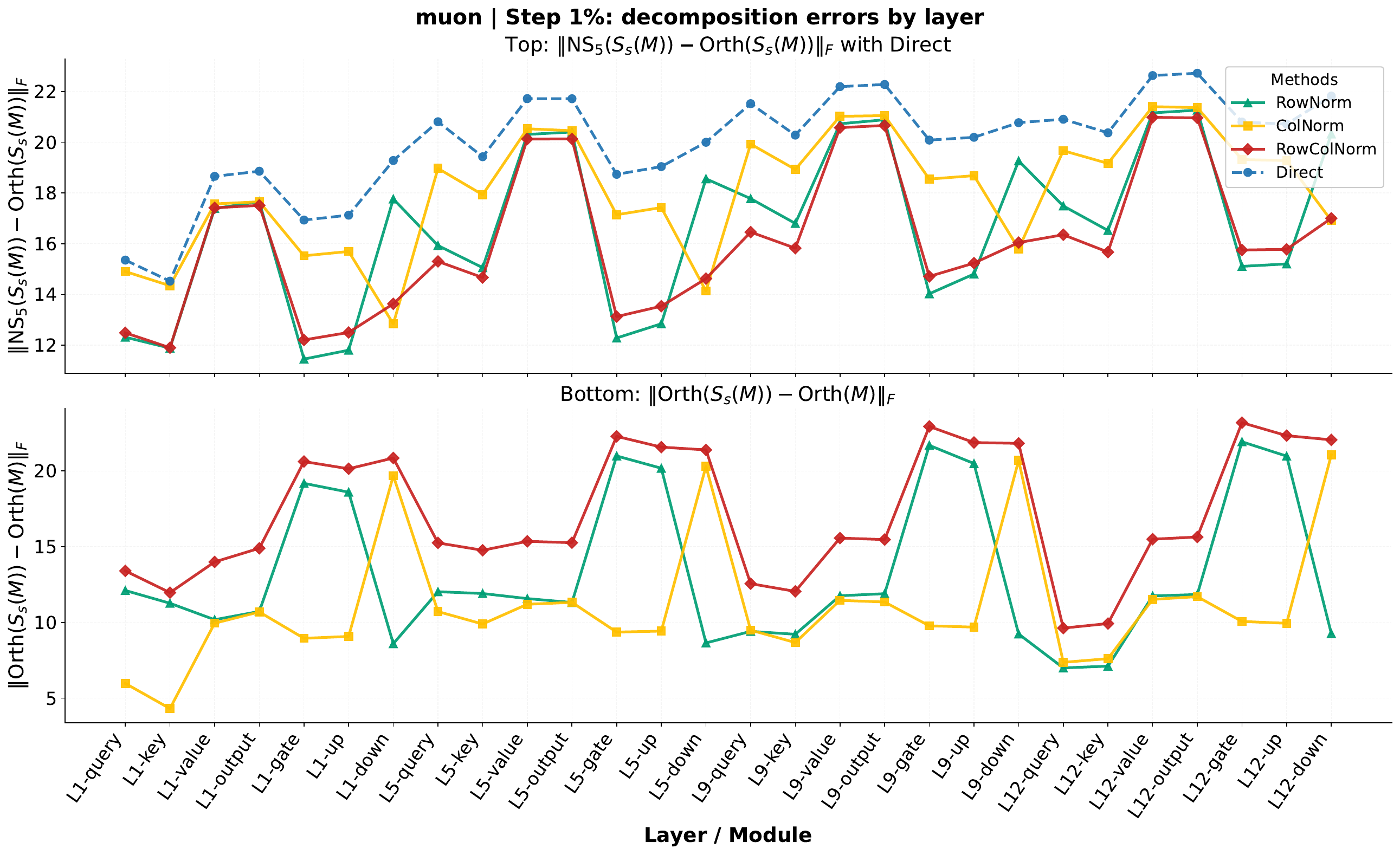} 
    \end{minipage}
    }
    \subfloat[Step 10\%]{
    \label{Fig:FG3-6}%
	\begin{minipage}[h]{0.48\textwidth}
	\centering
    \includegraphics[width=\textwidth]{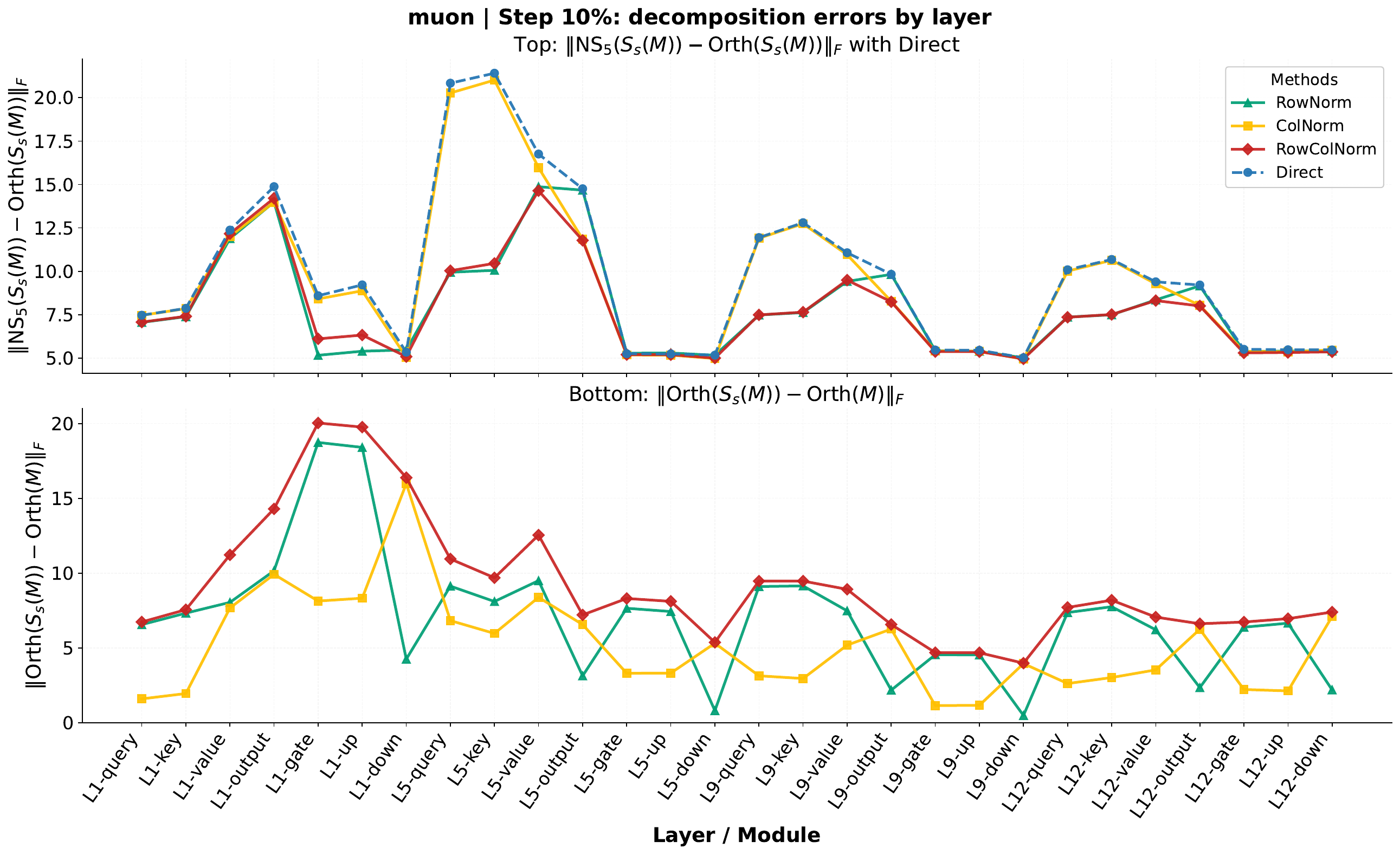} 
    \end{minipage}
    }
    \\
    \subfloat[Step 50\%]{
    \label{Fig:FG3-7}%
	\begin{minipage}[h]{0.48\textwidth}
	\centering
    \includegraphics[width=\textwidth]{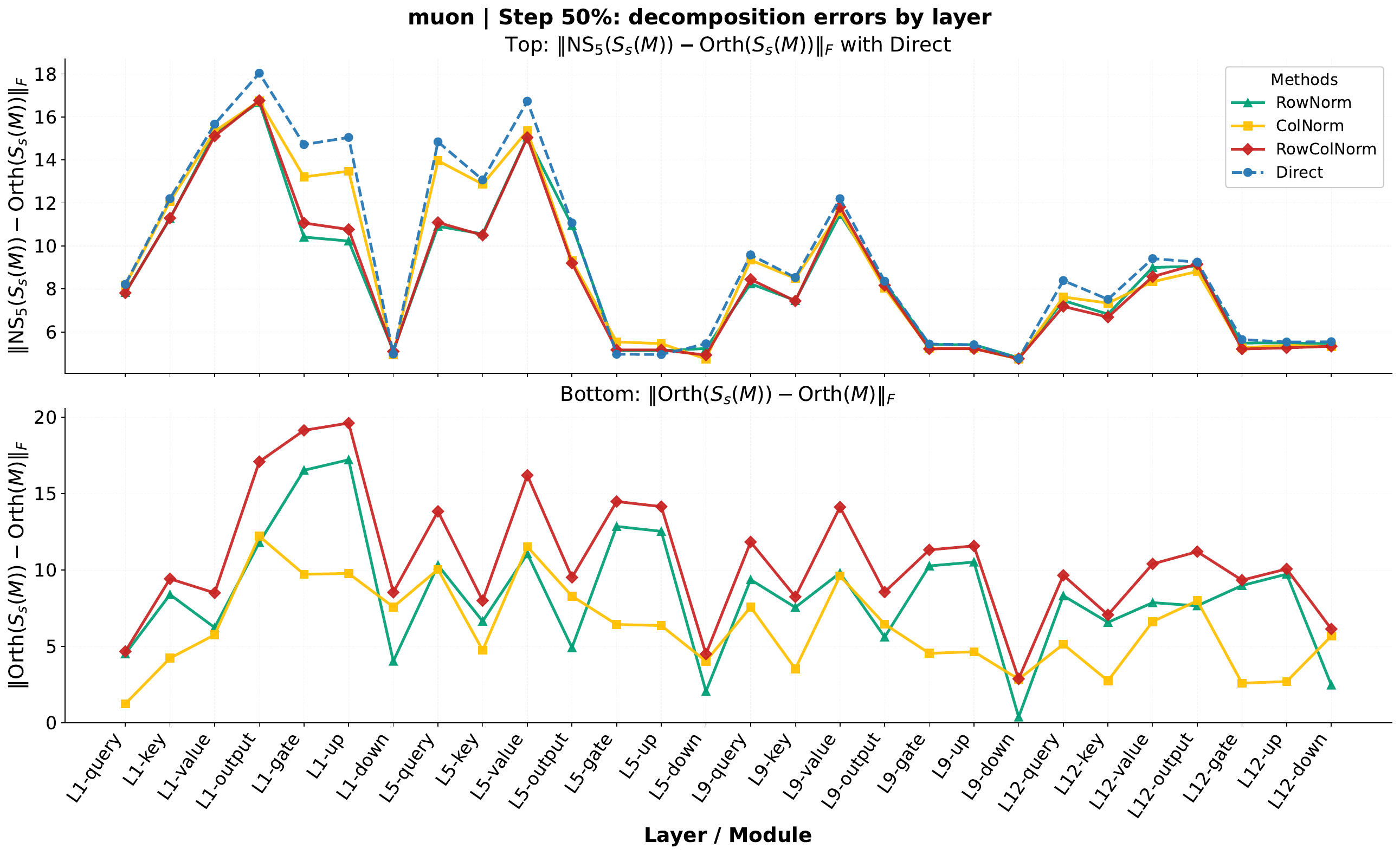} 
    \end{minipage}
    }
    \subfloat[Step 100\%]{
    \label{Fig:FG3-8}%
	\begin{minipage}[h]{0.48\textwidth}
	\centering
    \includegraphics[width=\textwidth]{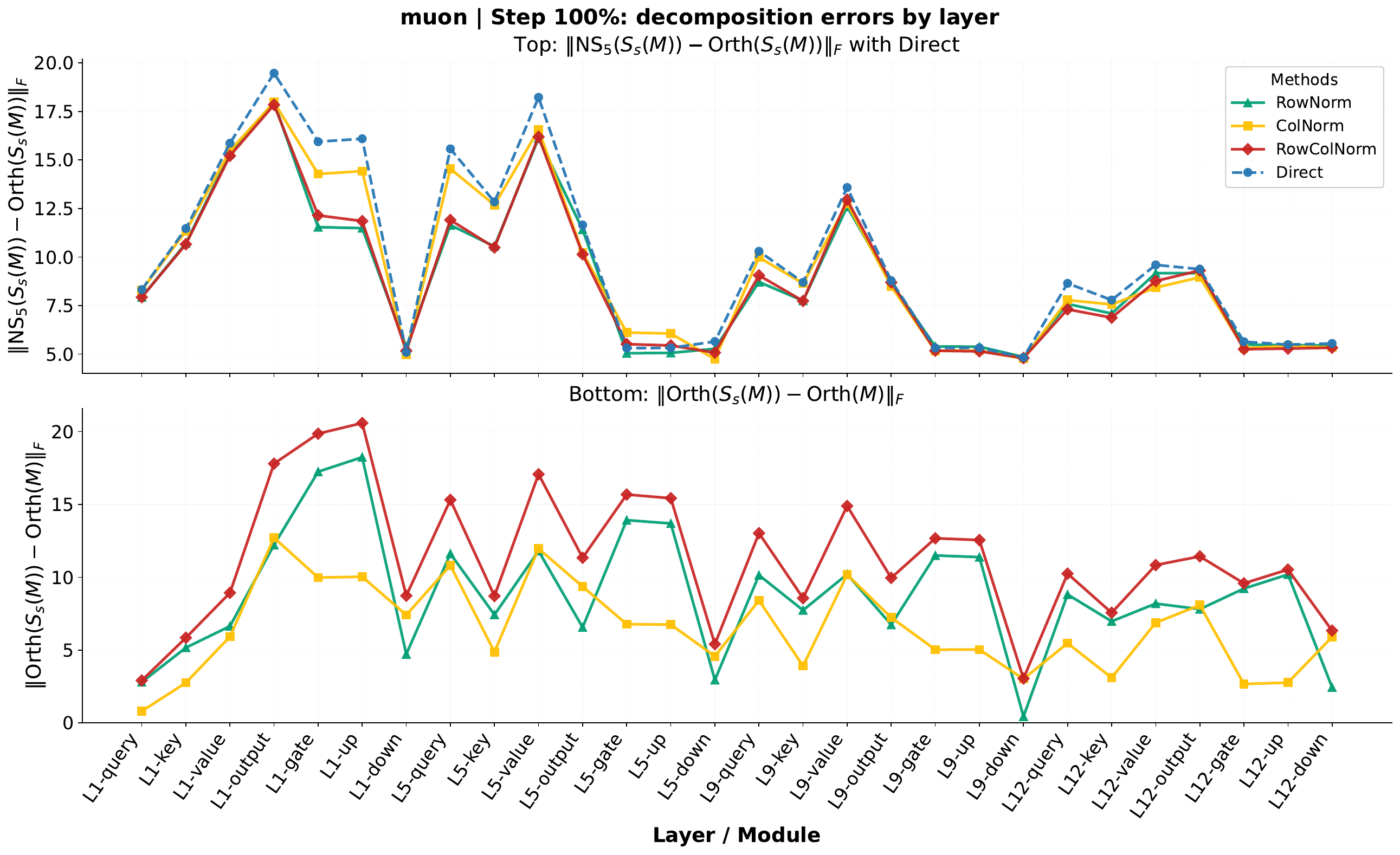} 
    \end{minipage}
    }
    \caption{Layerwise Muon NS5 bias decomposition under different normalization schemes during LLaMA2-130M training on C4 for 2.6B tokens. The top panel reports $\|\mathrm{NS}_5(S(\mathbf{M}))-\mathrm{Orth}(S(\mathbf{M}))\|_F$, while the bottom panel reports $\|\mathrm{Orth}(S(\mathbf{M}))-\mathrm{Orth}(\mathbf{M})\|_F$, both evaluated across layers at the same training step. These plots reveal how the choice of normalization shifts the error budget between the \textbf{finite-step approximation error} term and the \textbf{preconditioning bias} term.}
\label{fig:e_b_metrics}
\end{figure}

% \begin{figure}
%     \centering
%     \includegraphics[width=0.9\linewidth]{fig/orth_error/ns5_orth_mean_spec_norm.pdf}
%    \caption{The averaged error $\varepsilon_{\mathrm{ns}}$ remains below $1$ across all tracked training steps, indicating a uniformly bounded spectral-norm deviation from exact orthogonalization.}
% \label{fig:ns5_spec_norm}
% \end{figure}

\begin{figure}[!htbp]
    \centering
    \subfloat[Muon-Nes | K=4\label{fig:heatmap-muon-k4}]{
    \includegraphics[width=0.4\textwidth]{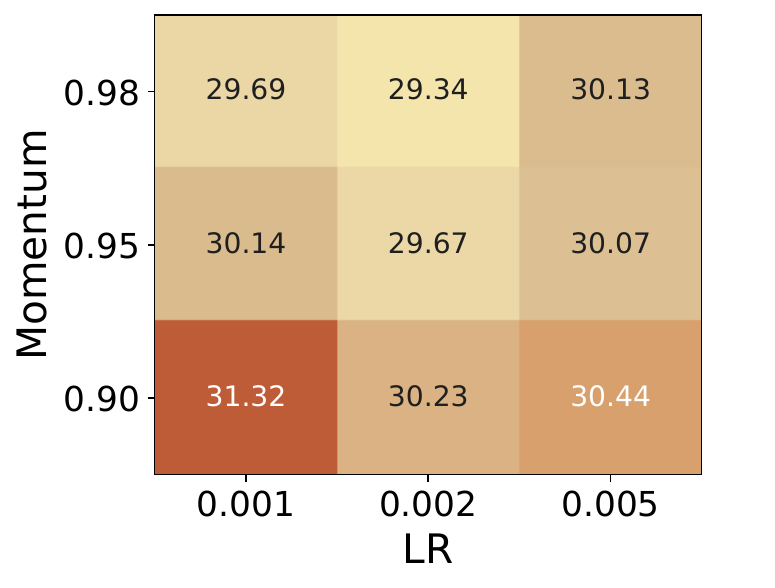}
    }
    \subfloat[MuonEq-Nes | K=4\label{fig:heatmap-muoneq-k4}]{
        \includegraphics[width=0.4\textwidth]{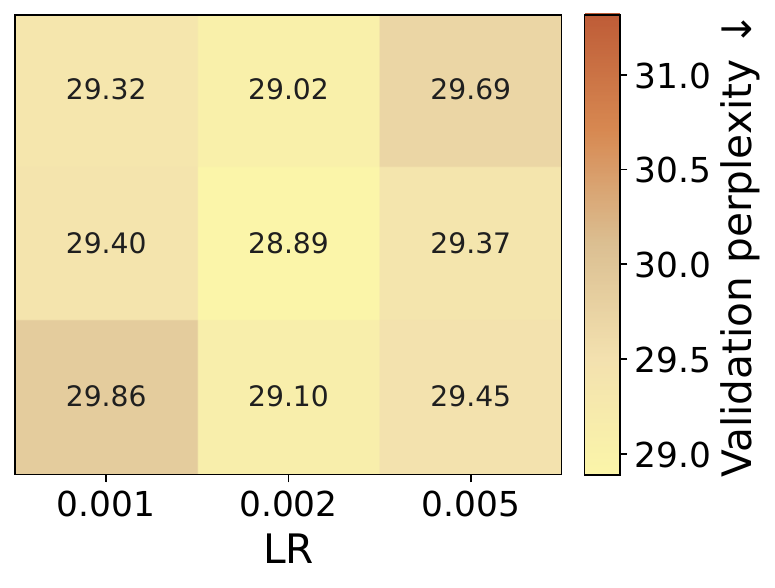}
    }
    \\
    \subfloat[Muon-Nes | K=5\label{fig:heatmap-muon-k5-ExpDetial}]{
        \includegraphics[width=0.4\textwidth]{fig/heatmap/heatmap_muon_k5.pdf}
    }
    \subfloat[MuonEq-Nes | K=5\label{fig:heatmap-muoneq-k5-ExpDetial}]{
        \includegraphics[width=0.4\textwidth]{fig/heatmap/heatmap_muoneqrow_k5.pdf}
    }
    \caption{Validation perplexity heatmaps on GPT2-small/FineWeb (2.6B tokens) over learning rate and momentum at $K=4$ and $K=5$.}
    \label{fig:heatmap}
\end{figure}

\begin{table}[h!]
\centering
\caption{Hyperparameters used for training LLaMA2-130M on C4}
\resizebox{1.0\textwidth}{!}{
\begin{tabular}{lccccccc}
\hline
Hyper-parameter & AdamW &MARS-AdamW & Muon & Muon-Nes & MuonEq & MuonEq-Nes \\
\hline
Max Learning Rate & 5e-3 & 3e-3 & 3e-3 & 3e-3 & 3e-3 & 3e-3 \\
% & 350M & 8e-3 & 8e-3 & 4e-3 \\
Warmup Steps  & \multicolumn{6}{c}{500}\\
Batch Size & \multicolumn{6}{c}{128} \\
Maximum Length & \multicolumn{6}{c}{4096} \\
Weight Decay  & \multicolumn{6}{c}{0.1} \\
$(\beta_1,\beta_2)$  & \multicolumn{6}{c}{(0.9,0.95)} \\
% Stability term  $\varepsilon$  & \multicolumn{2}{c}{1e-8} & \multicolumn{2}{c}{\XSolidBrush} & \multicolumn{2}{c}{1e-8} \\
Muon-Momentum  & \XSolidBrush  & \XSolidBrush  & \multicolumn{4}{c}{0.95} \\
Gamma  & \XSolidBrush  & 0.025 & \XSolidBrush & 0.05 & \XSolidBrush & 0.05 \\
\hline
\end{tabular}
}
\label{tab:llama130m_c4_hyperparameters}
\end{table}

\begin{table}[h!]
\centering
\caption{Hyperparameters used for training LLaMA2-350M on C4}
\resizebox{1.0\textwidth}{!}{
\begin{tabular}{lccccccc}
\hline
Hyper-parameter & AdamW &MARS-AdamW & Muon & Muon-Nes & MuonEq & MuonEq-Nes \\
\hline
Max Learning Rate & 1.5e-3 & 2e-3 & 1.5e-3 & 2e-3 & 2e-3 & 2e-3 \\
% & 350M & 8e-3 & 8e-3 & 4e-3 \\
Warmup Steps  & \multicolumn{6}{c}{500}\\
Batch Size & \multicolumn{6}{c}{256} \\
Maximum Length & \multicolumn{6}{c}{4096} \\
Weight Decay  & \multicolumn{6}{c}{0.1} \\
$(\beta_1,\beta_2)$  & \multicolumn{6}{c}{(0.9,0.95)} \\
% Stability term  $\varepsilon$  & \multicolumn{2}{c}{1e-8} & \multicolumn{2}{c}{\XSolidBrush} & \multicolumn{2}{c}{1e-8} \\
Muon-Momentum  & \XSolidBrush  & \XSolidBrush  & \multicolumn{4}{c}{0.95} \\
Gamma  & \XSolidBrush  & 0.025 & \XSolidBrush & 0.05 & \XSolidBrush & 0.05 \\
\hline
\end{tabular}
}
\label{tab:llama350m_c4_hyperparameters}
\end{table}

\begin{table}[h!]
\centering
\caption{Hyperparameters used for training LLaMA2-1B on C4}
\resizebox{1.0\textwidth}{!}{
\begin{tabular}{lccccccc}
\hline
Hyper-parameter & AdamW &MARS-AdamW & Muon & Muon-Nes & MuonEq & MuonEq-Nes \\
\hline
Max Learning Rate & \multicolumn{6}{c}{\{5e-4, 8e-4\}} \\
Warmup Steps  & \multicolumn{6}{c}{1000}\\
Batch Size & \multicolumn{6}{c}{512} \\
Maximum Length & \multicolumn{6}{c}{4096} \\
Weight Decay  & \multicolumn{6}{c}{0.1} \\
$(\beta_1,\beta_2)$  & \multicolumn{6}{c}{(0.9,0.95)} \\
% Stability term  $\varepsilon$  & \multicolumn{2}{c}{1e-8} & \multicolumn{2}{c}{\XSolidBrush} & \multicolumn{2}{c}{1e-8} \\
Muon-Momentum  & \XSolidBrush  & \XSolidBrush  & \multicolumn{4}{c}{0.95} \\
Gamma  & \XSolidBrush  & 0.025 & \XSolidBrush & 0.05 & \XSolidBrush & 0.05 \\
\hline
\end{tabular}
}
\label{tab:llama1b_c4_hyperparameters}
\end{table}

\subsection{Ablation Study Details}
\label{app:ablation_details}
\begin{table}[h!]
\centering
\caption{Hyperparameters used for training ResNet-18 on CIFAR10}
\resizebox{1.0\textwidth}{!}{
\begin{tabular}{lccccccc}
\hline
Hyper-parameter & SGD & AdamW & Muon & Muon-Nes & \makecell{MuonEq-Nes \\ (RC)} & \makecell{MuonEq-Nes \\ (C)} &  \makecell{MuonEq-Nes \\ (R) } \\
\hline
Max Learning Rate & 5e-2 & 1e-3 & 5e-2 & 5e-2 & 5e-2 & 5e-2 & 5e-2 \\
Warmup Steps  & \multicolumn{7}{c}{0}\\
Epochs & \multicolumn{7}{c}{100}\\
Batch Size & \multicolumn{7}{c}{128} \\
$(\beta_1,\beta_2)$  &  & (0.9,0.999) &\multicolumn{5}{c}{(0.9,0.95)} \\
% Stability term  $\varepsilon$  & \XSolidBrush & 1e-8 & \multicolumn{2}{c}{\XSolidBrush} & \multicolumn{3}{c}{1e-8} \\
Muon-Momentum  & \XSolidBrush  & \XSolidBrush  & \multicolumn{5}{c}{0.9} \\
\hline
\end{tabular}
}
\label{tab:resnet18_cifar10_hyperparameters}
\end{table}

\begin{table}[h!]
\centering

\caption{Hyperparameters used for training GPT2-small on FineWeb}
\resizebox{1.0\textwidth}{!}{
\begin{tabular}{lcccccccc}
\hline
Hyper-parameter & AdamW & AdaMuon & \makecell{Muon+ \\ (colrow)} & Mousse & Muon-Nes & \makecell{MuonEq-Nes \\ (RC)} & \makecell{MuonEq-Nes \\ (C)} &  \makecell{MuonEq-Nes \\ (R) } \\
\hline
Max Learning Rate & 2e-3 & 5e-3 & 5e-3 & 2e-3 & 2e-3 & 2e-3 & 2e-3 & 2e-3 \\
Warmup Steps  & \multicolumn{8}{c}{1000}\\
Total Steps  & \multicolumn{8}{c}{20000}\\
Batch Size & \multicolumn{8}{c}{128} \\
Maximum Length & \multicolumn{8}{c}{4096} \\
Weight Decay  & \multicolumn{8}{c}{0.1} \\
$(\beta_1,\beta_2)$  & \multicolumn{8}{c}{(0.9,0.95)} \\
Muon-Momentum  & \XSolidBrush   & \multicolumn{7}{c}{0.95} \\
Gamma  & \XSolidBrush & \XSolidBrush & \XSolidBrush & \XSolidBrush & \XSolidBrush & \multicolumn{3}{c}{0.05} \\
\hline
\end{tabular}
}
\label{tab:gpt2small_fineweb_hyperparameters}
\end{table}
$\blacktriangleright$ \textbf{Experimental setup.} All ablation experiments are conducted on 4 RTX Pro6000 (96GB) GPUs, with results averaged over three random seeds. We use two workloads: CIFAR-10~\cite{krizhevsky2009learning} with ResNet-18~\cite{he2016deep} and FineWeb~\cite{penedo2024fineweb} pretraining of GPT2-small~\cite{radford2019language} up to 10.5B tokens. Unless otherwise stated, the architecture, data pipeline, training budget, learning-rate schedule, warmup length, weight decay, momentum, Newton--Schulz iteration budget, and parameter grouping are kept identical to the corresponding reference recipe; only the diagonal map before orthogonalization is changed. The selected hyperparameters are summarized in Tables~\ref{tab:resnet18_cifar10_hyperparameters} and~\ref{tab:gpt2small_fineweb_hyperparameters}. For learning-rate tuning, we use discrete search grids that depend on the workload and optimizer family. On CIFAR-10 with ResNet-18, the SGD baseline is tuned over $\{1\mathrm{e}{-1}, 5\mathrm{e}{-2}, 1\mathrm{e}{-2}, 5\mathrm{e}{-3}\}$, AdamW over $\{5\mathrm{e}{-3}, 1\mathrm{e}{-3}, 5\mathrm{e}{-4}, 1\mathrm{e}{-4}\}$ , and Muon-type optimizers over $\{1\mathrm{e}{-1}, 5\mathrm{e}{-2}, 1\mathrm{e}{-2}, 5\mathrm{e}{-3}, 1\mathrm{e}{-3}\}$. We use the Muon implementation from \texttt{cifar10-airbench}\footnote{\url{https://github.com/KellerJordan/cifar10-airbench}}. On FineWeb pretraining of GPT2-small, all methods are tuned over the same learning-rate grid $\{5\mathrm{e}{-4}, 1\mathrm{e}{-3}, 2\mathrm{e}{-3}, 5\mathrm{e}{-3}\}$. We use the Muon implementation from \texttt{Moonlight}\footnote{\url{https://github.com/MoonshotAI/Moonlight}}

We compare the three {\method}-Nes variants RC, R, and C. RC applies two-sided row/column normalization throughout training, R applies row normalization throughout training, and C applies column normalization throughout training. We take R as the default variant. In all cases, the orthogonalization routine, optimizer states, and scalar hyperparameters are unchanged.

The purpose of this ablation is to isolate the effect of pre-orthogonalization geometry rather than to redesign the optimizer for each variant. We therefore use the same training budget and evaluation protocol across all runs. On CIFAR-10 we report final test accuracy. On FineWeb/GPT2-small we report validation perplexity at the end of the 10.5B-token budget.

This ablation is directly connected to Eq.~\eqref{eq:error_decomp}. RC provides the strongest two-sided spectral correction and mainly targets the finite-step Newton--Schulz term. R and C test the two one-sided geometries under the same optimizer configuration. Since {\method}/{\method}-Nes targets hidden matrix weights, R is the more relevant default in our setting, while C is included as the column-sided companion form.

\subsection{A Note on Row/Column Terminology in Prior Work}
\label{app:ablation_term_note}
A careful comparison to prior row/column normalization papers requires fixing the matrix convention. \citep{glentis2025minimalist} write linear weights as $\theta \in \mathbb{R}^{d_{\mathrm{in}}\times d_{\mathrm{out}}}$ and define column-wise normalization along the output dimension. Under the standard deep-learning storage layout $\mathbf{W}=\theta^\top \in \mathbb{R}^{d_{\mathrm{out}}\times d_{\mathrm{in}}}$, this corresponds to row normalization of the stored hidden-weight matrix. Their empirical observation that row-wise normalization can be unstable is also traced largely to the LM-head, where row-wise normalization produces extreme gradient values. This is not the same setting as {\method}/{\method}-Nes, which is aimed at hidden matrix weights.

By contrast, \citep{pethick2025training} use the standard convention $\mathbf{W} \in \mathbb{R}^{d_{\mathrm{out}}\times d_{\mathrm{in}}}$, so their ColNorm is genuinely column-wise in the same orientation as ours. Their main recommendation for language models is first-layer ColNorm, hidden-layer Spectral, and last-layer Sign. We therefore interpret our row-sided default choice as directly aligned with \citep{Xu2026moga}, compatible with \citep{glentis2025minimalist} after accounting for layout conventions, and not in tension with \citep{pethick2025training} once first-layer/input geometry is separated from hidden-layer geometry.

\section{Limitations}
\label{appendix:lim}
{\method} is designed as a lightweight pre-orthogonalization correction for matrix-valued hidden weights. The row-normalized variant is our default in this setting, but it should not be viewed as a universally optimal choice for every parameter block or architecture. Our analysis isolates the core geometry of the pre-balanced orthogonalized update, while broader coverage of implementation variants and training regimes remains an open direction. Our convergence analysis for the RC variant is restricted to an auxiliary exact-polar regime: Proposition~\ref{th_conv_rc} assumes Nesterov momentum is disabled, no decoupled weight decay is used, and the stabilization parameter \(\varepsilon\) is sufficiently large. Therefore, this RC guarantee does not formally cover the practical RC ablations, which use finite-step Newton--Schulz iterations and the small numerical stabilization \(\varepsilon=10^{-8}\). Extending the RC theory to this practical regime is left for future work. This limitation is specific to the auxiliary RC analysis; the main guarantee for the default R variant is given by Theorem~\ref{th_conv_r} and Corollary~\ref{cor:conv_r_ns}. Empirically, our large-scale evaluation focuses on LLaMA-style pretraining on C4, complemented by smaller-scale ablations.

\section{Broader Impacts}
\label{appendix:impact}
This work studies an application-agnostic optimizer that may improve training efficiency for matrix-valued neural network parameters without introducing new data, data collection, or deployment pipelines. Its direct societal risks are therefore mainly inherited from the downstream models and datasets on which it is used. As with other efficiency improvements, {\method} could lower the cost of both beneficial and harmful model training, so appropriate downstream safeguards remain necessary. Any environmental benefit depends on whether efficiency gains are used to reduce compute for a fixed objective rather than to scale up training.
% \newpage
% \input{checklist.tex}
\end{document}